\documentclass[11pt]{article}
\usepackage[T1]{fontenc}
\usepackage{lmodern}
\usepackage[a4paper,margin=27mm]{geometry}
\usepackage{amsmath,amssymb,amsthm,mathtools}
\usepackage{microtype}
\usepackage{needspace}
\usepackage{tikz}
\usepackage{placeins}
\theoremstyle{definition}
\newtheorem{example}{Example}[section]
\usetikzlibrary{arrows.meta,positioning}
\usepackage[hidelinks]{hyperref}
\numberwithin{equation}{section}
\newtheorem{theorem}{Theorem}[section]
\newtheorem{lemma}[theorem]{Lemma}
\newtheorem{proposition}[theorem]{Proposition}
\newtheorem{corollary}[theorem]{Corollary}
\theoremstyle{remark}
\newtheorem{remark}[theorem]{Remark}
\newcommand{\R}{\mathbb R}
\newcommand{\N}{\mathbb N}
\newcommand{\T}{\mathbb T}
\newcommand{\Z}{\mathbb Z}
\newcommand{\SOS}{\Sigma}
\newcommand{\M}{\mathcal M}
\newcommand{\PP}{\mathcal P}
\newcommand{\An}{\mathcal A}
\newcommand{\size}{\mathfrak S}

\newcommand{\degF}{\deg_{\mathrm F}}
\newcommand{\QM}{\operatorname{QM}}
\newcommand{\acosh}{\operatorname{arcosh}}
\newcommand{\asinh}{\operatorname{arsinh}}
\newcommand{\sq}[1]{[ #1 ]^2}
\hypersetup{
  pdftitle={Sharp Degree Bounds for Putinar Certificates in Terms of the Positivity Margin via Analytic Sums of Squares},
  pdfauthor={Yuxiang Zhao, Ting Gan},
  pdfsubject={Quantitative sums of squares, positive margins, and matching lower bounds}
}

\title{Sharp Degree Bounds for Putinar Certificates in Terms of the Positivity Margin via Analytic Sums of Squares}

\author{
Yuxiang Zhao \qquad Ting Gan\\
School of Computer Science, Wuhan University\\
Wuhan, China\\
\texttt{yuxiangzhao@whu.edu.cn}\\
\texttt{ganting@whu.edu.cn}
}

\date{October 2026}
\begin{document}
\maketitle

\begin{abstract}
We establish sharp degree bounds for Putinar sum-of-squares
(SOS) certificates in terms of the positivity margin $\delta$. For any fixed Archimedean quadratic module with
prescribed generators and any fixed degree bound on the
objective, normalized polynomial objectives admit a
uniform certificate-degree upper bound of
$O(\delta^{-1/2}\log^2(2/\delta))$.
The bound is sharp: we construct a matching lower-bound
example in three variables. We further show that the
number of inequality generators affects the optimal
degree bound. For quadratic modules with at most two
inequality generators, the bound improves to
$O(\delta^{-1/2}\log(2/\delta))$, which is also sharp.
This improvement is obtained by directly converting
two-generator Stengle refutations into certificates
for the original quadratic module.
We also establish sharp degree bounds for trigonometric
SOS certificates on tori. In particular, for the full
torus without additional constraints, the optimal
worst-case order improves to $\delta^{-1/4}$, whereas additional equality
constraints can lead the order to 
$\delta^{-1/2}\log(2/\delta)$.
These results also yield sharp worst-case convergence
rates for the Lasserre hierarchy.
\end{abstract}


\section{Introduction}\label{sec:introduction}
\subsection{The positivity-margin problem}
A polynomial positive on a basic closed semialgebraic set admits a
Putinar-type sum-of-squares (SOS) representation when the associated
quadratic module is Archimedean~\cite{Putinar}.  Existence alone does not
quantify how many degrees a certificate requires as the polynomial approaches
the boundary of the nonnegative cone.  We study the \emph{dependence of the certificate degree on the positivity
margin}: the set of constraints is fixed, the objective degree is bounded,
and only its positive minimum is allowed to tend to zero.

Let $\SOS[x]$ denote the sums of polynomial squares in
$\R[x_1,\ldots,x_n]$. Fix
\begin{equation}\label{eq:module}
 \M=\SOS[x]+\sum_{i=1}^m g_i\SOS[x]+(h_1,\ldots,h_\ell),
 \qquad K=\{\mathbf x:g_i(\mathbf x)\ge0,\ h_j(\mathbf x)=0\}\ne\varnothing.
\end{equation}
A polynomial Putinar certificate for $f\in\M$ is a representation
\[
 f=\sq{\boldsymbol P_0}
   +\sum_{i=1}^m g_i\sq{\boldsymbol P_i}
   +\sum_{k=1}^\ell h_k\lambda_k,
\]
where $\boldsymbol P_i=(P_{i,1},\ldots,P_{i,N_i})$ is a finite
vector of polynomials and
$\sq{\boldsymbol P_i}=\sum_{j=1}^{N_i}P_{i,j}^2$.
Its degree budget is at most $d$ when the polynomials $P_{i,j}$ and equality multipliers satisfy
\[
 2\deg P_{0,j}\le d,\qquad
 \deg g_i+2\deg P_{i,j}\le d,\qquad
 \deg h_k+\deg\lambda_k\le d,
\]
for all applicable indices.
Write $d_\M(f)$ for the least such budget. We fix an Archimedean certificate witnessing
\begin{equation}\label{eq:ball-general}
 r^2-\sum_{j=1}^n x_j^2\in\M_{B_0},\qquad r>0.
\end{equation}
In particular, $K$ is compact. The specified Archimedean
certificate is stronger than compactness of $K$ alone and will
be used to absorb approximation errors. Equality constraints can equivalently
be encoded by adjoining $h_j$ and $-h_j$ as inequality generators, since
$h_j\lambda=h_j((\lambda+1)/2)^2+(-h_j)((\lambda-1)/2)^2$.
This equivalence concerns unrestricted module membership
and need not preserve degree budgets; we retain the
equality ideal to track the original budgets.

The result gives a uniform degree bound in terms of the positivity margin
for each fixed constraint system, together with an example showing that the
asymptotic order cannot be improved for the full class.
\begin{theorem}[Sharp margin bound]\label{thm:main}
Fix the presentation \eqref{eq:module}, the witness
\eqref{eq:ball-general}, and $D\ge1$.
Put $Q=[-4r,4r]^n$.
There exists $C_{\M,D}<\infty$ such that
\begin{equation}\label{eq:main-upper}
 \deg f\le D,\quad \|f\|_{Q,\infty}\le1,\quad f|_K\ge\delta>0
 \quad\Longrightarrow\quad
 d_\M(f)\le C_{\M,D}\delta^{-1/2}\log^2(2/\delta).
\end{equation}
Since $K\subseteq Q$, the assumptions imply
$0<\delta\le1$. The certificate uses only the original generators and equality ideal.

The order is attained by the fixed module
\begin{equation}\label{eq:example}
 \M_* =\QM\left(3-u^2-x^2-y^2,\ -ux,\ -uy,\ \tfrac12u(xy+u^4)\right).
\end{equation}
Let $K_*$ be its feasible set. For $R_*=4\sqrt3$,
\[
 f_\varepsilon=\frac{u+\varepsilon}{R_*+\varepsilon},\qquad
 \delta_\varepsilon=\frac{\varepsilon}{R_*+\varepsilon},
\]
we have $\|f_\varepsilon\|_{[-R_*,R_*]^3,\infty}=1$, $\min_{K_*} f_\varepsilon=\delta_\varepsilon$, and
\begin{equation}\label{eq:sharp}
 d_{\M_*}(f_\varepsilon)
 =\Theta\!\left(\delta_\varepsilon^{-1/2}\log^2(1/\delta_\varepsilon)\right)
 \quad(\varepsilon\downarrow0).
\end{equation}
\end{theorem}

\subsection{Normalization and sharpness}
The estimate is uniform over normalized objectives of degree at most $D$
for the specified generators and Archimedean witness. In unnormalized form,
the upper bound is
\[
 C_{\M,D}\sqrt{\kappa(f)}\log^2(2\kappa(f)),\qquad
 \kappa(f)=\frac{\|f\|_{Q,\infty}}{\min_Kf}.
\]
The fixed example in Theorem~\ref{thm:main} establishes worst-case sharpness.

\subsection{From preorderings to quadratic modules}
The main upper-bound obstacle is algebraic: a \emph{preordering}
admits weighted SOS terms such as $g_1g_2\sigma$ with $\sigma$ SOS,
whereas a quadratic module generally does not. Thus a Stengle certificate
involving generator products cannot directly serve as a Putinar certificate.  For example, $g_1g_2$ is a generator product in the preordering,
whereas $g_1g_2\in\M$ is not guaranteed by the defining axioms of $\M$.
The decisive upper-bound step below is a quantitative conversion that
\emph{removes} such products at controlled analytic cost, without changing
the set of generators.

\subsection{Proof architecture and the two logarithms}
The proof has two complementary parts. Sections~\ref{sec:analytic}--\ref{sec:upper}
construct an exact Putinar certificate and bound its degree; Sections~\ref{sec:corner}--\ref{sec:lower}
show that this bound is attained by a fixed quadratic module. The diagram
identifies the task and main output of each section. Section~\ref{sec:consequences}
then derives the SOS optimization consequences.

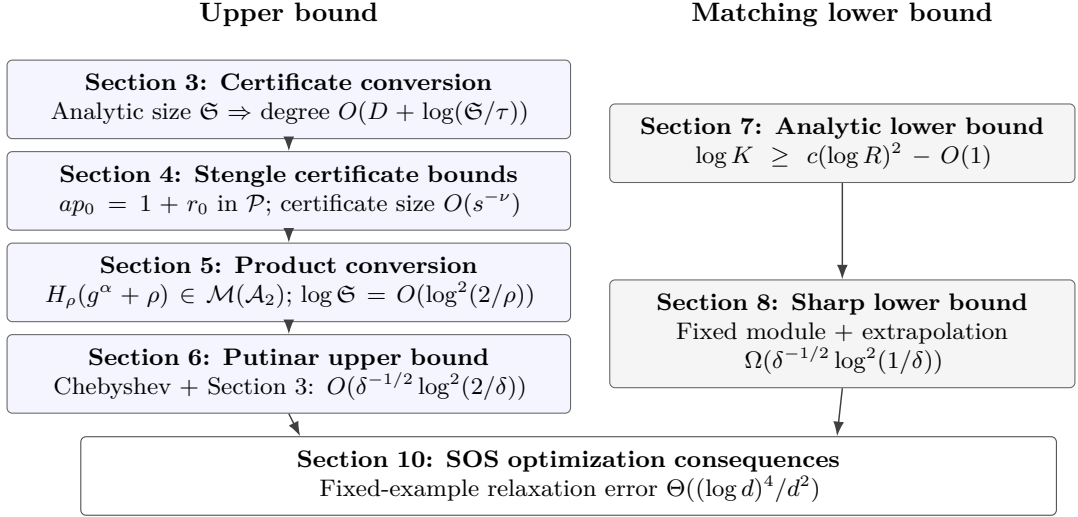
\begin{figure}[!htbp]
\centering
\begin{tikzpicture}[
  >=Latex,
  upper/.style={draw=black!60,rounded corners=2pt,fill=blue!4,
    text width=7.10cm,align=center,inner sep=5pt,font=\footnotesize},
  lower/.style={draw=black!60,rounded corners=2pt,fill=black!4,
    text width=5.80cm,align=center,inner sep=5pt,font=\footnotesize},
  result/.style={draw=black!70,rounded corners=2pt,
    text width=12.55cm,align=center,inner sep=5pt,font=\footnotesize},
  edge/.style={->,line width=.55pt,draw=black!75}
]
\node[font=\small\bfseries] at (-3.70,0.20) {Upper bound};
\node[font=\small\bfseries] at ( 3.62,0.20) {Matching lower bound};
\node[upper] (u2) at (-3.70,-0.93)
 {\textbf{Section~\ref{sec:analytic}: Certificate conversion}\\
  Analytic size $\size$ $\Rightarrow$ degree $O(D+\log(\size/\tau))$};
\node[upper] (u3) at (-3.70,-2.14)
 {\textbf{Section~\ref{sec:stengle}: Stengle certificate bounds}\\
  $ap_0=1+r_0$ in $\PP$; certificate size $O(s^{-\nu})$};
\node[upper] (u4) at (-3.70,-3.35)
 {\textbf{Section~\ref{sec:products}: Product conversion}\\
  $H_\rho(g^\alpha+\rho)\in\M(\An_2)$; $\log\size=O(\log^2(2/\rho))$};
\node[upper] (u5) at (-3.70,-4.56)
 {\textbf{Section~\ref{sec:upper}: Putinar upper bound}\\
  Chebyshev + Section~\ref{sec:analytic}: $O(\delta^{-1/2}\log^2(2/\delta))$};
\node[lower] (l6) at (3.62,-1.51)
 {\textbf{Section~\ref{sec:corner}: Analytic lower bound}\\
  $\log K\ge c(\log R)^2-O(1)$};
\node[lower] (l7) at (3.62,-4.03)
 {\textbf{Section~\ref{sec:lower}: Sharp lower bound}\\
  Fixed module + extrapolation\\
  $\Omega(\delta^{-1/2}\log^2(1/\delta))$};
\node[result] (app) at (0,-5.91)
 {\textbf{Section~\ref{sec:consequences}: SOS optimization consequences}\\
  Fixed-example relaxation error $\Theta((\log d)^4/d^2)$};
\draw[edge] (u2.south) -- (u3.north);
\draw[edge] (u3.south) -- (u4.north);
\draw[edge] (u4.south) -- (u5.north);
\draw[edge] (l6.south) -- (l7.north);
\draw[edge] (u5.south) -- ([xshift=-3.55cm]app.north);
\draw[edge] (l7.south) -- ([xshift=3.55cm]app.north);
\end{tikzpicture}
\caption{Proof roadmap: the role and output of each section.}
\label{fig:proof-map}
\end{figure}
\FloatBarrier

Section~\ref{sec:products} produces a logarithmic certificate-size cost
$\log\size=O(\log^2(2/s))$. The Chebyshev step in
Section~\ref{sec:upper} contributes the factor $s^{-1/2}$, and
Lemma~\ref{lem:exactification} converts logarithmic size into polynomial
degree. On the lower-bound side, Section~\ref{sec:corner} establishes
$\log K\ge c(\log R)^2-O(1)$; Section~\ref{sec:lower} converts this
analytic norm estimate into a certificate-degree lower bound.

\paragraph{Position relative to previous work.}
Uniform degree bounds for positive polynomials of bounded degree and norm
for a fixed Archimedean quadratic module with specified generators were already known;
see Kriel~\cite{Kriel}. Their dependence on the positivity margin has been
studied in effective versions of Putinar's Positivstellensatz
\cite{NieSchweighofer2007,BaldiMourrain,BaldiMourrainParusinski,HKVZ}.
Sharp degree bounds on special domains also depend on the type of certificate:
Schm\"udgen's preordering and Putinar's quadratic module are generally
different~\cite{LaurentSlot2023,BaldiSlot}. More recent polynomial-kernel
methods give improved rates for the standard hypercube quadratic
module~\cite{GriblingKlerkVera2026}.
Our result determines the worst-case margin dependence for any
Archimedean quadratic module with specified generators, with
constants depending on the generators and objective-degree bound. Section~\ref{sec:related} compares it with
degree-only stability, geometry-dependent estimates, and bounds for
other SOS representations.

\section{Preliminaries}\label{sec:Pre}

\paragraph{Notation and conventions.}
We write $\mathbf x=(x_1,\ldots,x_n)$ and use the multi-index
conventions $\mathbf x^\alpha=\prod_{j=1}^n x_j^{\alpha_j}$
and $|\alpha|=\sum_{j=1}^n\alpha_j$.
Scalar polynomials and power series are written in ordinary type,
while finite vectors of such functions are written in bold.
For $\boldsymbol U=(U_1,\ldots,U_N)$, set
\[
 \sq{\boldsymbol U}=\sum_{j=1}^N U_j^2,
\]
where the squares are algebraic (without complex conjugation).
We use $\oplus$ for vector concatenation and
$\|\cdot\|_{\ell^2}$ for the Euclidean norm of a finite real vector.
For a bounded real-valued function $f$ on $X$, write
$\|f\|_{X,\infty}=\sup_{x\in X}|f(x)|$.

\paragraph{Quadratic modules and preorderings.}
We use the quadratic module $\M$ and its degree truncations $\M_d$
introduced in Eq.~\eqref{eq:module}. The associated preordering is
\begin{equation}\label{eq:preordering}
 \PP=\sum_{\alpha\in\{0,1\}^m}g^\alpha\SOS[x]+(h_1,\ldots,h_\ell),
 \qquad g^\alpha=\prod_{i=1}^m g_i^{\alpha_i}.
\end{equation}
Unlike a quadratic module, a preordering permits products of inequality
generators. Its truncation $\PP_d$ requires every weighted SOS term
and equality-ideal term to have degree at most $d$; we write
$d_\PP(f)$ for the least such budget.

\paragraph{Archimedean normalization.}
For the upper-bound analysis, rescaling the variables
in \eqref{eq:ball-general} and relabeling the
presentation allows us to assume
\begin{equation}\label{eq:normalized-witness}
b(\mathbf x)=1-\sum_{j=1}^n x_j^2
\in\M_{B_0},\qquad B_0\ge2.
\end{equation}
We fix a certificate for this Archimedean witness,
which will be used to absorb polynomial approximation
errors.

\paragraph{Analytic Wiener norms.}
For a real power series
$A(\mathbf x)=\sum_{\alpha\in\mathbb N_0^n}
 a_\alpha\mathbf x^\alpha$ and $R>0$, define
\begin{equation}\label{eq:analytic-norm}
 \|A\|_R=\sum_\alpha |a_\alpha|R^{|\alpha|},\qquad
 \An_R=\{A:\|A\|_R<\infty\}.
\end{equation}
This weighted analytic Wiener norm is submultiplicative,
$\|AB\|_R\le\|A\|_R\|B\|_R$, and $\An_R$ is a
commutative Banach algebra. In particular, $\|A\|_1$ is the
unweighted coefficient Wiener norm, $\|A\|_2$ is the norm
at radius two, and $\|A\|_1\le\|A\|_R$ for $R\ge1$.

\paragraph{Certificate terminology.}
A certificate is an algebraic identity with specified SOS
decompositions and equality multipliers. Its degree and size refer to
the chosen representation; the same polynomial may have certificates
of different degrees and sizes.
In a polynomial Putinar certificate
$f=\sigma_0+\sum_i g_i\sigma_i+\sum_k h_k\lambda_k$,
write $\sigma_i=\sq{\boldsymbol P_i}=\sum_j P_{i,j}^2$.
For $i\ge1$, the $\sigma_i$ are \emph{SOS multipliers} of the
inequality generators $g_i$, whereas $\sigma_0$ is the
\emph{unweighted SOS polynomial}. The $P_{i,j}$ are the
polynomials in the chosen SOS decompositions, and their squares
$P_{i,j}^2$ are the individual summands. The $g_i\sigma_i$
are \emph{weighted SOS terms}. The $\lambda_k$ are
\emph{equality multipliers}, and $h_k\lambda_k$ are
\emph{equality-ideal terms}.
For a preordering certificate, the weights are generator products
$g^\alpha$, with corresponding SOS polynomials $\sigma_\alpha$
(including the unweighted term for $\alpha=0$).
A \emph{Stengle certificate} for $a$ is an identity
$ap=1+r$ with specified preordering representations of $p$ and $r$.
We call $p$ the \emph{auxiliary multiplier}; both $p$ and $r$
belong to the preordering.

\paragraph{Analytic SOS certificates and size.}
An analytic SOS certificate for $F$ is a representation
\begin{equation}\label{eq:analytic-certificate}
 F=\sq{\boldsymbol A_0}
   +\sum_{i=1}^m g_i\sq{\boldsymbol A_i}
   +\sum_{k=1}^\ell h_kL_k,
\end{equation}
where $\boldsymbol A_i=(A_{i,1},\ldots,A_{i,M_i})$
is a finite vector of analytic functions and
$A_{i,j},L_k\in\An_2$.
We write $F\in\M(\An_2)$ when such a certificate exists.
The size of a certificate is
\begin{equation}\label{eq:analytic-size}
 \size=1+\sum_{i=0}^m\sum_{j=1}^{M_i}\|A_{i,j}\|_2^2
          +\sum_{k=1}^\ell\|L_k\|_2.
\end{equation}
For polynomial and preordering certificates, size is defined
analogously using the polynomials in the SOS decompositions
and the equality multipliers.
For an identity $ap=1+r$ with specified representations of $p$
and $r$, its certificate size $\size_{p,r}$ counts the contributions
from the SOS decompositions and equality multipliers of both, with
a single leading $1$. The same convention applies when $p,r$ have
analytic quadratic-module representations.

\paragraph{Chebyshev polynomials and coefficient bounds.}
Let $T_k$ and $U_k$ denote the Chebyshev
polynomials of the first and second kinds,
respectively, characterized by
\begin{equation}\label{eq:chebyshev-identities}
T_k(\cos\theta)=\cos(k\theta),\qquad
\sin(k\theta)=\sin\theta\,U_{k-1}(\cos\theta)
\quad(k\ge1).
\end{equation}
Both families satisfy the recurrence
$P_{k+1}=2tP_k-P_{k-1}$, with initial values
$T_0=U_0=1$, $T_1=t$, and $U_1=2t$.
Since $\|tP\|_1=\|P\|_1$, induction gives
\begin{equation}\label{eq:chebyshev-basic-coef}
\|T_k\|_1,\ \|U_k\|_1\le3^k,
\qquad k\ge0.
\end{equation}

For the shifted polynomials $q_k(t)=T_k(1-2t)$, the recurrence gives
$\|q_{k+1}\|_1 \le6\|q_k\|_1+\|q_{k-1}\|_1.$ Since $\|q_0\|_1=1$ and $\|q_1\|_1=3$,
another induction yields
\begin{equation}\label{eq:chebyshev-coef-bound}
\|T_k(1-2t)\|_1\le7^k.
\end{equation}

We also recall the extremal properties
\begin{equation}\label{eq:chebyshev-growth}
|T_k(t)|\le1\quad (|t|\le1),
\qquad
|T_k(t)|
=\cosh(k\,\acosh|t|)\quad (|t|>1).
\end{equation}
More generally, for any real polynomial
$p$ of degree at most $d$ and $a<b$,
the Chebyshev extrapolation inequality~\cite[Theorem~18, Eq.~(9)]{BaldiSlot} gives
\begin{equation}\label{eq:chebyshev-extrapolation}
|p(t)|\le
T_d\left(\frac{|2t-a-b|}{b-a}\right)
\|p\|_{[a,b],\infty},
\qquad t\notin[a,b].
\end{equation}
The last inequality also holds for
finite-dimensional real vector-valued
polynomials, with the absolute value
replaced by the Euclidean norm.

\section{From analytic size to polynomial degree}\label{sec:analytic}

The constructions in Sections~\ref{sec:stengle}--\ref{sec:upper}
produce analytic SOS certificates, whereas
Theorem~\ref{thm:main} requires polynomial ones.
We bridge this gap by showing that an analytic certificate
for $f-\tau$ of size $\size$ yields an exact polynomial
certificate for $f$ of degree
$O(D+\log(\size/\tau)+1)$, where $D$ bounds $\deg f$.

\paragraph{Proof strategy.}
Taylor truncation preserves the SOS structure but leaves a polynomial
residual that need not belong to the original module $\M$.
We first use the normalized Archimedean witness
\eqref{eq:normalized-witness} to absorb small residuals without
changing the generators (Lemma~\ref{lem:domination}).
We then combine truncation and error absorption to obtain
the exact conversion bound (Lemma~\ref{lem:exactification}).
The analytic norms and certificate sizes used below are defined in
Section~\ref{sec:Pre}.

\paragraph{Taylor truncation.}
For $L\ge0$, let $\Pi_L$ denote truncation to total degree
at most $L$:
\[
 \Pi_L A=\sum_{|\alpha|\le L}a_\alpha\mathbf x^\alpha,
\]
applied componentwise to finite vector-valued power series.
Given a certificate \eqref{eq:analytic-certificate}, put
$\boldsymbol P_i=\Pi_L\boldsymbol A_i$ and
$\lambda_k=\Pi_L L_k$. The resulting polynomial
\begin{equation}
\label{eq:F_L}
     F_L=\sq{\boldsymbol P_0}
     +\sum_{i=1}^m g_i\sq{\boldsymbol P_i}
     +\sum_{k=1}^\ell h_k\lambda_k
\end{equation}

belongs to $\M_{2L+C_0}$, where $C_0$ depends only
on the degrees of the fixed generators. In particular, truncating the polynomials in the SOS
decompositions before squaring preserves the SOS structure.
For $A\in\An_2$, one has
$\|A-\Pi_LA\|_1\le2^{-(L+1)}\|A\|_2$; hence
\begin{equation}\label{eq:square-truncation}
 \|A^2-(\Pi_LA)^2\|_1
 \le2\|A-\Pi_LA\|_1\|A\|_2
 \le2^{-L}\|A\|_2^2.
\end{equation}

The following lemma shows that any polynomial belongs to an Archimedean quadratic module after adding a sufficiently large constant.
\begin{lemma}[Coefficient domination]\label{lem:domination}
If $e\in\R[x]$ has degree at most $s$ and $\tau\ge\|e\|_1$, then
\begin{equation}\label{eq:domination}
 \tau+e\in\M_{2s+B_0}.
\end{equation}
For the fixed witness, the certificate size is at most
$1+C(s+1)4^s(\tau+\|e\|_1)$, with an enlarged fixed constant when $s=0$.
\end{lemma}

\begin{proof}
 We first prove $\tau+e\in\M_{2s+B_0}$. By the normalized Archimedean witness
\eqref{eq:normalized-witness}, we have
\begin{equation}\label{eq:coordinate-certificate}
1-x_i^2=b(\mathbf x)+\sum_{k\ne i}x_k^2
\in\M_{B_0},\qquad i=1,\ldots,n.
\end{equation}

For any multi-index $\alpha$ with
$|\alpha|=t\le s$, write
$\mathbf x^\alpha=x_{i_1}\cdots x_{i_t}$,
allowing repeated indices.
The following identities reduce monomial
domination to the fixed certificates:
\begin{equation}\label{eq:monomial-domination}
\begin{aligned}
1-\mathbf x^{2\alpha}
&=\sum_{j=1}^t
(x_{i_1}\cdots x_{i_{j-1}})^2(1-x_{i_j}^2) \in\M_{2s+B_0},\\
1\pm\mathbf x^\alpha
&=\frac12(1\pm\mathbf x^\alpha)^2
+\frac12(1-\mathbf x^{2\alpha}) \in\M_{2s+B_0}.
\end{aligned}
\end{equation}
Each term in the first identity is a square
multiple of a certificate in
\eqref{eq:coordinate-certificate}.
Thus both identities yield certificates in
$\M_{2s+B_0}$, and in particular
$1\pm\mathbf x^\alpha\in\M_{2s+B_0}$.

Now write $e=\sum_\alpha e_\alpha\mathbf x^\alpha$.
The decomposition
\begin{equation}\label{eq:coefficient-decomposition}
\tau+e=(\tau-\|e\|_1)
+\sum_{e_\alpha\ne0}|e_\alpha|
\bigl(1+\operatorname{sgn}(e_\alpha)
\mathbf x^\alpha\bigr)
\end{equation}
expresses $\tau+e$ as a nonnegative
combination of the certificate in \eqref{eq:monomial-domination},
since $\tau\ge\|e\|_1$.
This proves \eqref{eq:domination}.

It remains to bound the size of the
constructed certificate.
In \eqref{eq:monomial-domination}, let $C_0$ bound the size minus one
of the fixed certificates for $1-x_i^2$.
Writing $q_j=x_{i_1}\cdots x_{i_{j-1}}$, multiplication by $q_j^2$
replaces each polynomial $P$ in an SOS decomposition by $q_jP$
and each equality multiplier
$\lambda$ by $q_j^2\lambda$. Both contributions to the certificate
size increase by a factor $4^{j-1}$, since
\[
 \|q_jP\|_2^2=4^{j-1}\|P\|_2^2,
 \qquad
 \|q_j^2\lambda\|_2=4^{j-1}\|\lambda\|_2.
\]
The explicit square in the second identity
contributes at most
$\frac12(1+2^t)^2$.
Consequently, the certificate for
$1\pm\mathbf x^\alpha$ satisfies
\[
\size-1\le
\frac12(1+2^t)^2
+\frac{C_0}{2}\sum_{j=0}^{t-1}4^j
\le C_1(s+1)4^s.
\]
Applying this estimate to the decomposition
\eqref{eq:coefficient-decomposition} gives
\[
\begin{aligned}
S
&\le 1+(\tau-\|e\|_1)
+C_1(s+1)4^s\sum_\alpha|e_\alpha|\\
&\le 1+C(s+1)4^s(\tau+\|e\|_1).
\end{aligned}
\]
For $s=0$, the claim follows directly
because $\tau+e$ is a nonnegative constant.
\end{proof}

\medskip
Coefficient domination turns a sufficiently small truncation error
into an exact module certificate. We now truncate the analytic functions in the SOS
decompositions and use their coefficient norms to control the error, and hence
the degree and size of the resulting polynomial certificate.

\begin{lemma}[Conversion to a polynomial certificate]
\label{lem:exactification}
Suppose $\deg f\le D$, $0<\tau\le1$, and
$f-\tau$ has an analytic certificate of size $\size$.
Then
\begin{equation}\label{eq:exactification}
d_\M(f)\le C\bigl(D+\log(\size/\tau)+1\bigr).
\end{equation}
The logarithm of the size of the resulting polynomial
certificate satisfies the same bound.
Here $C$ depends only on the fixed constraint polynomials
and Archimedean witness.
\end{lemma}

\begin{proof}
We truncate the analytic certificate for $f-\tau$
and use Lemma~\ref{lem:domination} to absorb
the resulting polynomial error. Let $F_L$ be obtained by truncating every analytic
function in the SOS decompositions and every equality multiplier to total degree $L$,
as defined above. Set
\[
G=\max\bigl(\{0\}\cup
\{\deg g_i\}_i\cup\{\deg h_j\}_j\bigr).
\]
Since truncation preserves the SOS structure,
$F_L\in\M_{2L+G}$.

Define the polynomial residual $E_L=f-\tau-F_L$.
By~\eqref{eq:square-truncation} and the Taylor
remainder estimate for equality multipliers,
\begin{equation}\label{eq:conversion-error}
\|E_L\|_1\le C_g2^{-L}\size,
\qquad
\deg E_L\le\max\{D,2L+G\},
\end{equation}
where
\[
C_g=2+\sum_i\|g_i\|_1+\sum_j\|h_j\|_1
\]
depends only on the fixed generators.

Choose
\[
L=D+\left\lceil
\log_2\frac{2C_g\size}{\tau}
\right\rceil,
\qquad s=2L+G.
\]
Then $\|E_L\|_1\le\tau/2<\tau$ and $\deg E_L\le s$.
By Lemma~\ref{lem:domination},
$\tau+E_L\in\M_{2s+B_0}$.
Consequently,
\begin{equation}\label{eq:exact-f}
f=F_L+(\tau+E_L)\in\M_{2s+B_0},
\end{equation}
which gives an exact polynomial certificate
in the original quadratic module.

It remains to bound the certificate size.
Taylor truncation does not increase the radius-two
norm of any factor or equality multiplier, so
the certificate for $F_L$ has size at most $\size$.
By Lemma~\ref{lem:domination}, the correction
$\tau+E_L$ has size at most
\[
1+C_{\mathrm{dom}}(s+1)4^s(\tau+\|E_L\|_1),
\]
where $C_{\mathrm{dom}}$ depends only on the
fixed constraint polynomials and Archimedean witness.
Combining these certificates and using
$\|E_L\|_1\le\tau/2$ and $\tau\le1$,
the resulting polynomial certificate has size at most
\[
\size+C_{\mathrm{dom}}(s+1)4^s
(\tau+\|E_L\|_1)
\le\size+\frac32 C_{\mathrm{dom}}(s+1)4^s.
\]

Since $s=2L+G$ and
\[
L=D+\left\lceil
\log_2\frac{2C_g\size}{\tau}
\right\rceil
=O\bigl(D+\log(\size/\tau)+1\bigr),
\]
the logarithmic size of this certificate satisfies
\[
\log\left(
\size+\frac32 C_{\mathrm{dom}}(s+1)4^s
\right)
\le C\bigl(D+\log(\size/\tau)+1\bigr).
\]
Moreover, by~\eqref{eq:exact-f},
\[
d_\M(f)\le 2s+B_0
=4L+2G+B_0
\le C\bigl(D+\log(\size/\tau)+1\bigr).
\]
\end{proof}

\section{Quantitative Stengle certificates}\label{sec:stengle}

Lemma~\ref{lem:exactification} reduces the degree upper-bound
problem to constructing analytic SOS certificates with
controlled size. We begin with Stengle certificates in
the preordering $\PP$, where products of inequality
generators are allowed. The effective Krivine--Stengle
Positivstellensatz~\cite[Theorem~1.4.2]{LPR}
provides a uniform degree bound for these certificates,
but does not directly control their size.
To obtain a quantitative size bound, we introduce a
worst-case optimal coefficient-size function, prove
its semialgebraicity, and use the polynomial growth
property of semialgebraic functions to show that
the optimal certificate size grows at most polynomially
in the inverse positivity margin $1/s$.

\paragraph{Proof strategy.}
For polynomials $a$ of degree at most $D$
satisfying $a|_K\ge s>0$, we first apply the
effective Krivine--Stengle Positivstellensatz
to establish a uniform degree bound $B$ for
their Stengle certificates, independently of
$a$ and $s$. We then establish local boundedness
of certificate size under sufficiently small
perturbations of $a$. Next, we introduce the worst-case optimal
coefficient-size function $F(s)$ and show
that it is finite and semialgebraic, which
implies that it grows at most polynomially
in $1/s$ as $s\to0^+$. Finally,
finite-dimensional norm equivalence transfers
this bound to the actual Stengle certificate size.

We use the preordering $\PP$ and its degree truncations $\PP_d$
from Section~\ref{sec:Pre}.

\begin{lemma}[Uniform Stengle certificates]
\label{lem:stengle}
For fixed $D$, there exist an integer $B$
and constants $A,\nu>0$, depending only on
the fixed constraint polynomials, the fixed
Archimedean certificate, and $D$, such that
whenever
\[
\deg a\le D,\quad \|a\|_2\le1/2,\quad
a|_K\ge s,\qquad 0<s\le1/4,
\]
there is an identity
\begin{equation}\label{eq:stengle}
ap=1+r,\qquad p,r\in\PP_B,
\end{equation}
with preordering representations satisfying
$\size_{p,r}\le As^{-\nu}$, where the
certificate size is measured as in
\eqref{eq:analytic-size}.
\end{lemma}

\begin{proof}
\emph{Uniform degree and local size bounds.}
We first establish a uniform degree bound
for Stengle certificates and then show that
their sizes are locally bounded.

Let
\[
\mathcal U_D
=\{a\in\mathbb R[x]_{\le D}:a|_K>0\}.
\]
For any $a_0\in\mathcal U_D$, the system
$g_i\ge0$, $h_j=0$, and $-a_0\ge0$
is infeasible.
By the effective Krivine--Stengle
Positivstellensatz~\cite[Theorem~1.4.2]{LPR},
there exist SOS polynomials
$\sigma_{\alpha,0},\sigma_{\alpha,1}$
and equality multipliers $\lambda_j$ such that
\begin{equation}\label{eq:stengle-refutation}
-1=
\sum_{\alpha\in\{0,1\}^m}
g^\alpha(\sigma_{\alpha,0}-a_0\sigma_{\alpha,1})
+\sum_{j=1}^\ell h_j\lambda_j,
\end{equation}
with all terms of degree at most $B_*$,
where $B_*$ depends only on the fixed
constraint polynomials and $D$.
Setting
\[
p_0=\sum_\alpha g^\alpha\sigma_{\alpha,1},
\qquad
r_0=\sum_\alpha g^\alpha\sigma_{\alpha,0}
+\sum_jh_j\lambda_j
\]
gives $a_0p_0=1+r_0$, with
$p_0,r_0\in\PP_{B_*}$.

We next show that this identity persists
with locally bounded certificate size
under perturbations of $a_0$.
For $a=a_0+\Delta$, where
$\Delta\in\mathbb R[x]_{\le D}$, we have
\begin{equation}\label{eq:perturbation}
a(2p_0)=1+2r_0+(1+2\Delta p_0).
\end{equation}
By submultiplicativity,
\[
\|2\Delta p_0\|_1
\le2\|\Delta\|_2\|p_0\|_2.
\]
Since $p_0$ is fixed,
$\|2\Delta p_0\|_1\le1/2$
whenever $\|\Delta\|_2$ is sufficiently small.
Moreover,
$\deg(\Delta p_0)\le D+B_*$.
Applying Lemma~\ref{lem:domination}
with $\tau=1$ yields
\[
1+2\Delta p_0\in\M_B\subseteq\PP_B,
\qquad B=2(D+B_*)+B_0,
\]
with uniformly bounded certificate size.
Taking $p=2p_0$ and
$r=2r_0+1+2\Delta p_0$ gives $ap=1+r$.
Since the certificates for $p_0,r_0$ are fixed,
there exist $\varepsilon_{a_0},C_{a_0}>0$
such that every $a\in\mathcal U_D$ satisfying
$\|a-a_0\|_2<\varepsilon_{a_0}$
admits a Stengle certificate in $\PP_B$
of size at most $C_{a_0}$.
In particular, the degree budget $B$
is independent of $a_0$.

\medskip
\noindent\emph{Semialgebraic growth of coefficient size.}
To obtain a uniform size bound, we first introduce
a coefficient-size function, prove its
semialgebraicity, and control its growth
as the positivity margin decreases.
In the representations below, multipliers attached to identically
zero weights $g^\alpha$ or identically zero equality generators $h_j$
are set to zero. Removing these terms preserves the identity and the
degree budget and cannot increase either certificate size or the
coefficient size defined below. This convention concerns polynomial
identities, not generators that merely vanish on $K$.
For $a\in\mathcal U_D$, let
$\mathcal C_B(a)$ denote the set of Stengle
certificates $c$ consisting of an identity
$ap=1+r$ with specified degree-$B$
preordering representations
\[
t=\sum_{\alpha\in\{0,1\}^m}
g^\alpha\sigma_\alpha^t
+\sum_{j=1}^{\ell}h_j\lambda_j^t,
\qquad t\in\{p,r\}.
\]
Their certificate size $\size_{p,r}$
is measured as in \eqref{eq:analytic-size}.
For the semialgebraic argument, we define
the corresponding coefficient size by
replacing the squared SOS-factor norms
with the norms of the SOS polynomials:
\begin{equation}\label{eq:coefficient-size}
L(c)=1+\sum_{t\in\{p,r\}}
\left(
\sum_{\alpha\in\{0,1\}^m}
\|\sigma_\alpha^t\|_2
+\sum_{j=1}^{\ell}\|\lambda_j^t\|_2
\right).
\end{equation}
For each SOS polynomial
$\sigma_\alpha^t=\sum_j(P_{\alpha,j}^t)^2$,
submultiplicativity gives
\[
\|\sigma_\alpha^t\|_2
\le\sum_j\|P_{\alpha,j}^t\|_2^2.
\]
Consequently,
$L(c)\le\size_{p,r}$,
so the preceding local certificate-size bounds
also apply to $L(c)$. For $0<s\le1/4$, define
\begin{equation}\label{eq:Ss}
\begin{gathered}
S_s=\{a\in\mathbb R[x]_{\le D}:
\|a\|_2\le1/2,\ a|_K\ge s\},\\
F(s)=\sup_{a\in S_s}
\inf_{c\in\mathcal C_B(a)}L(c).
\end{gathered}
\end{equation}
Thus $F(s)$ is the worst-case optimal
coefficient size over $S_s$.

\medskip
We next establish that $F(s)$ is finite and semialgebric.
For every $0<s\le1/4$, the set $S_s$
is nonempty, since $1/2\in S_s$.
It is also closed and bounded in the
finite-dimensional space
$\mathbb R[x]_{\le D}$, hence compact.
Moreover, $S_s\subseteq\mathcal U_D$
because $s>0$.
The preceding local bounds, together with
a finite subcover of $S_s$, therefore give
\[
1\le F(s)<\infty,
\qquad 0<s\le1/4.
\]

Since $B$ is fixed, each SOS polynomial
in a degree-$B$ certificate admits a Gram
matrix representation of fixed dimension.
Positive semidefiniteness, the degree
constraints, and the coefficient identities
in $ap=1+r$ are semialgebraic conditions.
The coefficient norms in
\eqref{eq:coefficient-size} are also
semialgebraic, since they are finite
weighted sums of absolute coefficient values.
Thus the family of pairs $(a,c)$ with
$c\in\mathcal C_B(a)$, together with
$L(c)$, is semialgebraically definable. Likewise, $S_s$ is semialgebraically
definable: the condition $a|_K\ge s$
is a first-order formula over the reals
involving only polynomial inequalities
and equalities. By the definition of supremum and infimum,
\begin{equation}\label{eq:F-epigraph}
F(s)\le t
\quad\Longleftrightarrow\quad
\forall a\in S_s,\ \forall\varepsilon>0,\
\exists c\in\mathcal C_B(a):
L(c)<t+\varepsilon.
\end{equation}
Hence real quantifier elimination~\cite{BPR}
implies that the epigraph
\[
\{(s,t):0<s\le1/4,\ F(s)\le t\}
\]
is semialgebraic, and therefore $F$
is a semialgebraic function.

We then establish the polynomial growth of $F(s)$. By the polynomial growth property of
semialgebraic functions~\cite{BPR},
there exist $A_0,\nu>0$ and $s_0>0$
such that
$F(s)\le A_0s^{-\nu}$ for $0<s\le s_0$.
Since $S_{s_2}\subseteq S_{s_1}$
whenever $s_1<s_2$, the function $F$
is nonincreasing.
Enlarging $A_0$ if necessary therefore
extends the bound to the whole interval:
\begin{equation}\label{eq:F-growth}
F(s)\le A_0s^{-\nu},
\qquad 0<s\le1/4.
\end{equation}

We finally establish the size bound of certificate based on $F(s)$. For any $a\in S_s$, choose
$c\in\mathcal C_B(a)$ satisfying
$L(c)\le2F(s)$.
Since $B$ is fixed, all polynomials in the SOS
decompositions have uniformly bounded degree.
Finite-dimensional norm equivalence
therefore gives a constant $C>0$,
depending only on $B$ and $n$, such that
every SOS polynomial
$\sigma=\sum_{j=1}^N P_j^2$
in the certificate satisfies
\begin{equation}\label{eq:sos-factor-bound}
\sum_{j=1}^N\|P_j\|_2^2
\le C\sum_{j=1}^N
\int_{[-1,1]^n}P_j(x)^2\,dx=C\int_{[-1,1]^n}\sigma(x)\,dx\le2^nC\|\sigma\|_2.
\end{equation}
The last inequality follows from
$\|\sigma\|_{[-1,1]^n,\infty}
\le\|\sigma\|_2$.
The constant $C$ is independent
of the number of squares in the SOS decompositions.

Applying \eqref{eq:sos-factor-bound} to
all SOS polynomials in the representations
of $p$ and $r$, and including the equality
multipliers, yields
\[
\size_{p,r}
\le C'L(c)
\le2C'F(s)
\le As^{-\nu},
\]
where the final inequality follows from
\eqref{eq:F-growth}.
All constants depend only on the fixed
constraint polynomials, the fixed
Archimedean certificate, and $D$.
This completes the proof.
\end{proof}

\begin{remark}[Growth as the margin vanishes]
The coefficient-size function $F(s)$
necessarily diverges as $s\to0^+$.
Indeed, take the constant polynomial
$a_s=s\in S_s$ and fix $x_*\in K$.
For any Stengle certificate
$sp=1+r$, with $p,r\in\PP_B$,
nonnegativity of $r$ on $K$ gives
$
p(x_*)=\frac{1+r(x_*)}{s}\ge\frac1s.
$
On the other hand, the preordering
representation of $p$ gives
$
p(x_*)=\sum_\alpha
g^\alpha(x_*)\sigma_\alpha^p(x_*).
$
Since $K\subseteq[-1,1]^n$, there is
a constant $C_*>0$, depending only
on the fixed generators and $x_*$,
such that
$p(x_*)\le C_*
\sum_\alpha\|\sigma_\alpha^p\|_2
\le C_*L(c).$
Consequently,
$F(s)\ge\frac{1}{C_*s}.$
Together with \eqref{eq:F-growth},
this yields
$
\frac{1}{C_*s}\le F(s)\le A_0s^{-\nu}, 0<s\le1/4.
$
\end{remark}

This section shows that every polynomial
$a$ of degree at most $D$, satisfying
$\|a\|_2\le1/2$ and $a|_K\ge s$ for
$0<s\le1/4$, admits a Stengle identity
\[
ap_0=1+r_0,\qquad p_0,r_0\in\PP_B,
\]
with certificate size at most $As^{-\nu}$.
Here $B$, $A$, and $\nu$ depend only on the
fixed constraint polynomials, the fixed
Archimedean certificate, and $D$.
In particular, the certificate degree is
uniformly bounded independently of $a$ and $s$,
while its size grows at most polynomially
in $1/s$.

\section{Removing generator products}\label{sec:products}
Given the Stengle identity $ap_0=1+r_0$ with
$p_0,r_0\in\PP$ and controlled certificate size,
we aim in this section to eliminate the products
of inequality generators appearing in the
preordering representations of $p_0$ and $r_0$,
since such products need not belong to the original
quadratic module $\M$.
To this end, we construct a common polynomial
multiplier $H_\rho$ such that every
$H_\rho(g^\alpha+\rho)$ admits an analytic
certificate in $\M(\An_2)$ of size at most
$\exp(O(\log^2(2/\rho)))$, with $H_\rho-1$
a sum of squares.
In Section~\ref{sec:upper}, we apply this
conversion to $p_0$ and $r_0$ to obtain
\[
ap=1+r,\qquad p,r\in\M(\An_2),
\]
with controlled analytic certificate size.

\subsection{Interval SOS factorization}
\label{subsec:products-interval}

For univariate polynomials $R$
nonnegative on $[-1,1]$, the conversion of generator products in
Subsection~\ref{subsec:products-conversion}
requires analytic SOS representations of
$R(g_i)$. In general, such polynomial admits a representation
\[
R(t)=\sigma_0(t)+(1-t^2)\sigma_1(t),
\]
where $\sigma_0$ and $\sigma_1$ are sums
of squares of polynomials.
To obtain an analytic SOS representation
after substituting $t=a$, we first show
that $1-a^2$ is a square in $\An_2$,
with a bound on the norm of its square root.

For $\gamma>0$ and $\Delta\in\An_2$ with
$\|\Delta\|_2\le\gamma/2$, the binomial series
\[
B_\Delta=\sqrt\gamma\sum_{j=0}^{\infty}
\binom{1/2}{j}
\left(\frac{\Delta}{\gamma}\right)^j
\]
converges in $\An_2$ and satisfies
$B_\Delta^2=\gamma+\Delta$.
By submultiplicativity,
\begin{equation}\label{eq:analytic-root}
\|B_\Delta\|_2
\le\sqrt\gamma\sum_{j=0}^{\infty}
\left|\binom{1/2}{j}\right|2^{-j}
<2\sqrt\gamma.
\end{equation}
In particular, taking $\gamma=1$ and
$\Delta=-a^2$, we obtain
$1-a^2=b_a^2$ for some $b_a\in\An_2$
with $\|b_a\|_2\le2$ whenever
$\|a\|_2\le1/2$.
Combining this analytic square root with
quantitative bounds for the interval SOS
factorization yields the following lemma.

\begin{lemma}[Interval functional calculus]\label{lem:interval}
Let $R\in\R[t]$ be nonnegative on $[-1,1]$, with
$\deg R\le N$, and let $a\in\An_2$ satisfy $\|a\|_2\le1/2$.
Then $R(a)$ has an analytic SOS representation
$R(a)=\sum_j\psi_j^2$, with $\psi_j\in\An_2$, such that
\begin{equation}\label{eq:interval-size}
\sum_j\|\psi_j\|_2^2\le\exp(C(N+1))\|R\|_1,
\end{equation}
where $C$ is an absolute constant.
\end{lemma}

\begin{proof}
Since $R(\cos\theta)$ is a nonnegative trigonometric polynomial
of Fourier degree at most $N$, Fej\'er--Riesz
factorization~\cite{FR} gives
$Q(z)=\sum_{j=0}^Nq_jz^j$ satisfying
$R(\cos\theta)=|Q(e^{\mathrm i\theta})|^2$.
Using \eqref{eq:chebyshev-identities}, define
\[
V(t)=\sum_{j=0}^Nq_jT_j(t),\qquad
W(t)=\mathrm i\sum_{j=1}^Nq_jU_{j-1}(t).
\]
Then $Q(e^{\pm\mathrm i\theta})=V(t)\pm\sin\theta W(t)$
for $t=\cos\theta$. Averaging their squared moduli and writing
$V=A_1+\mathrm iA_2$, $W=B_1+\mathrm iB_2$ gives
\begin{equation}\label{eq:interval-rep}
R(t)=\sum_{j=1}^2 A_j(t)^2
+(1-t^2)\sum_{j=1}^2 B_j(t)^2.
\end{equation}
By \eqref{eq:chebyshev-basic-coef}, Cauchy--Schwarz, and Parseval,
\[
\begin{aligned}
\sum_{j=1}^2\bigl(\|A_j\|_1^2+\|B_j\|_1^2\bigr)
&\le4(N+1)3^{2N}\sum_{j=0}^N|q_j|^2\\
&=\frac{4(N+1)3^{2N}}{2\pi}
\int_0^{2\pi}R(\cos\theta)\,d\theta\\
&\le\exp(C(N+1))\|R\|_1.
\end{aligned}
\]
Now $1-a^2=b_a^2$ with $\|b_a\|_2\le2$, so
\[
R(a)=\sum_{j=1}^2 A_j(a)^2
+\sum_{j=1}^2\bigl(b_aB_j(a)\bigr)^2.
\]
Since $\|a\|_2\le1/2$,
$\|A_j(a)\|_2\le\|A_j\|_1$ and
$\|b_aB_j(a)\|_2\le2\|B_j\|_1$.
The preceding bound proves \eqref{eq:interval-size}.
\end{proof}

\subsection{Rational splitting and denominator control}
\label{subsec:products-splitting}
To eliminate products of inequality generators,
we seek a decomposition in which each term
contains at most one generator multiplied by
nonnegative factors.
The following lemma provides rational functions
$A,B,C,E$, nonnegative on $[-1,1]$, satisfying
$t=A-B$ and $-B=tC+E-cA-2b$.
These identities allow us to rewrite the
negative terms arising from product expansions,
with the contributions from $cA$ and $2b$
absorbed by positive terms and the perturbation
$\rho$, respectively.

The lemma also constructs a common polynomial
denominator $d$ such that all the cleared factors
are nonnegative on $[-1,1]$, with controlled
degree and coefficient norms.
Moreover, $d$ and $d-1$ are sums of squares.
After substituting $t=g_i$, these factors
admit analytic SOS representations by
Lemma~\ref{lem:interval}.

\begin{lemma}[Quantitative rational splitting]
\label{lem:rational-splitting}
Fix $0<c\le1/2$ and an integer $k\ge1$, and put $b=2^{-k}$.
There exist rational functions $A,B,C,E$ and an even polynomial
$d\in\R[t]$ with nonnegative even coefficients and constant
coefficient one such that
\begin{equation}\label{eq:scalar-split}
t=A-B,\qquad -B=tC+E-cA-2b,
\end{equation}
where, on $[-1,1]$,
\[
0\le A,B\le2,\qquad 0\le C\le2,\qquad E\ge0.
\]
For every
\begin{equation}\label{eq:scalar-list}
Z\in\mathcal Z:=\{1,A,B,C,E,2-A,2-B\},
\end{equation}
the product $dZ$ is a polynomial nonnegative on $[-1,1]$, and
\begin{equation}\label{eq:rational-splitting-bounds}
\deg(dZ)\le C_c(k+1),\qquad
\|dZ\|_1\le\exp\bigl(C_c(k+1)^2\bigr),
\end{equation}
where $C_c$ is a constant depends only on $c$. In particular,
$d,d-1\in\SOS[t]$.
\end{lemma}

\begin{proof}
\medskip
\noindent\emph{Rational majorant.}
Following Newman's construction for rational
approximation of the absolute value function
~\cite{Newman1964}, we construct an even rational majorant of $|t|$ satisfies $|t|\le\omega(t)\le\left(1+\frac c2\right)|t|+b$ on $t\in[-1,1]$.
Choose an integer $J=J(c)\ge1$ such that
\[
\zeta=3^{-J},\qquad
\lambda=\frac{1+\zeta}{1-\zeta},\qquad
\lambda^2\le1+\frac c2.
\]
Define
\begin{equation}\label{eq:filter}
\begin{gathered}
\Phi(t)=\prod_{j=0}^k(t+2^{-j})^J,\\[3pt]
\begin{aligned}
Q(t)&=\Phi(t)+\Phi(-t), \qquad
&U(t)&=t\bigl(\Phi(t)-\Phi(-t)\bigr),\\
P(t)&=\lambda U(t)+bQ(t),
&\omega(t)&=\frac{P(t)}{Q(t)}.
\end{aligned}
\end{gathered}
\end{equation}
The polynomials $Q,U,P$ have nonnegative even coefficients and
positive constant terms for $Q,P$. Hence $P,Q>0$ on $\R$ and
$\omega$ is well defined and even.

Given the definition above, we check the condition
\begin{equation}\label{eq:relative}
|t|\le\omega(t)\le\left(1+\frac c2\right)|t|+b
\qquad(-1\le t\le1).
\end{equation}
By evenness, assume $t\ge0$. For $b\le t\le1$, there exist some dyadic
node $v=2^{-j}$ satisfies $v\le t\le2v$. Thus
\[
\left|\frac{\Phi(-t)}{\Phi(t)}\right|
\le\left(\frac{|v-t|}{v+t}\right)^J
\le3^{-J}=\zeta,
\]
which implies $t/\lambda\le U(t)/Q(t)\le\lambda t$.
Consequently $t\le\omega(t)\le\lambda^2t+b$.
For $0\le t\le b$, all factors of $\Phi(-t)$ are
nonnegative and $0\le\Phi(-t)\le\Phi(t)$, so
$0\le U(t)/Q(t)\le t$. Hence
$b\le\omega(t)\le\lambda t+b$.
The choice of $\lambda$ proves \eqref{eq:relative}.

\medskip
\noindent\emph{Rational sign decomposition.}
Using the rational majorant $\omega$, we construct
nonnegative rational functions $A,B,C,E$ satisfying
\eqref{eq:scalar-split}. The first identity decomposes
$t$ into a difference of nonnegative functions, while
the second expresses $-B$ as $tC+E-(cA+2b)$,
where $C,E\ge0$.
Set
\begin{equation}\label{eq:ABCE}
\begin{aligned}
A&=\frac{(\omega+t)^2}{4\omega},&
B&=\frac{(\omega-t)^2}{4\omega},\\
C&=\frac{\omega-t}{\omega},&
E&=\frac{t^2}{\omega}-(1-c)A+2b.
\end{aligned}
\end{equation}
Direct calculation verifies \eqref{eq:scalar-split}.
By \eqref{eq:relative}, $0\le A,B\le\omega<2$
and $0\le C\le2$ on $[-1,1]$.
Moreover, since $A\le\omega$ and
$t^2/\omega\ge2|t|-\omega$, we obtain
\begin{equation}\label{eq:E-positive}
E\ge2|t|-(2-c)\omega+2b
\ge\frac{c^2}{2}|t|+cb\ge0.
\end{equation}
Thus all functions in $\mathcal Z$ defined in
\eqref{eq:scalar-list} are nonnegative on $[-1,1]$.

\medskip
\noindent\emph{Common denominator.}
We construct a common polynomial denominator $d$
such that every $dZ$, with $Z\in\mathcal Z$,
is polynomial and nonnegative on $[-1,1]$,
while $d,d-1\in\SOS[t]$.
Set
\[
D_0=P(0)Q(0)>0,\qquad
d(t)=\frac{P(t)Q(t)}{D_0}.
\]
Substituting $\omega=P/Q$ into \eqref{eq:ABCE}
gives
\begin{equation}\label{eq:all-cleared-factors}
\begin{aligned}
dA&=\frac{(P+tQ)^2}{4D_0},&
dB&=\frac{(P-tQ)^2}{4D_0},\\
dC&=\frac{Q(P-tQ)}{D_0},&
dE&=\frac{t^2Q^2-\frac{1-c}{4}(P+tQ)^2+2bPQ}{D_0}.
\end{aligned}
\end{equation}
The remaining cases follow from
$d(2-A)=2d-dA$, $d(2-B)=2d-dB$,
and $d\cdot1=d$.
Thus $dZ\in\R[t]$ for every $Z\in\mathcal Z$.
Since $d>0$ on $\R$ and each $Z$ is nonnegative
on $[-1,1]$, every $dZ$ is nonnegative on
$[-1,1]$.
Moreover, $P,Q$ have nonnegative even coefficients
and $d(0)=1$, so $d,d-1\in\SOS[t]$.

\noindent\emph{Coefficient bounds.} It remains to establish the degree and coefficient
bounds. Since $\Phi$ has nonnegative coefficients,
\[
\deg\Phi=J(k+1),\qquad
\|\Phi\|_1=\Phi(1)\le e^{2J},\qquad
\Phi(0)=2^{-Jk(k+1)/2}.
\]
Also, $\Phi(-1)=0$, so
$\|Q\|_1=Q(1)=\Phi(1)$ and
$\|P\|_1=P(1)=(\lambda+b)\Phi(1)=O_c(1)$.
Since $Q(0)=2\Phi(0)$ and $P(0)=bQ(0)$,
\begin{equation}\label{eq:denominator-height}
D_0=4b\Phi(0)^2=2^{2-k-Jk(k+1)},\qquad
D_0^{-1}=2^{Jk(k+1)+k-2}.
\end{equation}
The expressions in \eqref{eq:all-cleared-factors},
together with submultiplicativity, imply that
uniformly over $Z\in\mathcal Z$,
\begin{equation}\label{eq:cleared-height}
\deg(dZ)\le2J(k+1)+2,\qquad
\|dZ\|_1\le C'_cD_0^{-1}
\le\exp\bigl(C_c(k+1)^2\bigr).
\end{equation}
Enlarging $C_c$ if necessary gives
\eqref{eq:rational-splitting-bounds}.
\end{proof}

\subsection{Conversion of generator products}
\label{subsec:products-conversion}
We now combine the preceding results to construct
a common polynomial multiplier $H_\rho$ such that
every $H_\rho(g^\alpha+\rho)$ admits an analytic
certificate in $\M(\An_2)$ with controlled size,
while $H_\rho-1$ is a sum of squares.

\paragraph{Proof strategy.} Expanding $g^\alpha$ using $t_i=A_i-B_i$
produces products of nonnegative rational
functions with both positive and negative signs.
For each negative term, we apply
Lemma~\ref{lem:rational-splitting} to replace
a factor $-B_i$ by $t_iC_i+E_i-cA_i-2b$.
The terms involving $t_iC_i$ and $E_i$
have the desired structure, while the negative
contributions from $cA_i$ are absorbed by
the positive terms of the expansion, and
those from $2b$ are compensated by $\rho$.
Finally, a common polynomial multiplier clears
the rational denominators, and
Lemma~\ref{lem:interval} provides analytic
SOS representations for the resulting
nonnegative factors.

Before presenting the general proof, we illustrate
the main idea with a product of two factors.
The proof uses two expansions: the first applies
$t=A-B$ to expand the product, while the second
uses the identity for $-B$ to rewrite the negative
terms.

\begin{example}[Two-factor illustration]
\label{ex:two-factor}
Let $0<c\le1/4$ and choose $k\ge1$ such that
$b=2^{-k}\le\rho/8$.
For $i=1,2$, write $A_i=A(t_i)$ and
$B_i=B(t_i)$, with analogous notation for
$C_i$ and $E_i$.
Using $t_i=A_i-B_i$, we obtain
\[
t_1t_2
=A_1A_2-B_1A_2-A_1B_2+B_1B_2.
\]
By the second identity in \eqref{eq:scalar-split},
\[
\begin{aligned}
-B_1A_2
&=(t_1C_1+E_1)A_2-cA_1A_2-2bA_2,\\
-A_1B_2
&=(t_2C_2+E_2)A_1-cA_1A_2-2bA_1.
\end{aligned}
\]
Combining these identities gives
\begin{equation}\label{eq:two-factor}
\begin{aligned}
t_1t_2+\rho
={}&(1-2c)A_1A_2+B_1B_2\\
&+(t_1C_1+E_1)A_2
 +(t_2C_2+E_2)A_1\\
&+\rho-2b(A_1+A_2).
\end{aligned}
\end{equation}
Since $c\le1/4$, we have $1-2c\ge1/2$.
Moreover, the remaining term satisfies
\[
\rho-2b(A_1+A_2)
=(\rho-8b)+2b(2-A_1)+2b(2-A_2),
\]
where every term on the right is nonnegative
on $[-1,1]^2$.

Thus $t_1t_2+\rho$ is expressed as a sum of
nonnegative rational products and terms of
the form $t_i$ times a nonnegative rational
product. After clearing the denominators
and substituting $t_i=g_i$, these terms
yield an analytic quadratic-module
representation by Lemma~\ref{lem:interval}.
\end{example}

To apply the interval lemma, we positively rescale the fixed
nonzero inequality generators so that $\|g_i\|_2\le1/2$.
Such rescaling preserves the quadratic module, the preordering,
and all degree budgets; it affects certificate sizes only through
fixed constants, including those in Lemma~\ref{lem:stengle}.
All statements below use these rescaled generators. The next
lemma records the resulting uniform conversion and its size bound.

\begin{lemma}[Product conversion]\label{lem:products}
For $0<\rho\le1/4$, put
$k=\lceil\log_2(4^m/\rho)\rceil$.
There exists a polynomial $H$ of degree $O_g(k)$
such that $H,H-1\in\SOS[x]$ and
\begin{equation}\label{eq:products}
H(g^\alpha+\rho)\in\M(\An_2)
\qquad(\forall \alpha\in\{0,1\}^m).
\end{equation}
The analytic quadratic-module certificates
in \eqref{eq:products} and the polynomial
SOS certificates for $H$ and $H-1$
all have size at most $\exp(C_gk^2)$.
In particular, $\|H\|_2\le\exp(C_gk^2)$.
The constants depend only on the fixed
inequality generators.
\end{lemma}

\begin{proof}
If $m\le1$, take $H=1$.
Then $1+\rho$ and, when $m=1$,
$g_1+\rho=g_1\cdot1^2+(\sqrt\rho)^2$
have certificates of bounded size.
Assume henceforth that $m\ge2$.
Set $c=1/(2m)$ and $b=2^{-k}\le\rho/4^m$,
and take $A,B,C,E,d$ from
Lemma~\ref{lem:rational-splitting}.

\medskip
\noindent\emph{Product decomposition.}
Fix $\alpha\in\{0,1\}^m$ with nonempty
support $J=\operatorname{supp}(\alpha)$
of size $\ell\le m$.
Relabel the corresponding generators as
$t_1,\ldots,t_\ell$, so that
$g^\alpha=\prod_{j=1}^{\ell}t_j$.
Write $[\ell]=\{1,\ldots,\ell\}$ and
$A_j=A(t_j)$, $B_j=B(t_j)$, with analogous
notation for $C_j$ and $E_j$. Expanding $t_j=A_j-B_j$, let
$S\subseteq[\ell]$ denote the indices
for which $-B_j$ is selected. Then
\begin{equation}\label{eq:product-expansion}
\prod_{j=1}^{\ell}t_j
=\sum_{S\subseteq[\ell]}(-1)^{|S|}V_S,
\qquad
V_S:=\prod_{j\in S}B_j
     \prod_{j\in[\ell]\setminus S}A_j.
\end{equation}
For each odd subset $S$, put
$i_S=\min(S)$ and define
\[
W_S:=
\prod_{j\in S\setminus\{i_S\}}B_j
\prod_{j\in[\ell]\setminus S}A_j.
\]
By the second identity in \eqref{eq:scalar-split},
each negative term in \eqref{eq:product-expansion}
can be rewritten as
\[
-V_S=-B_{i_S}W_S
=(t_{i_S}C_{i_S}+E_{i_S}
-cA_{i_S}-2b)W_S.
\]
Substituting this identity into
\eqref{eq:product-expansion} gives
\[
\begin{aligned}
\prod_{j=1}^{\ell}t_j+\rho
={}&\sum_{|T|\,\mathrm{even}}V_T
+\sum_{|S|\,\mathrm{odd}}
(t_{i_S}C_{i_S}+E_{i_S})W_S\\
&-c\sum_{|S|\,\mathrm{odd}}A_{i_S}W_S
+\rho-2b\sum_{|S|\,\mathrm{odd}}W_S.
\end{aligned}
\]

We handle the two negative terms separately.
First, we absorb the terms $-cA_{i_S}W_S$
into the positive terms $V_T$.
For each odd subset $S$, let
$T=S\setminus\{i_S\}$.
Then $T$ is even and $A_{i_S}W_S=V_T$.
For each even subset $T$, let $n_T$
denote the number of odd subsets $S$
satisfying $S\setminus\{i_S\}=T$.
Grouping the terms according to $T$ gives
\[
\sum_{|S|\,\mathrm{odd}}A_{i_S}W_S
=\sum_{|T|\,\mathrm{even}}n_TV_T.
\]
Thus, setting $a_T=1-cn_T$, we obtain
\begin{equation}\label{eq:product-decomposition}
\begin{gathered}
\prod_{j=1}^{\ell}t_j+\rho
=\sum_{|T|\,\mathrm{even}}a_TV_T
+\sum_{|S|\,\mathrm{odd}}
(t_{i_S}C_{i_S}+E_{i_S})W_S+\mathcal R,\\
\mathcal R=\rho-2b\sum_{|S|\,\mathrm{odd}}W_S.
\end{gathered}
\end{equation}Every odd subset $S$ counted by $n_T$
has the form $S=T\cup\{i\}$ with
$i\notin T$. Hence
$n_T\le\ell-|T|\le m$, giving
\[
a_T=1-cn_T\ge\frac12.
\]

We next express
$\mathcal R$ as a sum of nonnegative
rational products. For the remainder $\mathcal R$, since
$0\le A_j,B_j\le2$, each $W_S$ is bounded
by $2^{\ell-1}$.
As there are $2^{\ell-1}$ odd subsets,
we obtain
\[
\mathcal R
=(\rho-2^{2\ell-1}b)
+2b\sum_{|S|\,\mathrm{odd}}
(2^{\ell-1}-W_S),
\]
where $\rho-2^{2\ell-1}b\ge\rho/2$
by $b\le\rho/4^m$ and $\ell\le m$. Each $W_S$ is a product of $\ell-1$
factors, each equal to some $A_j$ or $B_j$,
with distinct indices.
The telescoping identity
\[
2^p-\prod_{j=1}^p Z_j
=\sum_{j=1}^p2^{p-j}(2-Z_j)
\prod_{r<j}Z_r
\]
expresses $2^{\ell-1}-W_S$ as a sum of
nonnegative products whose factors belong
to $\mathcal Z$ in \eqref{eq:scalar-list}.
When $\ell=1$, we have $W_S=1$, so
$2^{\ell-1}-W_S=0$.

Consequently, $g^\alpha+\rho$ is a nonnegative
linear combination of terms of the form
\[
\prod_{j\in I}Z_j(g_j)
\quad\text{and}\quad
g_i\prod_{j\in I}Z_j(g_j),
\]
where $I\subseteq J$, $i\in J$,
$Z_j\in\mathcal Z$ and each index $j$ occurs at most once
in these products

\medskip
\noindent\emph{Common multiplier and analytic SOS conversion.}
We now construct a common polynomial multiplier
$H$ to clear the rational denominators arising
from the preceding product decomposition. Based on $\eqref{lem:rational-splitting}$, define the common polynomial multiplier
\begin{equation}\label{eq:one-denominator}
H=\prod_{i=1}^m d(g_i).
\end{equation}
Since $d$ has nonnegative even coefficients
and $d(0)=1$, both $d(g_i)$ and
$d(g_i)-1$ are polynomial SOS.
Moreover,
\[
H-1=\sum_{i=1}^m
(d(g_i)-1)\prod_{j<i}d(g_j),
\]
so $H,H-1\in\SOS[x]$. For every $Z\in\mathcal Z$,
Lemma~\ref{lem:rational-splitting} gives
a polynomial $dZ$ nonnegative on $[-1,1]$,
with
\[
\deg(dZ)=O_m(k),\qquad
\|dZ\|_1\le\exp(O_m(k^2)).
\]
Since $\|g_i\|_2\le1/2$,
Lemma~\ref{lem:interval} yields
\begin{equation}\label{eq:cleared-analytic-sos}
d(g_i)Z(g_i)=\sum_j\psi_{i,Z,j}^2,
\qquad
\sum_j\|\psi_{i,Z,j}\|_2^2
\le\exp(O_g(k^2)).
\end{equation}

By the product decomposition, each term is
a nonnegative multiple of
$\prod_{j\in I}Z_j(g_j)$ or
$g_i\prod_{j\in I}Z_j(g_j)$,
where $I\subseteq J$, $i\in J$,
and $Z_j\in\mathcal Z$.
Multiplying each product by
$\prod_{j\in J}d(g_j)$ gives
\[
\left(\prod_{j\in J}d(g_j)\right)
\left(\prod_{j\in I}Z_j(g_j)\right)
=
\prod_{j\in I}(dZ_j)(g_j)
\prod_{j\in J\setminus I}d(g_j).
\]
Every factor on the right is an analytic
SOS by \eqref{eq:cleared-analytic-sos}
and the SOS property of $d(g_j)$.
Since products of SOS are SOS,
terms without an explicit generator $g_i$
yield analytic SOS, while terms containing
$g_i$ yield $g_i$ multiplied by an analytic SOS.

Consequently,
\[
\left(\prod_{i\in J}d(g_i)\right)
(g^\alpha+\rho)\in\M(\An_2).
\]
Multiplication by the remaining SOS polynomial
$\prod_{i\notin J}d(g_i)$
preserves membership in $\M(\An_2)$.
Hence
\[
H(g^\alpha+\rho)\in\M(\An_2).
\]
For $\alpha=0$, $H(1+\rho)$ is already
polynomial SOS. Thus \eqref{eq:products}
holds for all $\alpha\in\{0,1\}^m$.

\medskip
\noindent\emph{Degree and certificate-size bounds.}
For two SOS representations,
submultiplicativity gives
\[
\sum_{j,l}\|A_jB_l\|_2^2
\le
\left(\sum_j\|A_j\|_2^2\right)
\left(\sum_l\|B_l\|_2^2\right).
\]
The sums $\sum_{|T|\,\mathrm{even}}a_TV_T$
and $\sum_{|S|\,\mathrm{odd}}
(t_{i_S}C_{i_S}+E_{i_S})W_S$
in \eqref{eq:product-decomposition} contribute
$2^{\ell-1}$ and $2^\ell$ terms, respectively,
after expanding the latter.
The remainder $\mathcal R$ contributes at most
$1+(\ell-1)2^{\ell-1}$ terms after applying
the telescoping identity.
Thus the total number of terms is at most
$(\ell+2)2^{\ell-1}+1=O(m2^m)$.
Each term contains at most $\ell\le m$
factors from $\mathcal Z$, possibly multiplied
by a single inequality generator, and its
nonnegative scalar coefficient is bounded
in terms of $m$.
Since $m$ is fixed, \eqref{eq:cleared-analytic-sos}
and the preceding product estimate show that,
uniformly over $\alpha\in\{0,1\}^m$,
the constructed analytic certificate for
$H(g^\alpha+\rho)\in\M(\An_2)$
has size at most $\exp(C_gk^2)$.

Next, by Lemma~\ref{lem:rational-splitting},
we may write
$d(t)=1+\sum_{r\ge1}d_{2r}t^{2r}$
with $d_{2r}\ge0$. Hence
\[
d(g_i)=\sum_{r\ge0}
(\sqrt{d_{2r}}\,g_i^r)^2,
\qquad
d(g_i)-1=\sum_{r\ge1}
(\sqrt{d_{2r}}\,g_i^r)^2,
\]
where $d_0=1$.
By submultiplicativity and $\|g_i\|_2\le1/2$,
\[
\begin{aligned}
\sum_{r\ge0}\|\sqrt{d_{2r}}\,g_i^r\|_2^2
&\le\sum_{r\ge0}d_{2r}\|g_i\|_2^{2r}\le\|d\|_1
\le\exp(O_m(k^2)).
\end{aligned}
\]
Using the preceding SOS product estimate
and the telescoping identity for $H-1$,
we obtain
\[
\size_H\le1+\prod_{i=1}^m B_i\le1+\|d\|_1^m,\qquad
\size_{H-1}\le1+\sum_{i=1}^m(B_i-1)\prod_{j<i}B_j
=\prod_{i=1}^m B_i\le\|d\|_1^m,
\]
where
$B_i=\sum_{r\ge0}d_{2r}\|g_i^r\|_2^2
\le\|d\|_1$.
Both bounds are at most $\exp(C_gk^2)$.

Finally, by submultiplicativity,
\[
\begin{aligned}
\|H\|_2
&\le\prod_{i=1}^m\|d(g_i)\|_2
\le\|d\|_1^m
\le\exp(C_gk^2),\\
\deg H
&\le\deg d\sum_{i=1}^m\deg g_i
=O_g(k).
\end{aligned}
\]
Since $k=O_g(\log(2/\rho))$,
this completes the proof.
\end{proof}

\section{From Stengle certificates to Putinar certificates}\label{sec:upper}

Section~\ref{sec:stengle} provides a Stengle identity
$ap_0=1+r_0$ with $p_0,r_0\in\PP$ with controlled
certificate size. In this section, we first apply
the product-conversion result of Section~\ref{sec:products}
to  $p_0$ and $r_0$,
obtaining an identity $ap=1+r$ with
$p,r\in\M(\An_2)$ and controlled analytic certificate size.
We then eliminate the auxiliary multiplier $p$ using
Chebyshev polynomials and apply
Lemma~\ref{lem:exactification} to obtain the polynomial
degree bound in Theorem~\ref{thm:main}.

\subsection{Converting Stengle certificates to the analytic module}

We apply Lemma~\ref{lem:products} to the Stengle
certificate provided by Lemma~\ref{lem:stengle} to obtain analytic certificates in $\M(\An_2)$ with controlled size.

We first derive the analytic certificate. Let $a$ satisfy the hypotheses of Lemma~\ref{lem:stengle}.
Choose a Stengle certificate
\[
ap_0=1+r_0,\qquad p_0,r_0\in\PP_B,\qquad
\size_{p_0,r_0}\le S:=As^{-\nu},
\]
where we may assume $S\ge1$ by enlarging $A$.
If $m=0$, then $\PP=\M$, so we may take
$p=p_0$ and $r=r_0$ directly.
Henceforth assume $m\ge1$. For each $t\in\{p_0,r_0\}$, fix a preordering
representation
\[
t=\sum_\alpha g^\alpha\sigma_\alpha^t
  +\sum_j h_j\lambda_j^t.
\]
The certificate-size bound implies
$\|\sum_\alpha\sigma_\alpha^t\|_2\le S$.
Set
\[
\kappa_0=\frac18,\qquad
\rho=\frac{\kappa_0}{2S},
\]
and let $H=H_\rho$ be the common multiplier
provided by Lemma~\ref{lem:products}.
Then
\begin{equation}\label{eq:convert-p}
\begin{aligned}
H(t+\kappa_0)
={}&\sum_\alpha
\sigma_\alpha^t H(g^\alpha+\rho)+H\left(\kappa_0-\rho\sum_\alpha
\sigma_\alpha^t\right)
+\sum_j h_jH\lambda_j^t.
\end{aligned}
\end{equation}
By Lemma~\ref{lem:products}, each
$H(g^\alpha+\rho)$ admits an analytic
certificate in $\M(\An_2)$.
Moreover, the choice of $\rho$ gives
\[
\left\|\rho\sum_\alpha\sigma_\alpha^t\right\|_2
\le\rho S=\frac{\kappa_0}{2}.
\]
Applying \eqref{eq:analytic-root} with
$\gamma=\kappa_0$ and
$\Delta=-\rho\sum_\alpha\sigma_\alpha^t$,
we see that the factor in parentheses
in \eqref{eq:convert-p} is an analytic square.
Since $H$ is SOS, all terms in
\eqref{eq:convert-p} belong to $\M(\An_2)$.
Consequently,
\[
H(p_0+\kappa_0),\quad
H(r_0+\kappa_0)\in\M(\An_2).
\]

Set
\[
P=H(p_0+\kappa_0),\qquad
R=H(r_0+\kappa_0).
\]
The Stengle identity $ap_0=1+r_0$ gives
\[
aP=R+H\bigl(1+\kappa_0(a-1)\bigr).
\]
To recover an identity of the form $ap=1+r$,
we use
\[
\begin{aligned}
2H\bigl(1+\kappa_0(a-1)\bigr)
={}&1+(H-1)+H\bigl(1+2\kappa_0(a-1)\bigr).
\end{aligned}
\]
Here $H-1$ is SOS by Lemma~\ref{lem:products}.
since $\|a/4\|_2\le1/8$,
\[
1+2\kappa_0(a-1)=\frac34+\frac a4
\]
is an analytic square by \eqref{eq:analytic-root}.
We may therefore define
\begin{equation}\label{eq:internal}
\begin{aligned}
p=2P,\quad r=2R+H\bigl(1+2\kappa_0(a-1)\bigr)+(H-1),
\end{aligned}
\end{equation}
which yields $ap=1+r$ with
$p,r\in\M(\An_2)$.

It remains to bound the analytic certificate size.
Write $k=\lceil\log_2(4^m/\rho)\rceil$.
By Lemma~\ref{lem:products}, the certificates
for $H$, $H-1$, and all $H(g^\alpha+\rho)$
have size at most $\exp(Ck^2)$, and
$\|H\|_2\le\exp(Ck^2)$.
The analytic square roots used above have
uniformly bounded radius-two norms.
Together with the Stengle certificate-size
bound and submultiplicativity, this gives
\[
\size_{p,r}+\|p\|_2
\le C(1+S)\exp(Ck^2).
\]
Since $S=As^{-\nu}$ and
$\rho=\kappa_0/(2S)$, we have
$k=O(\log(2/s))$.
Consequently,
\begin{equation}\label{eq:internal-summary}
\begin{gathered}
ap=1+r,\qquad p,r\in\M(\An_2),\\
\log\size_{p,r}+\log(2+\|p\|_2)
\le C\log^2(2/s).
\end{gathered}
\end{equation}
For $m=0$, the same estimate follows directly
from Lemma~\ref{lem:stengle}.

\subsection{Eliminating the auxiliary multiplier with Chebyshev polynomials}
The preceding subsection establishes
$ap=1+r$ with $p,r\in\M(\An_2)$.
To eliminate the auxiliary multiplier $p$,
we construct two nonnegative polynomials
using Chebyshev polynomials.
These allow us to obtain an analytic certificate
for $a+\varepsilon$ in $\M(\An_2)$
with controlled size.

We first choose an integer $N$ so that
the Chebyshev polynomial $T_N$ is bounded
on $[-1,1]$ and sufficiently large beyond
this interval. Set
\[
 C_p=2(1+\|p\|_2),\quad 0<\varepsilon\le1/4,\quad
 N=\left\lceil
 \frac{\acosh\sqrt{2C_p}}{\acosh(1+2\varepsilon)}
 \right\rceil.
\]
Define
\begin{equation}\label{eq:chebyshev}
 \beta(t)=\frac{T_N(1-2t)^2}{2C_p},\quad
 R_0(t)=t+\varepsilon+\beta(t),\quad
 R_1(t)=t+\varepsilon+(1-C_pt)\beta(t).
\end{equation}The choice of $N$ makes $\beta$ small on $[0,1]$ and at least
one on $[-1,-\varepsilon]$. Indeed, for $0\le t\le1$,
$0\le\beta(t)\le1/(2C_p)$; for $-1\le t\le-\varepsilon$,
$\beta(t)\ge1$, by \eqref{eq:chebyshev-growth}.
Both $R_0,R_1$ are nonnegative on $[-1,1]$. For $t\ge0$,
$R_1\ge\varepsilon+t/2+\beta$ and $R_0\ge0$; for
$-\varepsilon\le t\le0$, $R_1\ge R_0\ge t+\varepsilon$; for
$t\le-\varepsilon$, $R_1\ge R_0\ge t+\varepsilon+1$.

We now use $ap=1+r$ to construct a direct
certificate for $a+\varepsilon$.
By the definition of $R_0$ and the identity
$1=ap-r$, we have
\[
\begin{aligned}
a+\varepsilon=R_0(a)-\beta(a)=R_0(a)-ap\beta(a)+\beta(a)r.
\end{aligned}
\]
Since $R_1(t)-R_0(t)=-C_pt\beta(t)$,
we can rewrite the middle term in right hand as
\[
-ap\beta(a)
=\frac{p}{C_p}\bigl(R_1(a)-R_0(a)\bigr).
\]
Substituting this identity and collecting
the terms involving $R_0(a)$ gives
\begin{equation}\label{eq:elimination}
a+\varepsilon
=\beta(a)r+
\frac{C_p-p}{C_p}R_0(a)
+\frac{p}{C_p}R_1(a).
\end{equation} We then show that the identity is an analytic certificate in the
original module. By Lemma~\ref{lem:interval}, both $R_i(a)$ are
analytic SOS; $\beta(a)$ is one square, and $C_p-p$ is an analytic
square by $\|p/C_p\|_2\le1/2$. Thus the right-hand side of
\eqref{eq:elimination} is a valid module representation.

Finally, we estimate its analytic size. By
\eqref{eq:chebyshev-coef-bound},
$\|T_N(1-2t)\|_1\le7^N$, hence
$\|\beta\|_1\le7^{2N}/(2C_p)$ and
$\|R_0\|_1+\|R_1\|_1\le\exp(O(N))$; in the latter estimate, the
factor $C_p$ in $(1-C_pt)\beta$ cancels the denominator $C_p$.
The squared norm of the analytic square root of $C_p-p$ is $O(C_p)$,
while multiplication of the analytic SOS representations
multiplies their size bounds. Lemma~\ref{lem:interval} and
\eqref{eq:analytic-root} therefore bound the logarithmic size of
\eqref{eq:elimination} by
\begin{equation}\label{eq:elim-size}
 C\bigl(N+\log\size_{p,r}+\log(2C_p)+1\bigr).
\end{equation}
Because $\acosh(1+2\varepsilon)=2\asinh\sqrt\varepsilon$,
\begin{equation}\label{eq:N-bound}
 N\le C\varepsilon^{-1/2}\log(2C_p).
\end{equation}
Combining these estimates with \eqref{eq:internal-summary} and
choosing $\varepsilon=s/2$ yields an analytic certificate for
$a+s/2$ with logarithmic size at most
$Cs^{-1/2}\log^2(2/s)$. The factor $s^{-1/2}$ comes from the
Chebyshev degree, while the squared logarithm comes from the size
of the common multiplier in Section~\ref{sec:products}.
Lemma~\ref{lem:exactification} will convert this logarithmic size
bound into a polynomial certificate-degree bound.

\Needspace{10\baselineskip}
\subsection{Return to the original objective}

The preceding subsection constructs an analytic certificate
for $a+\varepsilon$ with controlled size.
We now apply this construction to a shifted and normalized
version of the original objective $f$, obtaining an
analytic certificate for $f-\delta/4$.
The remaining positive slack allows us to apply
Lemma~\ref{lem:exactification} and obtain an exact
polynomial Putinar certificate for $f$.

\begin{proof}[Proof of the upper assertion in Theorem~\ref{thm:main}]
We work in the coordinates \eqref{eq:normalized-witness}.
By finite-dimensional norm equivalence, there exists
a constant $K_D$ such that
\[
\|f\|_2\le K_D\|f\|_{[-4,4]^n,\infty}
\]
for all polynomials of degree at most $D$.

Set $H_D=4(K_D+1)$ and define
\[
a=\frac{f-\delta/2}{H_D},
\qquad
s=\frac{\delta}{2H_D}.
\]
Since $\|f\|_{[-4,4]^n,\infty}\le1$
and $f|_K\ge\delta$, we have
\[
\|a\|_2\le\frac12,\qquad
a|_K\ge s,\qquad 0<s\le\frac14.
\]
Thus $a$ satisfies the hypotheses of
Lemma~\ref{lem:stengle}.

Applying the preceding construction with
$\varepsilon=s/2$ gives an analytic certificate
for $a+s/2$ with
\[
\log\size
\le Cs^{-1/2}\log^2(2/s).
\]
By the definitions of $a$ and $s$,
\[
H_D\left(a+\frac{s}{2}\right)
=f-\frac{\delta}{4}.
\]
Multiplying the analytic certificate by the fixed
positive constant $H_D$ therefore gives
\[
f-\frac{\delta}{4}\in\M(\An_2),
\]
with certificate size satisfying
\begin{equation}\label{eq:final-analytic-size}
\log\size
\le C\delta^{-1/2}\log^2(2/\delta).
\end{equation}

Finally, applying Lemma~\ref{lem:exactification}
with $\tau=\delta/4$ and using
\eqref{eq:final-analytic-size}, we obtain
\[
\begin{aligned}
d_\M(f)
&\le C\left(D+\log\frac{4\size}{\delta}+1\right)\\
&\le C_{\M,D}\delta^{-1/2}\log^2(2/\delta),
\end{aligned}
\]
where the last inequality uses $0<\delta\le1$
and the fact that $D$ is fixed.
The resulting polynomial certificate also satisfies
the same asymptotic bound on its logarithmic size.
Undoing the coordinate scaling completes the proof.
\end{proof}

The product-conversion step introduces the second
logarithmic factor in the preceding upper bound.
For preorderings, generator products are already
allowed in the representations, so this step can be omitted.

\begin{corollary}[Preorderings]\label{cor:preordering}
Fix an Archimedean polynomial preordering $\PP$
and $D\ge1$, and let $Q$ be the normalization box
associated with a fixed Archimedean witness.
If $\deg f\le D$, $\|f\|_{Q,\infty}\le1$,
and $f|_K\ge\delta>0$, then
\[
d_\PP(f)
\le C_{\PP,D}\delta^{-1/2}\log(2/\delta).
\]
The same bound holds for a fixed Archimedean
quadratic module with one inequality generator
and any fixed equalities.
\end{corollary}

\begin{proof}
The preordering $\PP$ is the quadratic module
generated by the nonconstant squarefree products
$g^\alpha$, together with the same equality ideal.
Its degree truncations and the fixed Archimedean
witness are therefore preserved.

Normalize $f$ as in the preceding proof, setting
\[
a=\frac{f-\delta/2}{H_D},
\qquad s=\frac{\delta}{2H_D}.
\]
By Lemma~\ref{lem:stengle}, there exist
\[
ap_0=1+r_0,\qquad p_0,r_0\in\PP_B,
\qquad \size_{p_0,r_0}\le As^{-\nu}.
\]
Since the generators are fixed, the certificate-size
bound also implies $\|p_0\|_2\le Cs^{-\nu}$.

We may therefore apply the Chebyshev elimination
\eqref{eq:elimination} directly with
$p=p_0$, $r=r_0$, and $\varepsilon=s/2$,
without using Lemma~\ref{lem:products}.
This gives an analytic certificate for $a+s/2$
satisfying
\[
\log\size\le Cs^{-1/2}\log(2/s).
\]
Since $H_D(a+s/2)=f-\delta/4$, the normalization
and exactification argument above yields
\[
d_\PP(f)
\le C_{\PP,D}\delta^{-1/2}\log(2/\delta).
\]
Finally, with one inequality generator,
the quadratic module coincides with its
preordering, giving the last assertion.
\end{proof}

\section{An analytic coefficient-norm lower bound}
\label{sec:corner}

The preceding section establishes an upper bound on the degree of Putinar SOS certificates. In this section and the next, we construct an example that attains this bound up to constant factors, thereby establishing its asymptotic sharpness.

In this section, we introduce an auxiliary SOS identity and show that it admits no polynomial solutions of any finite degree. We then consider analytic power-series representations of this identity. Truncating these series introduces an error that destroys the exact identity. By constructing a suitable linear functional, we establish a lower bound on the resulting discrepancy between two side, while analytic tail estimates provide an upper bound in terms of the coefficientwise Wiener norms. Optimizing the truncation degree yields a lower bound on these norms. In the next section, we apply this result to a  constructed example. The resulting contradiction establishes the matching degree lower bound
$\Omega\left(\delta^{-1/2}\log^2(1/\delta)\right).$



\paragraph{Vector-valued Wiener norms.}
Let $X,Y$ be scalar variable.
For a finite vector of real power series
$\boldsymbol F=(F_1,\ldots,F_N)$, write
\[
\boldsymbol F(X,Y)
=\sum_{\alpha\in\mathbb N_0^2}
\boldsymbol f_\alpha X^{\alpha_1}Y^{\alpha_2},
\qquad \boldsymbol f_\alpha\in\R^N.
\]
Here $\boldsymbol f_\alpha$ collects the coefficients
of $X^{\alpha_1}Y^{\alpha_2}$ across the $N$ components.
Its coefficientwise Wiener norm at radius $R>0$ is
\begin{equation}\label{eq:vector-analytic-norm}
\|\boldsymbol F\|_R
=\sum_\alpha\|\boldsymbol f_\alpha\|_{\ell^2}R^{|\alpha|}.
\end{equation}
For $N=1$, this agrees with the scalar norm
\eqref{eq:analytic-norm}. For example, if
$\boldsymbol F=(1+X,1+Y)$, then
\[
\boldsymbol f_{(0,0)}=(1,1),\quad
\boldsymbol f_{(1,0)}=(1,0),\quad
\boldsymbol f_{(0,1)}=(0,1),
\]
and hence
\[
\|\boldsymbol F\|_R=\sqrt{2}+2R
\ne\sqrt{2}(1+R)
=\left(\sum_{j=1}^2\|F_j\|_R^2\right)^{1/2}.
\]

\begin{lemma}[Coefficient norm lower bound for analytic iedntity]\label{lem:corner}
Suppose $R>1$, $0\le\eta\le1$, and the finite real
vector-valued power series
$\boldsymbol B,\boldsymbol A_0,\boldsymbol A_1,\boldsymbol A_2$
have finite radius-$R$ norms and satisfy
\begin{equation}\label{eq:corner}
 (XY+\eta)\sq{\boldsymbol B}
 =\tfrac12+\sq{\boldsymbol A_0}
   +X\sq{\boldsymbol A_1}+Y\sq{\boldsymbol A_2}.
\end{equation}
Put
\[
 W_R=\max\bigl\{1,\|\boldsymbol B\|_R,\|\boldsymbol A_0\|_R,\|\boldsymbol A_1\|_R,\|\boldsymbol A_2\|_R\bigr\}.
\]
Then
\begin{equation}\label{eq:corner-growth}
\log W_R
\ge \frac{(\log R)^2}{32\log 16}
-\frac{\log 16+\log 20}{2}.
\end{equation}
\end{lemma}

\begin{proof}
We first show that Eq.~\eqref{eq:corner}
admits no polynomial-vector solutions.
The key is to construct a linear functional
that is nonpositive on the left-hand side
of Eq.~\eqref{eq:corner} but strictly
positive on its right-hand side whenever
the vectors are polynomial.

\textit{Diagonal dominance.}
Set $q=16$ and define a linear functional
$\Lambda:\R[X,Y]\to\R$ by
\begin{equation}\label{eq:corner-functional}
\Lambda(X^iY^j)=(-1)^{ij}q^{i^2+j^2}.
\end{equation}
For $a,b\in\{0,1\}$, define
\begin{equation}\label{eq:corner-weights}
\begin{aligned}
h_{ab}&=(-1)^{ab}X^aY^b,\\
(h_{00},h_{10},h_{01},h_{11})
&=(1,X,Y,-XY).
\end{aligned}
\end{equation}
We will show that $\Lambda$ is nonnegative
on polynomial squares multiplied by
each of the four weights in
Eq.~\eqref{eq:corner-weights}.

For a polynomial $p=\sum_{r=1}^N c_rm_r$,
where $m_r=X^{i_r}Y^{j_r}$ are distinct
monomials, let
$\mathbf c=(c_1,\ldots,c_N)^{\mathsf T}$.
Then
$\Lambda(h_{ab}p^2)
=\mathbf c^{\mathsf T}G^{ab}\mathbf c$,
where $G^{ab}$ is a symmetric matrix.
For $r,s\in\{1,\ldots,N\}$, define
\begin{equation}\label{eq:corner-indices}
\alpha_{rs}=i_r+i_s+a,\qquad
\beta_{rs}=j_r+j_s+b.
\end{equation}
By Eqs.~\eqref{eq:corner-functional}--\eqref{eq:corner-indices},
the entries of $G^{ab}$ are
\begin{equation}\label{eq:corner-gram}
G^{ab}_{rs}
=\Lambda(h_{ab}m_rm_s)
=(-1)^{ab+\alpha_{rs}\beta_{rs}}
q^{\alpha_{rs}^2+\beta_{rs}^2}.
\end{equation}
Its diagonal entries are positive, since
$ab+\alpha_{rr}\beta_{rr}
=2(ab+2i_rj_r+i_rb+j_ra)$ is even.
Hence,
$G^{ab}_{rr}
=q^{\alpha_{rr}^2+\beta_{rr}^2}>0$.
The diagonal terms sum to
\begin{equation}\label{eq:corner-diagonal-sum}
D_{ab}(p)=\sum_{r=1}^N c_r^2G^{ab}_{rr}.
\end{equation}

To control the off-diagonal terms, we normalize
$G^{ab}$ by its diagonal entries.
Define $\mathbf w=(w_1,\ldots,w_N)^{\mathsf T}$
and the symmetric matrix $H^{ab}$ by
\begin{equation}\label{eq:corner-normalization}
w_r=c_r\sqrt{G^{ab}_{rr}},\qquad
H^{ab}_{rs}=\frac{G^{ab}_{rs}}
{\sqrt{G^{ab}_{rr}G^{ab}_{ss}}}.
\end{equation}
Then Eqs.~\eqref{eq:corner-diagonal-sum}
and \eqref{eq:corner-normalization} give
$\Lambda(h_{ab}p^2)
=\mathbf w^{\mathsf T}H^{ab}\mathbf w$
and $D_{ab}(p)=\sum_{r=1}^N w_r^2$.
Using
\[
\begin{aligned}
&(i_r+i_s+a)^2
-\frac{(2i_r+a)^2+(2i_s+a)^2}{2}=-(i_r-i_s)^2,
\end{aligned}
\]
and the analogous identity for $j_r,j_s,b$,
we obtain
\begin{equation}\label{eq:corner-normalized-entries}
H^{ab}_{rr}=1,\qquad
|H^{ab}_{rs}|
=q^{-(i_r-i_s)^2-(j_r-j_s)^2}.
\end{equation}

We next show that $H^{ab}$ is strictly
diagonally dominant. Let $\gamma$ denote
its maximum absolute off-diagonal row sum.
Since the exponent pairs $(i_r,j_r)$
are distinct, Eq.~\eqref{eq:corner-normalized-entries}
gives
\[
\begin{aligned}
\gamma
&=\max_r\sum_{s\ne r}|H^{ab}_{rs}|
\le\sum_{\substack{m,n\in\mathbb Z\\
(m,n)\ne(0,0)}}q^{-m^2-n^2}\\
&=\left(1+2\sum_{m=1}^{\infty}q^{-m^2}\right)^2-1\\
&\le\left(1+\frac{2}{q-1}\right)^2-1
=\frac{64}{225}<\frac12,
\end{aligned}
\]
where we used $m^2\ge m$ for $m\ge1$.
Since $H^{ab}_{rr}=1>\gamma$,
the matrix $H^{ab}$ is strictly
diagonally dominant.

By symmetry and $H^{ab}_{rr}=1$, we have
\[
\Lambda(h_{ab}p^2)-D_{ab}(p)
=2\sum_{r<s}H^{ab}_{rs}w_rw_s.
\]
Using $2|w_rw_s|\le w_r^2+w_s^2$,
we obtain
\[
\begin{aligned}
|\Lambda(h_{ab}p^2)-D_{ab}(p)|
&\le 2\sum_{r<s}|H^{ab}_{rs}||w_rw_s|\le\sum_{r<s}|H^{ab}_{rs}|(w_r^2+w_s^2)\\
&=\sum_{r=1}^N w_r^2
\left(\sum_{s>r}|H^{ab}_{rs}|
+\sum_{s<r}|H^{ab}_{sr}|\right)\\
&=\sum_{r=1}^N w_r^2\sum_{s\ne r}|H^{ab}_{rs}|
\le\gamma D_{ab}(p)
\le\tfrac12D_{ab}(p).
\end{aligned}
\]
Consequently,
\begin{equation}\label{eq:diagonal}
\tfrac12D_{ab}(p)
\le\Lambda(h_{ab}p^2)
\le\tfrac32D_{ab}(p).
\end{equation}

By Eqs.~\eqref{eq:corner-weights} and
\eqref{eq:diagonal}, $\Lambda$ is nonnegative
on products of SOS polynomials with
$1$, $X$, $Y$, and $-XY$.
In particular, taking $(a,b)=(1,1)$
and $(a,b)=(0,0)$ in Eq.~\eqref{eq:diagonal}
gives
\begin{equation}\label{eq:corner-sign-bounds}
\begin{aligned}
\Lambda((-XY)p^2)&\ge\tfrac12D_{11}(p),\\
\Lambda(p^2)&\le\tfrac32D_{00}(p).
\end{aligned}
\end{equation}
Moreover, since $i_r,j_r\ge0$ and $q>1$,
we have
\[
\begin{aligned}
D_{11}(p)
&=\sum_{r=1}^N c_r^2
q^{(2i_r+1)^2+(2j_r+1)^2}\\
&\ge q^2\sum_{r=1}^N c_r^2
q^{4i_r^2+4j_r^2}
=q^2D_{00}(p).
\end{aligned}
\]
Thus, by Eq.~\eqref{eq:corner-sign-bounds},
for $0\le\eta\le1$,
\begin{equation}\label{eq:negative-localizer}
\begin{aligned}
\Lambda((XY+\eta)p^2)
&=-\Lambda((-XY)p^2)+\eta\Lambda(p^2)\\
&\le-\tfrac12D_{11}(p)
+\tfrac{3\eta}{2}D_{00}(p)\\
&\le\left(-\tfrac{q^2}{2}
+\tfrac{3\eta}{2}\right)D_{00}(p)
\le0.
\end{aligned}
\end{equation}

Now suppose that $B,A_0,A_1,A_2$
in Eq.~\eqref{eq:corner} were polynomial
vectors. By Eq.~\eqref{eq:negative-localizer},
$\Lambda$ is nonpositive on the left-hand
side of Eq.~\eqref{eq:corner}.
On the other hand, Eqs.~\eqref{eq:corner-weights}
and \eqref{eq:diagonal} show that its value
on the right-hand side is at least
$\tfrac12\Lambda(1)=\tfrac12$.
This is a contradiction. Therefore,
Eq.~\eqref{eq:corner} admits no
polynomial-vector solutions, regardless
of their degrees or the numbers of square terms.

\textit{Coefficient norm lower bound.}
For power series, we instead apply $\Lambda$
to polynomial truncations.
Fix $j\ge1$, and let $\Pi_{j-1}$ denote
truncation to total degree at most $j-1$.
For each $F\in\{B,A_0,A_1,A_2\}$,
write $F^{(j)}=\Pi_{j-1}F$,
with truncation applied componentwise.
Define
\begin{equation}\label{eq:truncation-residual}
\begin{aligned}
E_j={}&(XY+\eta)[B^{(j)}]^2
-\tfrac12-[A_0^{(j)}]^2-X[A_1^{(j)}]^2-Y[A_2^{(j)}]^2.
\end{aligned}
\end{equation}
Since each truncated vector has degree
at most $j-1$, we have $\deg E_j\le2j$.
Moreover, each component of the truncated
vectors is a polynomial, so
Eqs.~\eqref{eq:negative-localizer}
and \eqref{eq:diagonal} apply directly.
In particular,
\begin{equation}
\begin{aligned}
&\Lambda((XY+\eta)[B^{(j)}]^2)\le0,\quad
\Lambda([A_0^{(j)}]^2)\ge0,\\
&\Lambda(X[A_1^{(j)}]^2)\ge0,\quad
\Lambda(Y[A_2^{(j)}]^2)\ge0.
\end{aligned}
\end{equation}

Consequently, by linearity of $\Lambda$,
\[
\Lambda(E_j)
\le-\tfrac12\Lambda(1)=-\tfrac12.
\]

We next bound the coefficient norm of $E_j$.
For each $F\in\{B,A_0,A_1,A_2\}$,
the definition of the weighted Wiener norm
gives
\begin{equation}
\begin{gathered}
\|F-F^{(j)}\|_1
\le R^{-j}\|F\|_R
\le W_RR^{-j},\\
\|F\|_1,\ \|F^{(j)}\|_1
\le W_R.
\end{gathered}
\end{equation}
For any two vector series $F,G$ of the same
dimension, coefficientwise Cauchy--Schwarz
and the convolution inequality yield
\[
\|[F]^2-[G]^2\|_1
\le(\|F\|_1+\|G\|_1)\|F-G\|_1.
\]
These estimates use the coefficientwise
Euclidean norm. Taking $G=F^{(j)}$, we obtain
\[
\begin{aligned}
\|[F]^2-[F^{(j)}]^2\|_1
&\le(\|F\|_1+\|F^{(j)}\|_1)
\|F-F^{(j)}\|_1\le2W_R^2R^{-j}.
\end{aligned}
\]

Using the exact identity \eqref{eq:corner},
we express $E_j$ in terms of truncation errors:
\[
\begin{aligned}
E_j={}&(XY+\eta)
\bigl([B^{(j)}]^2-[B]^2\bigr)
-\bigl([A_0^{(j)}]^2-[A_0]^2\bigr)\\
&-X\bigl([A_1^{(j)}]^2-[A_1]^2\bigr)
-Y\bigl([A_2^{(j)}]^2-[A_2]^2\bigr).
\end{aligned}
\]
Since $0\le\eta\le1$, we have
$\|XY+\eta\|_1=1+\eta\le2$ and
$\|X\|_1=\|Y\|_1=1$.
Thus, the triangle inequality and
submultiplicativity of the coefficient norm give
\[
\|E_j\|_1
\le2(2+1+1+1)W_R^2R^{-j}
=10W_R^2R^{-j}.
\]Since $\deg E_j\le2j$, every monomial
$X^uY^v$ occurring in $E_j$ satisfies
$u+v\le2j$. By Eq.~\eqref{eq:corner-functional},
\[
|\Lambda(X^uY^v)|
=q^{u^2+v^2}
\le q^{(u+v)^2}
\le q^{4j^2}.
\]
Combining this bound with
$\Lambda(E_j)\le-\tfrac12$
and the coefficient-norm estimate, we obtain
\begin{equation}\label{eq:corner-test}
\begin{aligned}
&\tfrac12
\le|\Lambda(E_j)|
\le q^{4j^2}\|E_j\|_1\le10W_R^2R^{-j}q^{4j^2},\\
&\log W_R
\ge\tfrac12\bigl(
j\log R-4j^2\log q-\log20\bigr).
\end{aligned}
\end{equation}

We now optimize the lower bound in
Eq.~\eqref{eq:corner-test} over positive integers $j$.
The quadratic expression
$j\log R-4j^2\log q$ is concave
and attains its maximum over $\R$
at $j_*=(\log R)/(8\log q)$.
If $\log R\ge4\log q$, then $j_*\ge\tfrac12$,
so we choose the nearest positive integer $j$
with $|j-j_*|\le\tfrac12$.
Completing the square gives
\[
\begin{aligned}
\log W_R
&\ge\frac{(\log R)^2}{32\log q}
-2\log q(j-j_*)^2-\frac{\log20}{2}\\
&\ge\frac{(\log R)^2}{32\log q}
-\frac{\log q+\log20}{2}.
\end{aligned}
\]
If $\log R<4\log q$, the claimed lower
bound is negative, whereas $W_R\ge1$.
This proves the lemma.
\end{proof}

The same functional also shows that
$-1\notin\QM(X,Y,-XY-1)$, although this quadratic module
has an empty feasible set. The underlying properness obstruction
was already observed by Marshall~\cite[Example~1.5]{Marshall2009};
our functional provides an explicit quantitative witness. Indeed,
\[
 \Lambda((-XY-1)p^2)
 \ge(q^2/2-3/2)D_{00}(p)\ge0,
\]
so $\Lambda$ is nonnegative on every generator-weighted
square, whereas $\Lambda(-1)=-1$.

\section{A fixed module attaining the lower bound}
\label{sec:lower}

We now consider the fixed quadratic module $\M_*$
and the polynomial family in \eqref{eq:example}.
We show that every Putinar certificate for $u+\varepsilon$
over $\M_*$ produces analytic vectors satisfying
\eqref{eq:corner}. We first bound the polynomial vectors
in the certificate on a box contained in the feasible set,
then use Chebyshev extrapolation to transfer these bounds
to the slice $u=-2\varepsilon$. Rescaling the remaining
variables and absorbing the ball generator into an analytic
square root gives an upper bound for the Wiener norms
in terms of the certificate degree $d$ and $\varepsilon$.
If $d$ were asymptotically smaller than
$\varepsilon^{-1/2}\log^2(1/\varepsilon)$, this estimate
would contradict Lemma~\ref{lem:corner}. This establishes
a certificate-degree lower bound matching the order of
the upper bound in Section~\ref{sec:upper}.

\paragraph{Markov and Chebyshev inequalities.}
For a polynomial $p$ of degree at most $d$ and $a<b$,
Markov's inequality states that
\begin{equation}\label{eq:markov-inequality}
\|p'\|_{[a,b],\infty}
\le\frac{2d^2}{b-a}\|p\|_{[a,b],\infty}.
\end{equation}
We also use the Chebyshev extrapolation inequality
\eqref{eq:chebyshev-extrapolation}; see
\cite[Theorem~18 and Lemma~20]{BaldiSlot}.
Both inequalities extend to finite-dimensional Euclidean-valued
polynomials by applying the scalar inequalities after taking
inner products with unit vectors.

\subsection{The fixed example and certificate setup}

Recall the fixed quadratic module from
Eq.~\eqref{eq:example}:
\[
\M_*=\QM\left(
3-u^2-x^2-y^2,\,
-ux,\,-uy,\,
\tfrac12u(xy+u^4)
\right).
\]
Its feasible set $K_*$ is
\begin{equation}\label{eq:feasible-example}
\{u^2+x^2+y^2\le3\}
\cap\bigl(\{u=0\}\cup
\{u>0,\ x\le0,\ y\le0\}\bigr).
\end{equation}
Indeed, if $u<0$, the constraints $-ux\ge0$
and $-uy\ge0$ force $x,y\ge0$, contradicting
$u(xy+u^4)\ge0$.
The cases $u=0$ and $u>0$ follow directly.
Thus $K_*$ has nonempty interior and
$\min_{K_*}u=0$.

Since $3-u^2-x^2-y^2$ is a generator of
$\M_*$, the Archimedean witness
\eqref{eq:ball-general} holds with $r=\sqrt3$.
The normalization box in Theorem~\ref{thm:main}
is therefore $Q_*=[-R_*,R_*]^3$, where
$R_*=4\sqrt3$.
For $\varepsilon>0$, define
\[
f_\varepsilon=\frac{u+\varepsilon}{R_*+\varepsilon},
\qquad
\delta_\varepsilon=\frac{\varepsilon}{R_*+\varepsilon}.
\]
Since $\min_{K_*}u=0$, we obtain
\[
\min_{K_*}f_\varepsilon=\delta_\varepsilon.
\]
Moreover, $|u+\varepsilon|\le R_*+\varepsilon$
on $Q_*$, with equality at $(R_*,0,0)$.
Hence
\[
\|f_\varepsilon\|_{Q_*,\infty}=1.
\]

Since positive scaling preserves the degree of
Putinar certificates, we have
$d_{\M_*}(f_\varepsilon)=d_{\M_*}(u+\varepsilon)$.
We therefore work directly with $u+\varepsilon$.
Write $b=3-u^2-x^2-y^2$.
Absorbing the factor $1/2$ in the last generator
into its SOS multiplier, any Putinar certificate
for $u+\varepsilon$ of degree at most $d$
can be written as
\begin{equation}\label{eq:lower-cert}
u+\varepsilon
=\sigma_0+b\sigma_b
-ux\sigma_1-uy\sigma_2+u(xy+u^4)\sigma_3,
\end{equation}
where $\sigma_i=\sq{\boldsymbol P_i}$ and
$\boldsymbol P_i=(P_{i,1},\ldots,P_{i,N_i})$
is a finite polynomial vector for
$i\in\{0,b,1,2,3\}$, with $\deg P_{i,j}\le d$.

We assume $0<\varepsilon\le1/16$.
Our goal is to show that every certificate
of the form \eqref{eq:lower-cert} satisfies
\[
d\ge c\varepsilon^{-1/2}\log^2(1/\varepsilon)
\]
for some constant $c>0$ independent of
$\varepsilon$.
This bound is immediate if
$d>\varepsilon^{-1/2}\log^2(1/\varepsilon)$.
Hence, in the remainder of the proof,
we may assume
\begin{equation}\label{eq:d-regime}
d\le\varepsilon^{-1/2}\log^2(1/\varepsilon).
\end{equation}

\subsection{Bounds on a slice outside the feasible set}

We first restrict \eqref{eq:lower-cert}
to the slice $u=-2\varepsilon$ and derive
uniform bounds on the SOS polynomial
factors there. These bounds will be used
in the next subsection to construct the
analytic vectors required by
Lemma~\ref{lem:corner} and control
their radius-$R$ Wiener norms.

\textit{Restriction to the slice.}
We begin by restricting \eqref{eq:lower-cert}
to the slice $u=-2\varepsilon$, where
$u+\varepsilon=-\varepsilon$.
Write
\[
b_-=3-4\varepsilon^2-x^2-y^2,
\qquad
\sigma_i^-=\sigma_i(-2\varepsilon,x,y).
\]
Substituting into \eqref{eq:lower-cert} gives
\[
\begin{aligned}
-\varepsilon
={}&\sigma_0^-+b_-\sigma_b^-
+2\varepsilon x\sigma_1^-\!+
2\varepsilon y\sigma_2^-\!\!-2\varepsilon(xy+16\varepsilon^4)\sigma_3^-.
\end{aligned}
\]
Rearranging and dividing by $2\varepsilon$,
we obtain
\begin{equation}\label{eq:restricted-certificate}
(xy+16\varepsilon^4)\sigma_3^-
=\frac12+\frac{\sigma_0^-+b_-\sigma_b^-}{2\varepsilon}
+x\sigma_1^-+y\sigma_2^-.
\end{equation}
This identity will be transformed into
the form required by Lemma~\ref{lem:corner}.

\textit{Bounds on a feasible box.}
However, on the slice $u=-2\varepsilon$,
the terms $2\varepsilon x\sigma_1^-$ and
$2\varepsilon y\sigma_2^-$ are nonpositive
when $x,y\le0$, as is
$-2\varepsilon(xy+16\varepsilon^4)\sigma_3^-$.
Thus the terms of the certificate are no
longer all nonnegative on this slice.
We therefore consider the box
\[
I_\varepsilon\times J^2
=[\varepsilon,1/2]\times[-1/2,0]^2
\subseteq K_*.
\]
On this box, $b\ge9/4$ and
$u+\varepsilon\le2u\le1$.
Every term on the right-hand side of
\eqref{eq:lower-cert} is nonnegative
and hence bounded by $u+\varepsilon$.
We obtain
\[
\begin{gathered}
\sigma_0,\ \sigma_b\le1,\qquad
\sigma_3
\le\frac{u+\varepsilon}{u(xy+u^4)}
\le\frac{2}{u^4}
\le2\varepsilon^{-4},\\
(-x)\sigma_1\le2,\qquad
(-y)\sigma_2\le2.
\end{gathered}
\]

The last two estimates do not directly give
uniform bounds for $\sigma_1$ and $\sigma_2$.
We obtain such bounds using Markov's inequality
\eqref{eq:markov-inequality}.

\textit{Uniform bounds for the SOS factors.}
For fixed $u,y$, the polynomial
$q(x)=x\sigma_1(u,x,y)$
has degree at most $2d+1$, since
$\deg P_{1,j}\le d$ and
$\sigma_1=\sq{\boldsymbol P_1}$.
It vanishes at $x=0$ and satisfies
$\|q\|_{J,\infty}\le2$.
Since $J=[-1/2,0]$, Markov's inequality
therefore gives
\[
\|q'\|_{J,\infty}
\le\frac{2(2d+1)^2}{1/2}
\|q\|_{J,\infty}\le8(2d+1)^2.
\]
Since $q(0)=0$, the mean-value theorem
yields, for $x<0$,
\[
\begin{aligned}
\sigma_1(u,x,y)
&=\frac{|q(x)|}{|x|}
=\frac{|q(x)-q(0)|}{|x|}\le\|q'\|_{J,\infty}
\le8(2d+1)^2.
\end{aligned}
\]
The same bound holds at $x=0$ by continuity.
Similarly, $\sigma_2\le8(2d+1)^2$
throughout the box.

Since $\sigma_i=\sq{\boldsymbol P_i}$,
we have
$\sigma_i=\|\boldsymbol P_i\|_{\ell^2}^2$
at every real point. Taking square roots
of the preceding bounds gives
\[
\begin{aligned}
\|\boldsymbol P_0\|_{\ell^2},
\ \|\boldsymbol P_b\|_{\ell^2}
&\le1,\\
\|\boldsymbol P_1\|_{\ell^2},
\ \|\boldsymbol P_2\|_{\ell^2}
&\le2\sqrt2(2d+1),\\
\|\boldsymbol P_3\|_{\ell^2}
&\le\sqrt2\varepsilon^{-2}.
\end{aligned}
\]
By \eqref{eq:d-regime},
$2d+1=O(\varepsilon^{-2})$.
Hence, for all five vectors,
\[
\|\boldsymbol P_i(u,x,y)\|_{\ell^2}
\le C\varepsilon^{-2},
\qquad
(u,x,y)\in I_\varepsilon\times J^2,
\]
for sufficiently small $\varepsilon$.

\textit{Chebyshev extrapolation.}
We now transfer these bounds to the slice
$u=-2\varepsilon$.
Fix $(x,y)\in J^2$.
Each $\boldsymbol P_i(u,x,y)$ is a
vector-valued polynomial in $u$ of
degree at most $d$.
For any vector $\boldsymbol w\in\R^{N_i}$
with $\|\boldsymbol w\|_{\ell^2}=1$,
define the scalar polynomial
\[
p_{\boldsymbol w}(u)
=\langle\boldsymbol w,
\boldsymbol P_i(u,x,y)\rangle.
\]
By the Cauchy--Schwarz inequality,
\[
\begin{aligned}
|p_{\boldsymbol w}(u)|
&=|\langle\boldsymbol w,\boldsymbol P_i(u,x,y)\rangle|\le\|\boldsymbol w\|_{\ell^2}
\|\boldsymbol P_i(u,x,y)\|_{\ell^2}\le C\varepsilon^{-2},\ u\in I_\varepsilon.
\end{aligned}
\]
Since $\deg p_{\boldsymbol w}\le d$,
we apply the Chebyshev extrapolation
inequality \eqref{eq:chebyshev-extrapolation}
with $a=\varepsilon$, $b=1/2$,
and $t=-2\varepsilon$.
The corresponding Chebyshev argument is
\[
\frac{|2t-a-b|}{b-a}
=\frac{|-4\varepsilon-\varepsilon-1/2|}{1/2-\varepsilon}
=\frac{1+10\varepsilon}{1-2\varepsilon}.
\]
Therefore,
\begin{equation}\label{eq:scalar-extrapolation}
\begin{aligned}
|\langle\boldsymbol w,
\boldsymbol P_i(-2\varepsilon,x,y)\rangle|
&=|p_{\boldsymbol w}(-2\varepsilon)|\\
&\le
T_d\left(\frac{1+10\varepsilon}{1-2\varepsilon}\right)
\sup_{u\in I_\varepsilon}|p_{\boldsymbol w}(u)|\\
&\le
C\varepsilon^{-2}
T_d\left(\frac{1+10\varepsilon}{1-2\varepsilon}\right).
\end{aligned}
\end{equation}
The right-hand side is independent of
$\boldsymbol w$.
Since the Euclidean norm satisfies
\[
\|\boldsymbol v\|_{\ell^2}
=\sup_{\|\boldsymbol w\|_{\ell^2}=1}
|\langle\boldsymbol w,\boldsymbol v\rangle|,
\]
taking the supremum over all unit vectors
$\boldsymbol w$ in \eqref{eq:scalar-extrapolation}
yields
\[
\|\boldsymbol P_i(-2\varepsilon,x,y)\|_{\ell^2}
\le C\varepsilon^{-2}
T_d\left(\frac{1+10\varepsilon}{1-2\varepsilon}\right).
\]

For $0<\varepsilon\le1/16$, we have
\[
\frac{1+10\varepsilon}{1-2\varepsilon}
=1+\frac{12\varepsilon}{1-2\varepsilon}
\le1+16\varepsilon.
\]
Using \eqref{eq:chebyshev-growth} and
$\operatorname{arcosh}(1+h)\le\sqrt{2h}$
for $h\ge0$, we obtain
\[
\begin{aligned}
T_d\left(\frac{1+10\varepsilon}{1-2\varepsilon}\right)
&\le\cosh\bigl(d\operatorname{arcosh}(1+16\varepsilon)\bigr)\le e^{d\sqrt{32\varepsilon}}
\le e^{6d\sqrt{\varepsilon}}.
\end{aligned}
\]
Consequently, for all $i\in\{0,b,1,2,3\}$,
\begin{equation}\label{eq:factor-extrapolation}
\begin{gathered}
\|\boldsymbol P_i(-2\varepsilon,x,y)\|_{\ell^2}
\le K_0,\qquad (x,y)\in J^2,\\
K_0:=C\varepsilon^{-2}e^{6d\sqrt{\varepsilon}},
\qquad J=[-1/2,0].
\end{gathered}
\end{equation}

\subsection{Reduction to an analytic identity}

The preceding subsection provides uniform
bounds on the Euclidean norms of the SOS
polynomial vectors on the slice
$u=-2\varepsilon$, for $(x,y)\in J^2$.
To apply Lemma~\ref{lem:corner},
we now derive Wiener norm bounds for
the rescaled polynomial vectors and
construct analytic vectors satisfying
\eqref{eq:corner}.

Set
\begin{equation}\label{eq:scales}
\begin{gathered}
t=2\varepsilon,\qquad
\rho=\varepsilon^{3/2},\\
\eta=\frac{t^4}{\rho^2}=16\varepsilon,
\qquad
R=\frac1{16d^2\rho}.
\end{gathered}
\end{equation}
Since $0<\varepsilon\le1/16$, we have
$\eta=16\varepsilon\le1$.
Moreover, under \eqref{eq:d-regime},
$\log R\ge\frac14\log(1/\varepsilon)$
for sufficiently small $\varepsilon$,
as verified in the degree comparison below.

\textit{Wiener bounds for the rescaled polynomial vectors.}
The coefficients of
$\boldsymbol P_i(-t,x,y)$ are determined
by its partial derivatives at
$(x,y)=(0,0)$.
We use Markov's inequality to bound
these derivatives and hence control
the coefficients of the rescaled vectors.

Fix $i\in\{0,b,1,2,3\}$.
For any vector $\boldsymbol w\in\R^{N_i}$
with $\|\boldsymbol w\|_{\ell^2}=1$,
consider the scalar polynomial
\[
q_{\boldsymbol w}(x,y)
=\langle\boldsymbol w,
\boldsymbol P_i(-t,x,y)\rangle.
\]
By \eqref{eq:factor-extrapolation}
and the Cauchy--Schwarz inequality,
$|q_{\boldsymbol w}(x,y)|
\le\|\boldsymbol w\|_{\ell^2}
\|\boldsymbol P_i(-t,x,y)\|_{\ell^2}
\le K_0$ on $J^2$. Hence
\[
\|q_{\boldsymbol w}\|_{J^2,\infty}\le K_0.
\]

Since $q_{\boldsymbol w}$ has degree
at most $d$ and $J=[-1/2,0]$,
Markov's inequality gives a factor
of at most $4d^2$ for each
differentiation in either variable.
Applying it successively $j$ times
in $x$ and $k$ times in $y$, we obtain
\[
\|\partial_x^j\partial_y^k q_{\boldsymbol w}\|_{J^2,\infty}
\le(4d^2)^{j+k}\|q_{\boldsymbol w}\|_{J^2,\infty}
\le K_0(4d^2)^{j+k}.
\]
Evaluating at $(0,0)$ and using
the definition of $q_{\boldsymbol w}$
gives
\[
\bigl|\langle\boldsymbol w,
\partial_x^j\partial_y^k\boldsymbol P_i(-t,0,0)
\rangle\bigr|\le K_0(4d^2)^{j+k}.
\]
Since
\[
\|\boldsymbol v\|_{\ell^2}
=\sup_{\|\boldsymbol w\|_{\ell^2}=1}
|\langle\boldsymbol w,\boldsymbol v\rangle|,
\]
taking the supremum over all unit
vectors $\boldsymbol w$ yields
\[
\bigl\|
\partial_x^j\partial_y^k
\boldsymbol P_i(-t,0,0)
\bigr\|_{\ell^2}
\le K_0(4d^2)^{j+k}.
\]

By Taylor's formula, the coefficient
of $x^jy^k$ in
$\boldsymbol P_i(-t,x,y)$ is
\[
\frac1{j!k!}
\partial_x^j\partial_y^k
\boldsymbol P_i(-t,0,0).
\]
After the substitution
$(x,y)=(\rho X,\rho Y)$,
the coefficient of $X^jY^k$ becomes
\[
\frac{\rho^{j+k}}{j!k!}
\partial_x^j\partial_y^k
\boldsymbol P_i(-t,0,0).
\]
Its Euclidean norm is therefore
bounded by
\[
K_0\frac{(4d^2\rho)^{j+k}}{j!k!}.
\]
Summing these coefficient norms
with weights $R^{j+k}$ gives
\begin{equation}\label{eq:rescaled-wiener-bound}
\begin{aligned}
\|\boldsymbol P_i(-t,\rho X,\rho Y)\|_R
&\le K_0\sum_{j,k\ge0}
\frac{(4d^2\rho R)^{j+k}}{j!k!}\\
&=K_0e^{8d^2\rho R}=e^{1/2}K_0.
\end{aligned}
\end{equation}
where the last equality uses
$R=1/(16d^2\rho)$.

\textit{Construction of the analytic identity.}
Write
\[
\widetilde{\boldsymbol P}_i
=\boldsymbol P_i(-t,\rho X,\rho Y),
\qquad
\widetilde b
=3-t^2-\rho^2(X^2+Y^2).
\]
Substituting $(x,y)=(\rho X,\rho Y)$
into \eqref{eq:restricted-certificate}
and using $t=2\varepsilon$ gives
\begin{equation}\label{eq:corner-substitution}
(\rho^2XY+t^4)\sq{\widetilde{\boldsymbol P}_3}
={}\frac12+t^{-1}\sq{\widetilde{\boldsymbol P}_0}
+t^{-1}\widetilde b\sq{\widetilde{\boldsymbol P}_b}+\rho X\sq{\widetilde{\boldsymbol P}_1}
+\rho Y\sq{\widetilde{\boldsymbol P}_2}.
\end{equation}

The vectors $\widetilde{\boldsymbol P}_i$
remain polynomial.
To express the weighted term
$\widetilde b\sq{\widetilde{\boldsymbol P}_b}$
as a square, we introduce the
generally nonpolynomial factor
$\sqrt{\widetilde b}$ and show
that it defines an analytic function
with finite radius-$R$ Wiener norm.

Write
\[
\widetilde b=(3-t^2)(1-Z),
\qquad
Z=\frac{\rho^2(X^2+Y^2)}{3-t^2}.
\]
By the choice of $R$,
\[
\|Z\|_R=\frac{2\rho^2R^2}{3-t^2}
\le\frac1{256d^4}<\frac12.
\]
Consequently, the binomial series
\[
\sqrt{\widetilde b}
=\sqrt{3-t^2}
\sum_{n\ge0}
\binom{1/2}{n}(-Z)^n
\]
converges in the radius-$R$ Wiener
norm and defines an analytic square
root of $\widetilde b$ with positive
constant term.
Moreover, by submultiplicativity,
\begin{equation}\label{eq:sqrt-b-wiener-bound}
\begin{aligned}
\|\sqrt{\widetilde b}\|_R
&\le\sqrt{3-t^2}
\sum_{n\ge0}
\left|\binom{1/2}{n}\right|\|Z\|_R^n\\
&\le\sqrt{3-t^2}
\sum_{n\ge0}
\left|\binom{1/2}{n}\right|2^{-n}
\le2\sqrt3<4.
\end{aligned}
\end{equation}

We now define
\begin{equation}\label{eq:analytic-amplitudes}
\begin{gathered}
\boldsymbol B
=\rho\,\widetilde{\boldsymbol P}_3,\\[2pt]
\boldsymbol A_0=t^{-1/2}
\bigl(
\widetilde{\boldsymbol P}_0
\oplus\sqrt{\widetilde b}\,
\widetilde{\boldsymbol P}_b
\bigr),\\[2pt]
\boldsymbol A_1
=\sqrt\rho\,\widetilde{\boldsymbol P}_1,
\qquad
\boldsymbol A_2
=\sqrt\rho\,\widetilde{\boldsymbol P}_2.
\end{gathered}
\end{equation}
Here $\oplus$ denotes vector concatenation,
so that
$\sq{\boldsymbol U\oplus\boldsymbol V}
=\sq{\boldsymbol U}+\sq{\boldsymbol V}$.
These vectors are analytic because
the $\widetilde{\boldsymbol P}_i$
are polynomial and
$\sqrt{\widetilde b}$ is analytic.

By the definitions in
\eqref{eq:analytic-amplitudes} and
$\eta=t^4/\rho^2$, we have
\[
\begin{gathered}
(\rho^2XY+t^4)\sq{\widetilde{\boldsymbol P}_3}
=(XY+\eta)\sq{\boldsymbol B},\\
t^{-1}\sq{\widetilde{\boldsymbol P}_0}
+t^{-1}\widetilde b\sq{\widetilde{\boldsymbol P}_b}
=\sq{\boldsymbol A_0},\\
\rho X\sq{\widetilde{\boldsymbol P}_1}
=X\sq{\boldsymbol A_1},\\
\rho Y\sq{\widetilde{\boldsymbol P}_2}
=Y\sq{\boldsymbol A_2}.
\end{gathered}
\]

Thus \eqref{eq:corner-substitution}
becomes
\[
(XY+\eta)\sq{\boldsymbol B}
=\frac12+\sq{\boldsymbol A_0}
+X\sq{\boldsymbol A_1}
+Y\sq{\boldsymbol A_2},
\]
which is precisely \eqref{eq:corner}.

\textit{Wiener bounds for the analytic vectors.}
By \eqref{eq:analytic-amplitudes} and
\eqref{eq:rescaled-wiener-bound}, we obtain
\[
\begin{aligned}
\|\boldsymbol B\|_R
&=\rho\|\widetilde{\boldsymbol P}_3\|_R
\le\rho e^{1/2}K_0,\\
\|\boldsymbol A_1\|_R
&=\sqrt\rho\,
\|\widetilde{\boldsymbol P}_1\|_R
\le\sqrt\rho\,e^{1/2}K_0,\\
\|\boldsymbol A_2\|_R
&=\sqrt\rho\,
\|\widetilde{\boldsymbol P}_2\|_R
\le\sqrt\rho\,e^{1/2}K_0.
\end{aligned}
\]
For $\boldsymbol A_0$, the triangle
inequality, submultiplicativity of
scalar-vector Wiener norms, and the bounds
\eqref{eq:rescaled-wiener-bound} and
\eqref{eq:sqrt-b-wiener-bound} give
\[
\begin{aligned}
\|\boldsymbol A_0\|_R
&\le t^{-1/2}
\bigl(\|\widetilde{\boldsymbol P}_0\|_R
+\|\sqrt{\widetilde b}\,\widetilde{\boldsymbol P}_b\|_R\bigr)\\
&\le t^{-1/2}
\bigl(\|\widetilde{\boldsymbol P}_0\|_R
+\|\sqrt{\widetilde b}\|_R
\|\widetilde{\boldsymbol P}_b\|_R\bigr)\\
&\le t^{-1/2}
\bigl(e^{1/2}K_0+4e^{1/2}K_0\bigr)
=5e^{1/2}t^{-1/2}K_0.
\end{aligned}
\]
Since $t=2\varepsilon$ and
$\rho=\varepsilon^{3/2}$ by
\eqref{eq:scales}, we have
$t^{-1/2}\ge1$ and $\rho\le1$.
As $K_0\ge1$, the bounds above imply
$W_R\le C t^{-1/2}K_0$.
Substituting $K_0$ from
\eqref{eq:factor-extrapolation} yields
\begin{equation}\label{eq:K-upper}
W_R\le
C\varepsilon^{-5/2}
e^{6d\sqrt\varepsilon}.
\end{equation}

\subsection{Degree comparison and sharpness}

The preceding subsection constructs
analytic vectors satisfying
\eqref{eq:corner} and establishes
the Wiener norm upper bound
\eqref{eq:K-upper}.
On the other hand,
Lemma~\ref{lem:corner} provides a
lower bound on the same Wiener norm
in terms of $\log R$.
We now show that our choice of $R$
makes $\log R$ sufficiently large,
and then compare these bounds to
derive the degree lower bound.

Put $L=\log(1/\varepsilon)$.
Since $\eta=16\varepsilon$,
we have $\eta\le1$.
To apply Lemma~\ref{lem:corner},
we also need $R>1$.
Recall from \eqref{eq:scales} that
$R=1/(16d^2\rho)$.
By the degree assumption
\eqref{eq:d-regime},
$d^2\le\varepsilon^{-1}L^4$.
Together with
$\rho=\varepsilon^{3/2}$,
this gives
\[
\begin{aligned}
d^2\rho
&\le\varepsilon^{-1}L^4
\varepsilon^{3/2}
=\varepsilon^{1/2}L^4.
\end{aligned}
\]
Consequently,
\[
\begin{aligned}
\log R
&=-\log16-\log(d^2\rho)\\
&\ge\frac12L-4\log L-\log16
\ge\frac14L,
\end{aligned}
\]
where the last inequality holds
for sufficiently large $L$.
In particular, $R>1$.

We can therefore apply
Lemma~\ref{lem:corner} to obtain
\[
\begin{aligned}
\log W_R
&\ge\frac{(\log R)^2}{32\log16}
-\frac{\log16+\log20}{2}\\
&\ge\frac{L^2}{512\log16}
-\frac{\log16+\log20}{2}.
\end{aligned}
\]
On the other hand, taking logarithms
in \eqref{eq:K-upper} gives
\[
\log W_R
\le6d\sqrt\varepsilon+\frac52L+C.
\]
Combining the lower and upper bounds
and absorbing fixed constants into $C$,
we obtain
\begin{equation}\label{eq:degree-comparison}
6d\sqrt\varepsilon+\frac52L+C
\ge\frac{L^2}{512\log16}.
\end{equation}

Since $L\to\infty$ as
$\varepsilon\to0$, the linear term
$\frac52L+C$ can be absorbed into
the quadratic term on the right.
Specifically, for sufficiently large $L$,
\[
\frac52L+C
\le\frac{L^2}{1024\log16}.
\]
Hence \eqref{eq:degree-comparison} implies
\[
6d\sqrt\varepsilon
\ge\frac{L^2}{1024\log16},
\]
and therefore
\[
d\ge\frac1{6144\log16}
\varepsilon^{-1/2}L^2.
\]
The same lower bound holds in the
case excluded by \eqref{eq:d-regime},
where $d>\varepsilon^{-1/2}L^2$.

Finally, since
\[
\delta_\varepsilon
=\frac{\varepsilon}{4\sqrt3+\varepsilon}
\asymp\varepsilon,
\]
and
$d_{\M_*}(f_\varepsilon)
=d_{\M_*}(u+\varepsilon)$,
the lower bound becomes
\[
d_{\M_*}(f_\varepsilon)
\ge c\delta_\varepsilon^{-1/2}
\log^2(1/\delta_\varepsilon).
\]
Since the original certificate was
arbitrary, this bound holds for every
Putinar certificate.
Together with the corresponding upper
bound, it establishes the claimed
sharpness of the margin dependence.

\section{Sharp bounds with at most two inequality generators}
\label{sec:few-generators}

For presentations with at most two inequality generators, the second
logarithm in Theorem~\ref{thm:main} can be removed.  The resulting uniform
bound is
\[
 d_{\M}(f)\le C_{\M,D}\delta^{-1/2}\log(2/\delta).
\]
We prove this by converting the Stengle identity directly into the
original quadratic module through a two-generator refutation identity.
The conversion has fixed degree and polynomial certificate size, so the
generator-product construction of Section~\ref{sec:products} is unnecessary.
We then give matching lower-bound families with exactly one and exactly
two inequalities.  A square-root lift of the one-generator family shows
that the same logarithm is necessary with zero inequalities when fixed
equality constraints are allowed.

Here the number of inequalities is the number $m$ in
\eqref{eq:module}, including a ball inequality if it is listed among the
$g_i$.  Equality constraints remain in their specified ideal.  Sharpness
means a uniform upper bound for every fixed admissible presentation and
attainment by a fixed example in each class; it does not assert a lower
bound for every presentation or every positive polynomial.

\subsection{Internal Stengle certificates with at most two inequalities}
\label{subsec:few-internal-stengle}

Throughout this section, the number of inequality generators counts the
polynomials $g_i$ in the fixed presentation \eqref{eq:module}; the equality
generators $h_j$ are retained in the ideal.  The following algebraic identity
is the key to treating two inequalities.  It applies to arbitrary equality
ideals and does not require an Archimedean hypothesis.

\begin{lemma}[Compression of a two-generator refutation]
\label{lem:two-refutation}
Let $I=(h_1,\ldots,h_\ell)\subseteq\R[\mathbf x]$, and set
\[
 M_2=\SOS[\mathbf x]+g_1\SOS[\mathbf x]
       +g_2\SOS[\mathbf x]+I,
 \qquad
 P_2=M_2+g_1g_2\SOS[\mathbf x].
\]
If $-1\in P_2$, then $-1\in M_2$.  More precisely, a preordering
refutation of degree budget at most $L$ can be converted into a quadratic
module refutation of degree budget at most $2L+3G$, where
$G=\max\{0,\deg g_1,\deg g_2\}$ and identically zero generators are omitted
when taking this maximum.
\end{lemma}

\begin{proof}
Discard identically zero generator-weighted terms before tracking
degree budgets. Group the given preordering representation as
\[
 -1=A+g_2B+e,\qquad
 A=\sigma_{00}+g_1\sigma_{10},\qquad
 B=\sigma_{01}+g_1\sigma_{11},\qquad e\in I,
\]
where all four $\sigma_{ij}$ are SOS.  Define
\[
 q=1+A,\qquad C=qB.
\]
The one-generator quadratic module $\QM(g_1)$ is a preordering.  In
particular, $C\in\QM(g_1)$, as witnessed explicitly by
\[
 \begin{aligned}
 C={}&(1+\sigma_{00})\sigma_{01}
       +g_1^2\sigma_{10}\sigma_{11}\\
     &+g_1\bigl((1+\sigma_{00})\sigma_{11}
                         +\sigma_{10}\sigma_{01}\bigr).
 \end{aligned}
\]
Products of SOS polynomials are SOS, so this is a valid representation.
Now put
\[
 T=2A+A^2
   +g_2\left(\frac{B(g_2-1)}2\right)^2
   +C\left(\frac{g_2+1}2\right)^2.
\]
Every term in this expression belongs to $\QM(g_1,g_2)$.  Direct expansion
gives the polynomial identity
\begin{equation}\label{eq:two-refutation-identity}
 T+1=(1+A+g_2B)
       \left(1+A+\frac{B(g_2-1)^2}{4}\right).
\end{equation}
Since $1+A+g_2B=-e$, it follows that
\[
 -1=T+e\left(1+A+\frac{B(g_2-1)^2}{4}\right)\in M_2.
\]
This also specifies the new equality multipliers.

The chosen representations of $A$ and $B$ have budgets at most $L$;
the displayed representation of $C$ has budget at most $2L$.
Consequently, the weighted square containing $g_2B^2$ has degree at
most $2L+3G$, the term involving $C$ has budget at most $2L+2G$, and
the other module terms have budget at most $2L$.
Each original equality-ideal term is multiplied by a polynomial of degree
at most $L+2G$, so its new budget is at most $2L+2G$.
These estimates prove the asserted bound.
\end{proof}

The lemma concerns refutations.  It does not assert that $M_2=P_2$:
membership of a general product $g_1g_2$ in $M_2$ is unnecessary.

\begin{lemma}[Internal Stengle certificates with controlled size]
\label{lem:few-internal-size}
Assume $m\le2$ and fix the Archimedean witness
\eqref{eq:normalized-witness}.  For fixed $D$, there are constants
$A,\nu>0$ and an integer $\widehat B$ such that
\[
 \deg a\le D,\qquad\|a\|_2\le\tfrac12,\qquad
 a|_K\ge s,\qquad 0<s\le\tfrac14
\]
imply the existence of an internal Stengle certificate satisfying
\begin{equation}\label{eq:few-internal-size}
 ap=1+r,\qquad p,r\in\M_{\widehat B},\qquad
 \size_{p,r}\le As^{-\nu}.
\end{equation}
All constants depend only on the fixed presentation, witness, and $D$.
\end{lemma}

\begin{proof}
Lemma~\ref{lem:stengle} supplies
\[
 ap_0=1+r_0,\qquad p_0,r_0\in\PP_{B_*},\qquad
 \size_{p_0,r_0}\le S_*:=A_0s^{-\nu_0},
\]
where $B_*,A_0,\nu_0$ are fixed and we may assume $S_*\ge1$.
If $m\le1$, then $\PP=\M$, so the result follows directly.
Suppose $m=2$. If either generator is identically zero, discard its
weighted terms and apply the preceding case. Hence both generators
may be assumed nonzero.

Group the chosen preordering representations as
\[
 p_0=P_0+g_2P_1+e_p,\qquad
 r_0=R_0+g_2R_1+e_r,
\]
where $P_0,P_1,R_0,R_1\in\QM(g_1)$ and $e_p,e_r\in I$.
Each monogenic-module representation has budget at most $B_*$ and
size at most $S_*$.
Set $e=e_r-ae_p\in I$.  Then
\begin{equation}\label{eq:few-grouped-stengle}
 a(P_0+g_2P_1)=1+R_0+g_2R_1+e.
\end{equation}
Define the polynomials
\[
 \alpha=\frac{g_2-1}{2},\qquad
 \beta=\frac{g_2+1}{2},\qquad
 U=R_0-aP_0,\qquad V=R_1-aP_1,
\]
and put $W=1+U+\alpha^2V$.
We claim that the following formulas give the required internal
certificate:
\begin{equation}\label{eq:few-internal-construction}
 \begin{aligned}
 p={}&2P_0+\beta^2\bigl((1+R_0)P_1+P_0R_1\bigr),\\
 r={}&2R_0+U^2+g_2(\alpha V)^2\\
     &+\beta^2\bigl((1+R_0)R_1+a^2P_0P_1\bigr)+eW.
 \end{aligned}
\end{equation}
First, these are representations in the original module.
The quadratic module $\QM(g_1)$ is a preordering, so all four products
\[
 (1+R_0)P_1,\quad P_0R_1,\quad
 (1+R_0)R_1,\quad P_0P_1
\]
belong to $\QM(g_1)$ with their explicit product representations.
Multiplication by $\beta^2$ or $a^2$ preserves this membership;
$U^2$ is a square, $g_2(\alpha V)^2$ is an allowed weighted square,
and $eW\in I$.  Thus $p,r\in\M$.

To verify the identity, note from \eqref{eq:few-grouped-stengle} that
\[
 1+U+g_2V=-e.
\]
The polynomial identity \eqref{eq:two-refutation-identity} holds for
arbitrary polynomials in place of $A$ and $B$; applying it to $U,V$ gives
\[
 \begin{aligned}
 T&=2U+U^2+g_2(\alpha V)^2+\beta^2(1+U)V,\\
 T+1&=(1+U+g_2V)(1+U+\alpha^2V)=-eW.
 \end{aligned}
\]
On the other hand,
\[
 (1+U)V=(1+R_0)R_1
  -a\bigl((1+R_0)P_1+P_0R_1\bigr)+a^2P_0P_1.
\]
Substitution into $T$ shows that $T=r-eW-ap$.
Consequently $ap=1+r$, as claimed.

We finally track the degree and size of this explicit transformation.
Put
\[
 G=\max\{0,\deg g_1,\deg g_2\}.
\]
The polynomials $U,V$ have degrees at most $B_*+D$.
Products of two chosen monogenic representations have budget at most
$2B_*$, while the new factors $a^2$, $\alpha^2$, and $\beta^2$ have fixed
bounded degrees.  The term $g_2(\alpha V)^2$ has degree at most
$2(B_*+D)+3G$.
Every equality multiplier in $e$ has weighted degree at most $B_*+D$;
multiplication by $W$, of degree at most $B_*+D+2G$, gives budget at
most $2(B_*+D)+2G$.
Thus
\[
 \widehat B=2(B_*+D)+3G
\]
is a common budget for the representations of $p$ and $r$.

For size, the four original monogenic representations have size
$O(S_*)$ and the four polynomials have radius-two norms $O(S_*)$.
The bound $\|a\|_2\le1/2$ implies
$\|U\|_2+\|V\|_2+\|W\|_2=O(S_*)$.
Products of SOS representations multiply their size bounds; the fixed
factors involving $g_1,g_2,\alpha,\beta$ contribute only constant factors.
The new square factors $U$ and $\alpha V$ have squared norms
$O(S_*^2)$, and multiplication by $a^2$ likewise contributes only a
bounded factor.  The equality multipliers in $e$ have total norm
$O(S_*)$ and are multiplied by $W$, whose norm is $O(S_*)$.
It follows that the total certificate size in
\eqref{eq:few-internal-construction} is at most
\[
 C(1+S_*)^2\le A s^{-2\nu_0}.
\]
Renaming the exponent proves \eqref{eq:few-internal-size}.
\end{proof}

\subsection{The uniform single-logarithm upper bound}
\label{subsec:few-upper}

\begin{theorem}[Upper bound for zero, one, or two inequalities]
\label{thm:few-upper}
Fix the presentation \eqref{eq:module} with $0\le m\le2$, a specified
Archimedean witness \eqref{eq:ball-general}, and $D\ge1$.
For $Q=[-4r,4r]^n$, there is a constant $C_{\M,D}$ such that
\[
 \deg f\le D,\qquad\|f\|_{Q,\infty}\le1,\qquad
 f|_K\ge\delta>0
\]
imply
\begin{equation}\label{eq:few-upper}
 d_{\M}(f)\le C_{\M,D}\delta^{-1/2}\log(2/\delta).
\end{equation}
The certificate uses the original inequality generators and equality
ideal.  The constant is uniform over all objectives satisfying the
displayed conditions.
\end{theorem}

\begin{proof}
Use the normalized coordinates \eqref{eq:normalized-witness}, and choose
$K_D$ so that
$\|f\|_2\le K_D\|f\|_{[-4,4]^n,\infty}$ for $\deg f\le D$.
As in the proof of Theorem~\ref{thm:main}, put
\[
 H_D=4(K_D+1),\qquad
 a=\frac{f-\delta/2}{H_D},\qquad
 s=\frac{\delta}{2H_D}.
\]
Then $\|a\|_2\le1/2$ and $a|_K\ge s>0$.
Lemma~\ref{lem:few-internal-size} gives
\[
 ap=1+r,\qquad p,r\in\M_{\widehat B},\qquad
 \size_{p,r}\le As^{-\nu}.
\]
Since the generators are fixed, their norms also give
$\|p\|_2\le Cs^{-\nu}$.

Apply the Chebyshev elimination \eqref{eq:elimination} directly to this
internal identity, with $\varepsilon=s/2$ and
$C_p=2(1+\|p\|_2)$.
Equation~\eqref{eq:N-bound} yields
\[
 N\le Cs^{-1/2}\log(2/s).
\]
The size estimate \eqref{eq:elim-size} then gives an analytic certificate
for $a+s/2$ whose logarithmic size is at most
$Cs^{-1/2}\log(2/s)$.
Here no generator-product conversion is used.

Since $H_D(a+s/2)=f-\delta/4$, multiplication by $H_D$ provides an
analytic certificate for $f-\delta/4$ with
\[
 \log\size\le C\delta^{-1/2}\log(2/\delta).
\]
Lemma~\ref{lem:exactification}, applied with $\tau=\delta/4$, now gives
\[
 \begin{aligned}
 d_{\M}(f)
 &\le C\left(D+\log\frac{4\size}{\delta}+1\right)\\
 &\le C_{\M,D}\delta^{-1/2}\log(2/\delta).
 \end{aligned}
\]
Undoing the fixed coordinate scaling proves the theorem.
\end{proof}

\subsection{A sharp polynomial example with a higher-order boundary zero}
\label{subsec:quartic-one-example}

We construct a fixed one-generator module on the unit disk.  Set
\begin{equation}\label{eq:quartic-generator}
 t=x^2+y^2,\qquad u=1-t,\qquad
 g=u(x^4+u^4),\qquad
 \M^{(4)}_1=\QM(g)\subseteq\R[x,y].
\end{equation}
The nonnegative factor $x^4+u^4$ vanishes only when $x=u=0$.
It follows that the feasible set is exactly
\[
 K^{(4)}_1=\{(x,y):x^2+y^2\le1\}.
\]
At either boundary point $(0,\pm1)$, the pair $(x,u)$ forms a smooth
local coordinate system, and the generator is $ux^4+u^5$ in these
coordinates.  Thus its order of vanishing is five, with quartic
vanishing in the tangential factor.

Define
\begin{equation}\label{eq:quartic-sos-witness}
 \begin{aligned}
 S={}&t(t-2)^4+3t^2(t-\tfrac73)^2
       +\tfrac{14}{3}(t-\tfrac97)^2\\
     &+tx^4+2x^2y^2+y^4+\tfrac27.
 \end{aligned}
\end{equation}
This is explicitly SOS: for instance,
$t(t-2)^4=[x(t-2)^2]^2+[y(t-2)^2]^2$ and
$tx^4=(x^3)^2+(x^2y)^2$, and the remaining terms are positive
constant multiples of squares.  Expansion gives
\begin{equation}\label{eq:quartic-archimedean}
 9-x^2-y^2=g+S\in(\M^{(4)}_1)_{10}.
\end{equation}
Hence \eqref{eq:quartic-generator} is a specified Archimedean
presentation with a degree-ten witness.

\begin{proposition}[Quartic-degeneracy lower bound]
\label{prop:quartic-lower}
For the fixed module \eqref{eq:quartic-generator},
\begin{equation}\label{eq:quartic-lower}
 d_{\M^{(4)}_1}(u+\varepsilon)
 \ge\frac1{64}\varepsilon^{-1/2}\log(1/\varepsilon)
 \qquad(0<\varepsilon\le e^{-8}).
\end{equation}
Consequently,
\begin{equation}\label{eq:quartic-sharp}
 d_{\M^{(4)}_1}(u+\varepsilon)
 =\Theta\!\left(\varepsilon^{-1/2}
                         \log(1/\varepsilon)\right).
\end{equation}
\end{proposition}

\begin{proof}
Consider any certificate of degree budget $d$,
\begin{equation}\label{eq:quartic-certificate}
 u+\varepsilon
 =\sq{\boldsymbol P_0}+g\sq{\boldsymbol P_1}.
\end{equation}
Each polynomial factor has degree at most $d$; this loose bound
will suffice.  We first assume $0<\varepsilon\le1/8$.

Fix $s\in[\varepsilon,1/2]$ and put $r_s=\sqrt{1-s}$.
On the circle $x^2+y^2=1-s$, the certificate has nonnegative
summands and $u=s$.  Hence
\[
 s(x^4+s^4)\|\boldsymbol P_1(x,y)\|_{\ell^2}^2
 \le s+\varepsilon\le2s,
\]
which implies
\begin{equation}\label{eq:quartic-weighted-circle}
 \|x^2\boldsymbol P_1(x,y)\|_{\ell^2}\le\sqrt2
 \qquad(x^2+y^2=1-s).
\end{equation}
Parametrize the circle by
$(x,y)=(r_s\sin\theta,r_s\cos\theta)$ and define
\[
 \boldsymbol V_s(\theta)
 =r_s^2\sin^2\theta\,
       \boldsymbol P_1(r_s\sin\theta,r_s\cos\theta).
\]
Its components are trigonometric polynomials of degree at most
$d+2$.  Applying Bernstein's inequality twice
\cite[Theorem~1.1]{small-Bernstein2019} gives
\[
 \|\boldsymbol V_s''(0)\|_{\ell^2}
 \le(d+2)^2\sup_{\theta\in\R}
                         \|\boldsymbol V_s(\theta)\|_{\ell^2}
 \le\sqrt2(d+2)^2.
\]
For completeness, the vector estimate follows by applying the scalar
inequality to
$\langle\boldsymbol w,\boldsymbol V_s\rangle$ for every
$\|\boldsymbol w\|_{\ell^2}=1$ and then taking the supremum over
$\boldsymbol w$.
Since $\sin^2\theta$ and its first derivative vanish at $\theta=0$,
the product rule yields
\[
 \boldsymbol V_s''(0)=2r_s^2\boldsymbol P_1(0,r_s).
\]
As $r_s^2\ge1/2$, we have
\begin{equation}\label{eq:quartic-interior-bound}
 \|\boldsymbol P_1(0,r_s)\|_{\ell^2}
 \le\sqrt2(d+2)^2.
\end{equation}

Thus the polynomial vector
$\boldsymbol p(y)=\boldsymbol P_1(0,y)$ is bounded by
$\sqrt2(d+2)^2$ on
\[
 [a,b_\varepsilon]
 =\left[\frac1{\sqrt2},\sqrt{1-\varepsilon}\right].
\]
Consider $y_\varepsilon=\sqrt{1+2\varepsilon}$, just outside the disk.
For $0<\varepsilon\le1/8$,
\[
 b_\varepsilon-a\ge\frac15,
 \qquad
 y_\varepsilon-b_\varepsilon
 =\frac{3\varepsilon}
        {\sqrt{1+2\varepsilon}+\sqrt{1-\varepsilon}}
 \le2\varepsilon.
\]
Therefore
\[
 \frac{2y_\varepsilon-a-b_\varepsilon}
      {b_\varepsilon-a}\le1+20\varepsilon.
\]
The vector-valued form of
\eqref{eq:chebyshev-extrapolation}, together with
$\operatorname{arcosh}(1+v)\le\sqrt{2v}$ for $v\ge0$, gives
\begin{equation}\label{eq:quartic-exterior-upper}
 \begin{aligned}
 \|\boldsymbol P_1(0,y_\varepsilon)\|_{\ell^2}
 &\le\sqrt2(d+2)^2T_d(1+20\varepsilon)\\
 &\le\sqrt2(d+2)^2\exp(7d\sqrt\varepsilon).
 \end{aligned}
\end{equation}

At the same point, $u=-2\varepsilon$ and
$g=-32\varepsilon^5$.  Evaluating
\eqref{eq:quartic-certificate} there yields
\[
 32\varepsilon^5
       \|\boldsymbol P_1(0,y_\varepsilon)\|_{\ell^2}^2
 =\varepsilon+
       \|\boldsymbol P_0(0,y_\varepsilon)\|_{\ell^2}^2
 \ge\varepsilon.
\]
Combining this estimate with \eqref{eq:quartic-exterior-upper} gives
\begin{equation}\label{eq:quartic-degree-inequality}
 \varepsilon^{-2}
 \le8(d+2)^2\exp(7d\sqrt\varepsilon).
\end{equation}
Let $L=\log(1/\varepsilon)$ and
$z=(d+2)\sqrt\varepsilon>0$.  Taking logarithms in
\eqref{eq:quartic-degree-inequality} and using $d\sqrt\varepsilon\le z$
gives
\[
 L\le\log8+2\log z+7z\le\log8+9z,
\]
where the last step uses $\log z\le z$.
If $\varepsilon\le e^{-8}$, then $L\ge8>2\log8$ and therefore
$z\ge L/18$.
Also $L/(36\sqrt\varepsilon)\ge2$ in this range, so
\[
 d\ge\frac{L}{18\sqrt\varepsilon}-2
   \ge\frac{L}{36\sqrt\varepsilon}
   \ge\frac{L}{64\sqrt\varepsilon}.
\]
This proves \eqref{eq:quartic-lower}.
The matching upper bound follows from Theorem~\ref{thm:few-upper}
after normalization of the fixed-degree family $u+\varepsilon$.
\end{proof}

The fixed witness \eqref{eq:quartic-archimedean} gives the normalization
box $[-12,12]^2$.  Consequently,
\begin{equation}\label{eq:quartic-normalization}
 f_\varepsilon
 =\frac{1-x^2-y^2+\varepsilon}{287-\varepsilon},
 \qquad
 \delta_\varepsilon=\frac{\varepsilon}{287-\varepsilon}
\end{equation}
satisfy
$\|f_\varepsilon\|_{[-12,12]^2,\infty}=1$ and
$\min_{K^{(4)}_1}f_\varepsilon=\delta_\varepsilon$
for $0<\varepsilon\le e^{-8}$.
Since positive scaling preserves the certificate degree and
$\delta_\varepsilon\asymp\varepsilon$, we obtain the normalized form
\[
 d_{\M^{(4)}_1}(f_\varepsilon)
 =\Theta\!\left(\delta_\varepsilon^{-1/2}
                         \log(1/\delta_\varepsilon)\right).
\]

\subsection{Two nonredundant polynomial inequalities}
\label{subsec:quartic-two-example}

Add an independent variable $w$ and define
\[
 \M^{(4)}_2=\QM\bigl(g(x,y),1-w^2\bigr)
 \subseteq\R[x,y,w].
\]
The feasible set is $K^{(4)}_1\times[-1,1]$, and neither inequality
is redundant in its description.  Equation~\eqref{eq:quartic-archimedean}
gives the fixed witness
\[
 10-x^2-y^2-w^2=g+S+(1-w^2)\in(\M^{(4)}_2)_{10}.
\]
For a polynomial $F\in\R[x,y]$, every $\M^{(4)}_1$ certificate is
also an $\M^{(4)}_2$ certificate.  Conversely, substituting $w=1$
into an $\M^{(4)}_2$ certificate eliminates its second weighted SOS
term without increasing the degree.  Hence
\begin{equation}\label{eq:quartic-two-degree-equality}
 d_{\M^{(4)}_2}(F)=d_{\M^{(4)}_1}(F)
 \qquad(F\in\M^{(4)}_1).
\end{equation}
In particular,
\[
 d_{\M^{(4)}_2}(u+\varepsilon)
 =\Theta\!\left(\varepsilon^{-1/2}\log(1/\varepsilon)\right).
\]
The witness gives the box $[-4\sqrt{10},4\sqrt{10}]^3$.
Thus division by $319-\varepsilon$ gives a unit-sup-norm objective
with positivity margin $\varepsilon/(319-\varepsilon)\asymp\varepsilon$.

\subsection{An equality-only polynomial presentation}
\label{subsec:quartic-zero-example}

Introduce $z$ and define
\begin{equation}\label{eq:quartic-zero-module}
 h(x,y,z)=z^2-g(x,y),\qquad
 \M^{(4)}_0=\SOS[x,y,z]+(h).
\end{equation}
There are no inequality generators in this presentation.  Its feasible
set projects onto the unit disk, and it is Archimedean because
\begin{equation}\label{eq:quartic-zero-witness}
 9-x^2-y^2-z^2=S-h\in(\M^{(4)}_0)_{10}.
\end{equation}

\begin{lemma}[Degree comparison for the degree-ten lift]
\label{lem:quartic-lift}
For every $F\in\R[x,y]$ belonging to $\M^{(4)}_1$,
\begin{equation}\label{eq:quartic-lift-comparison}
 \tfrac15d_{\M^{(4)}_1}(F)
 \le d_{\M^{(4)}_0}(F)
 \le d_{\M^{(4)}_1}(F).
\end{equation}
\end{lemma}

\begin{proof}
A certificate $F=\sum_jA_j^2+g\sum_jB_j^2$ lifts to
\[
 F=\sum_jA_j^2+\sum_j(zB_j)^2-h\sum_jB_j^2.
\]
As $\deg h=\deg g=10$, the budget does not increase.

Conversely, take an equality-only certificate
$F=\sum_jP_j(x,y,z)^2+h\lambda$ of budget $d$.
Write
\[
 P_j(x,y,z)=A_j(x,y,z^2)+zB_j(x,y,z^2).
\]
Averaging under $z\mapsto-z$ and then substituting $z^2=g(x,y)$
gives the original-module certificate
\[
 F=\sum_jA_j(x,y,g)^2+g\sum_jB_j(x,y,g)^2.
\]
Put $r_j=\deg P_j\le d/2$.
An even monomial $x^a y^b z^{2k}$ has
$a+b+2k\le r_j$ and becomes a polynomial of degree at most
$a+b+10k\le5r_j$.
An odd monomial $x^a y^b z^{2k+1}$ contributes to $B_j(x,y,g)$
a term of degree at most $a+b+10k\le5r_j-5$.
Every resulting unweighted square therefore has degree at most
$10r_j\le5d$, and every weighted square has degree at most
$10+2(5r_j-5)=10r_j\le5d$.
This proves the comparison.
\end{proof}

It follows that
\[
 d_{\M^{(4)}_0}(u+\varepsilon)
 =\Theta\!\left(\varepsilon^{-1/2}\log(1/\varepsilon)\right).
\]
The normalization box associated with \eqref{eq:quartic-zero-witness}
is $[-12,12]^3$; the family \eqref{eq:quartic-normalization}, independent
of $z$, again has unit sup norm and minimum $\delta_\varepsilon$.
This supplies a polynomial equality-only example for which the
logarithm cannot be removed.

\begin{corollary}[Sharp worst-case order for zero, one, and two inequalities]
\label{cor:few-sharp}
Under the fixed-presentation assumptions of Theorem~\ref{thm:main},
the uniform bound
\[
 d_{\M}(f)=O_{\M,D}\!\left(\delta^{-1/2}\log(2/\delta)\right)
\]
holds whenever $0\le m\le2$.  For each $m\in\{0,1,2\}$ there is a fixed
Archimedean presentation and a normalized family of degree-two objectives
with margin $\delta_\varepsilon\downarrow0$ for which
\[
 d_{\M}(f_\varepsilon)
 =\Theta\!\left(\delta_\varepsilon^{-1/2}
                         \log(1/\delta_\varepsilon)\right).
\]
For $m=0$, the example uses one equality constraint.  For $m=1$ and
$m=2$, the examples have no equalities and their feasible sets have
nonempty interior; the two inequalities in the $m=2$ example are
nonredundant.
\end{corollary}

\begin{proof}
The upper bound is Theorem~\ref{thm:few-upper}.  The matching examples
are given by Proposition~\ref{prop:quartic-lower},
\eqref{eq:quartic-two-degree-equality}, and
Lemma~\ref{lem:quartic-lift}, with the normalizations stated there.
\end{proof}

\section{Consequences for tori and SOS optimization}\label{sec:consequences}
Sections~\ref{sec:upper} and~\ref{sec:lower} established matching
degree bounds for polynomial quadratic-module certificates.
We now derive consequences for trigonometric certificates and
SOS optimization. Transferring the lower bound to a torus requires
accounting for arbitrary cosine-dependent SOS factors, while
passing from certificate degree to moment-relaxation error requires
a duality argument. We also show why the constants must depend on
the chosen generators, not just the feasible set.

\paragraph{Proof outline.}
We first use sine--cosine lifting and the circle equality ideal to
transfer the polynomial upper bound to trigonometric certificates.
We next obtain a matching torus lower bound by averaging arbitrary
trigonometric SOS factors over cosine parities, preserving the degree
and norm estimates used in the polynomial lower-bound proof.
For one or two inequalities, the direct conversion from
Section~\ref{sec:few-generators} removes one logarithm; new quartic
boundary examples prove sharpness. With no additional constraints,
cubic SOS completion and a localized Fourier kernel give the uniform
fourth-root upper bound. A direct quartic Taylor argument proves a
matching lower bound in four variables.
An explicit equality-only construction separates this result from
positivity on further equality-constrained subsets of a torus.
Finally, we invert the general degree--margin bounds to determine the
relaxation error, establish equality of SOS and moment values by
strict feasibility, and exhibit presentation-dependent constants.
\subsection{Upper bounds for trigonometric SOS certificates}
Theorem~\ref{thm:main} concerns polynomial presentations, so it
cannot be applied directly to trigonometric constraints.
Sine--cosine lifting turns their Fourier characters into polynomials
on a product of circles. The circle equalities supply the required
Archimedean witness, and substitution back will preserve the
original trigonometric generators.
Throughout this section, trigonometric objectives, constraint generators,
SOS factors, and equality multipliers are real-valued.
\paragraph{Fourier Wiener norm and degree.}
For $\T^q=(\R/2\pi\Z)^q$, a trigonometric polynomial has the form
\begin{equation}\label{eq:fourier-series}
 f(\boldsymbol\theta)=\sum_{\mathbf k\in\Z^q}
 \widehat f(\mathbf k)e^{\mathrm i\mathbf k\cdot\boldsymbol\theta}.
\end{equation}
Its Fourier Wiener norm and coordinate Fourier degree are
\begin{equation}\label{eq:wiener-norm}
 \|f\|_{\mathrm W}=\sum_{\mathbf k\in\Z^q}|\widehat f(\mathbf k)|,
 \qquad
 \degF f=\max_{\widehat f(\mathbf k)\ne0}\|\mathbf k\|_\infty,
\end{equation}
where $\|\mathbf k\|_\infty=\max_j|k_j|$ and $\degF0=0$.
The Wiener norm extends to absolutely summable Fourier series,
which converge uniformly on the torus. It is submultiplicative,
and for $\degF f\le D$, Parseval's identity and Cauchy--Schwarz give
\begin{equation}\label{eq:wiener-bounds}
 \|fg\|_{\mathrm W}\le\|f\|_{\mathrm W}\|g\|_{\mathrm W},
 \qquad
 \|f\|_{\T^q,\infty}\le\|f\|_{\mathrm W}
 \le(2D+1)^{q/2}\|f\|_{\T^q,\infty}.
\end{equation}
Unlike the radius-one coefficient Wiener norm, the Fourier Wiener norm
sums absolute Fourier coefficients rather than monomial coefficients.

In the torus setting, write
$\boldsymbol T_i=(T_{i,1},\ldots,T_{i,N_i})$ for a finite
vector of real trigonometric-polynomial SOS factors. The certificate
budget $d$ requires $2\degF T_{0,j}\le d$,
$\degF g_i+2\degF T_{i,j}\le d$, and
$\degF h_k+\degF\lambda_k\le d$. For a certificate
with SOS-factor vectors $\boldsymbol T_i$ ($i=0$ denoting the unweighted
SOS) and equality multipliers $\lambda_k$, define its size in the
Fourier Wiener norm by
\[
 \size_{\mathrm W}=1+\sum_{i=0}^m\sum_{j=1}^{N_i}
 \|T_{i,j}\|_{\mathrm W}^2+\sum_{k=1}^\ell\|\lambda_k\|_{\mathrm W}.
\]
For preorderings, a term $g^\alpha\sum_j T_{\alpha,j}^2$
obeys the degree budget
$\degF(g^\alpha)+2\degF T_{\alpha,j}\le d$.
The Wiener certificate size is defined analogously, including all
SOS factors $T_{\alpha,j}$ and equality multipliers.

\begin{corollary}\label{cor:torus}
Fix a trigonometric presentation with nonempty feasible set and a Fourier-degree bound $D$. If $\degF f\le D$, $\|f\|_{\T^q,\infty}\le1$, and $f|_K\ge\delta$ with $0<\delta\le1$, then $f$ has an exact certificate in the original module with
\[
 d\le C\delta^{-1/2}\log^2(2/\delta).
\]
It may be chosen with $\log\size_{\mathrm W}$ bounded by the same order. For a fixed preordering, replace $\log^2$ by $\log$.
\end{corollary}
\begin{proof}
We first lift the trigonometric presentation into polynomial
coordinates and establish the Archimedean hypothesis.
Introduce polynomial coordinates $c_j,s_j$ and circle equations $c_j^2+s_j^2-1=0$. The equality ideal contains the Archimedean witness
\[
 q-\sum_j(c_j^2+s_j^2)=-\sum_j(c_j^2+s_j^2-1).
\]
We next control the degree and norm of the lifted objective.
For $f(\boldsymbol\theta)=\sum_{\|\mathbf k\|_\infty\le D}
\widehat f(\mathbf k)e^{i\mathbf k\cdot\boldsymbol\theta}$,
set $k_j^+=\max\{k_j,0\}$ and $k_j^-=\max\{-k_j,0\}$, and define
\[
 F(\mathbf c,\mathbf s)=\sum_{\|\mathbf k\|_\infty\le D}
 \widehat f(\mathbf k)\prod_{j=1}^q
 (c_j+is_j)^{k_j^+}(c_j-is_j)^{k_j^-}.
\]
Conjugate symmetry makes $F$ a real polynomial. Moreover,
$F(\cos\boldsymbol\theta,\sin\boldsymbol\theta)=f(\boldsymbol\theta)$
and $\deg F\le qD$.
Normalized Parseval and Cauchy--Schwarz give
$\|f\|_{\mathrm W}\le(2D+1)^{q/2}\|f\|_{\T^q,\infty}$.
On $Q=[-4\sqrt q,4\sqrt q]^{2q}$,
$|c_j\pm is_j|\le4\sqrt{2q}$; hence
\[
 \|F\|_{Q,\infty}\le C_{q,D}\|f\|_{\T^q,\infty},
 \qquad C_{q,D}=(2D+1)^{q/2}(4\sqrt{2q})^{qD}.
\]

Finally, apply Theorem~\ref{thm:main}, or
Corollary~\ref{cor:preordering}, to $F/C_{q,D}$.
Its norm on $Q$ is at most one and its margin is at least
$\delta/C_{q,D}$. Multiplying the resulting certificate by
$C_{q,D}$ does not change its degree. Then substitute $c_j=\cos\theta_j$, $s_j=\sin\theta_j$. Circle-ideal terms vanish. Polynomial total degree bounds the resulting coordinate Fourier degree. For the radius normalization of Section~\ref{sec:analytic},
$\mathbf z=(\mathbf c,\mathbf s)/\sqrt q$, each normalized coordinate
has Wiener norm $1/\sqrt q\le1$. Thus for
$P(\mathbf z)=\sum_\alpha p_\alpha\mathbf z^\alpha$,
$\|P(\mathbf z(\boldsymbol\theta))\|_{\mathrm W}
\le\sum_\alpha|p_\alpha|\le\|P\|_2$.
The certificate-size bound in Lemma~\ref{lem:exactification}
therefore controls the Wiener certificate size after substitution.
Fixed objective and generator rescalings affect only constants. No inequality generators are added.
\end{proof}

\subsection{A matching lower bound for a three-generator torus example}
\label{subsec:torus-three}
Corollary~\ref{cor:torus} gives the torus upper bound, but sharpness
does not follow by simply substituting sine variables into the
polynomial example:
an arbitrary trigonometric SOS certificate may contain cosine-dependent
SOS factors. We use cosine-parity averaging to convert these factors
into weighted polynomial squares while retaining the degree and norm
controls needed for the lower-bound argument.
On $\T^3$, let $u=\sin\theta_0$, $x=\sin\theta_1$, $y=\sin\theta_2$ and take
\[
 g_1=-ux,\qquad g_2=-uy,\qquad g_3=\tfrac12u(xy+u^4).
\]
For $f_\varepsilon=(u+\varepsilon)/(1+\varepsilon)$, the
\emph{minimum certificate degree (measured using coordinate Fourier degree)} in this fixed torus module is
$\Theta(\varepsilon^{-1/2}\log^2(1/\varepsilon))$.
The upper bound follows from Corollary~\ref{cor:torus}.
For the lower bound, consider an arbitrary trigonometric SOS
certificate. We multiply it by $1+\varepsilon$ to work with
$u+\varepsilon$ and absorb the factor $1/2$ in $g_3$
into its SOS multiplier. These rescalings preserve
the degree budget and change the SOS-factor norms
by uniformly bounded factors.

The cosine-parity decomposition used next is a standard
trigonometric-to-polynomial device; compare the treatment of trigonometric
SOS and associated cube weights in Bach and Rudi~\cite{BachRudi2023}.
For constrained certificates, the decomposition preserves the norm
bounds uniformly over the number of SOS factors, as shown below.

We first eliminate mixed cosine-parity terms by averaging an arbitrary
trigonometric certificate under the independent reflections $\theta_j\mapsto\pi-\theta_j$. An SOS factor decomposes into its eight cosine-parity parts, so its averaged square is a sum of terms
\begin{equation}\label{eq:torus-parity}
 (1-u^2)^{\beta_0}(1-x^2)^{\beta_1}(1-y^2)^{\beta_2}
 \sq{\boldsymbol P_{i,\beta}(u,x,y)},\qquad \beta\in\{0,1\}^3,
\end{equation}
where each $\boldsymbol P_{i,\beta}
=(P_{i,\beta,1},\ldots,P_{i,\beta,N_{i,\beta}})$ is a finite
vector of ordinary polynomials obtained from the trigonometric
factors $\boldsymbol T_i$, with coordinate degree at most the original
Fourier budget. To see this, first express each trigonometric polynomial in the
polynomial basis $\sin\theta_j,\cos\theta_j$ and write it uniquely as a sum
of products $\prod_j(\cos\theta_j)^{\beta_j}$ times polynomials in the
sines, using $\cos^2\theta_j=1-\sin^2\theta_j$.  Reflection changes
only the cosine signs.  Averaging the square therefore cancels its
mixed-parity terms and yields \eqref{eq:torus-parity}. Each averaged multiplier itself has coordinate degree at most the budget.
We next check that this averaging preserves the certificate-degree
budget, rather than merely yielding a pointwise nonnegative
representation. More precisely, a Fourier character $e^{ik\theta}$ is represented by
$(\cos\theta+i\sin\theta)^k$ for $k\ge0$ and its conjugate for $k<0$.
Reducing even powers of $\cos\theta$ does not increase its coordinate
degree: in the cosine-odd sector the polynomial in $\sin\theta$ has
degree at most $|k|-1$, and in the even sector at most $|k|$.
Consequently, if an SOS factor has coordinate Fourier degree at most
$n$, each component of $\boldsymbol P_{i,\beta}$ has degree at most
$n-\beta_j$ in $z_j$, where $(z_0,z_1,z_2)=(u,x,y)$. Thus each weighted square in
\eqref{eq:torus-parity} has coordinate polynomial degree at most $2n$,
consistent with the original certificate budget. Although the parity
polynomials may have larger total degree, their coordinate degrees are
bounded as above; this suffices because the Markov and Chebyshev
estimates in Section~\ref{sec:lower} are applied coordinatewise.
The averaging only introduces a fixed number, eight, of parity sectors,
irrespective of the original number of SOS factors.

We then justify the forbidden-side substitution required for the
lower bound. After averaging, both sides of the certificate are polynomials in
$(u,x,y)$. Since the sine map from $\T^3$ onto $[-1,1]^3$ is
surjective, equality on the torus implies a polynomial identity.
Thus the subsequent forbidden-side substitution is valid.

Finally, we transfer the factor bounds of Section~\ref{sec:lower}
to the parity-weighted representation and recover the analytic
identity required by Lemma~\ref{lem:corner}.
Write $W_\beta=(1-u^2)^{\beta_0}(1-x^2)^{\beta_1}
(1-y^2)^{\beta_2}$, $S_i=\sum_\beta W_\beta\sq{\boldsymbol P_{i,\beta}}$,
and $\boldsymbol P_i=\bigoplus_\beta\boldsymbol P_{i,\beta}$.
On the feasible box of Section~\ref{sec:lower},
$W_\beta\ge a_*=(3/4)^3$; thus
$\|\boldsymbol P_i\|_{\ell^2}^2\le a_*^{-1}S_i$.
The averaged certificate is exact and has nonnegative terms there. Explicitly, the averaged identity is
\[
u+\varepsilon
=S_0-uxS_1-uyS_2+u(xy+u^4)S_3.
\]

Apply Markov's inequality first to the polynomials $xS_1$ and $yS_2$,
whose coordinate degrees are at most $d+1$, to remove the vanishing
weights. The comparison with $S_i$ then bounds the polynomial vectors
$\boldsymbol P_i$ on the feasible box. Applying
\eqref{eq:chebyshev-extrapolation} coordinatewise gives
\eqref{eq:factor-extrapolation} for these vectors, with constants
independent of the number of SOS factors. Only after rescaling do
we absorb the parity weights into analytic square roots, multiplying
the $\beta$th vector by
\[
 (1-t^2)^{\beta_0/2}(1-\rho^2X^2)^{\beta_1/2}
 (1-\rho^2Y^2)^{\beta_2/2}.
\]
Their analytic radius-$R$ norms are uniformly bounded, since $\rho R=1/(16d^2)$. Stack the resulting vectors. The four groups, scaled respectively by $\rho$, $t^{-1/2}$, $\sqrt\rho$, and $\sqrt\rho$, satisfy \eqref{eq:corner} and \eqref{eq:K-upper}. Lemma~\ref{lem:corner} gives \eqref{eq:degree-comparison}. This proves the torus lower bound for arbitrary real trigonometric SOS factors.

\subsection{Sharp single-logarithm bounds with one or two torus inequalities}
\label{subsec:torfew-sharp}

The improved upper bound for at most two inequality generators also
transfers to the torus. The equality constraints introduced by the
sine--cosine lift do not change the number of inequality generators.

\begin{corollary}[Torus upper bound with at most two inequalities]
\label{cor:torfew-upper}
Fix a real trigonometric presentation on $\T^q$ with at most two
inequality generators, any fixed trigonometric equalities, and a
nonempty feasible set $K$. For every fixed coordinate Fourier-degree
bound $D$, there is a constant $C$ such that
\[
 \degF f\le D,\qquad
 \|f\|_{\T^q,\infty}\le1,\qquad
 f|_K\ge\delta>0
\]
imply an exact certificate in the original trigonometric module of
budget
\begin{equation}\label{eq:torfew-upper}
 d\le C\delta^{-1/2}\log(2/\delta).
\end{equation}
The constant is uniform over the displayed objective class.
\end{corollary}

\begin{proof}
Use the sine--cosine lift in the proof of Corollary~\ref{cor:torus}.
The circle equations provide a fixed Archimedean witness. The lifted
objective has total degree at most $qD$ and bounded sup norm on the
fixed normalization box, with bounds depending only on $q,D$.
The lifted presentation still has at most two inequality generators.
Apply Theorem~\ref{thm:few-upper}, then substitute the sine and cosine
coordinates back. Equality-ideal terms from the circle equations vanish,
and polynomial total degree bounds coordinate Fourier degree. All
normalization factors are fixed, so the claimed bound follows.
\end{proof}

We now give fixed torus examples attaining this order. The generator
has quartic tangential and normal factors; the proof controls arbitrary
cosine-dependent SOS factors by cosine-parity averaging.

On $\T^2$, write
\begin{equation}\label{eq:torfew-example}
 x=\sin\theta_1,\qquad y=\sin\theta_2,\qquad
 u=\tfrac14-x^2-y^2,\qquad
 g=u(x^4+u^4).
\end{equation}
Let $\mathcal M_{\mathrm{tor},1}$ be the trigonometric quadratic
module generated by $g$, with no additional equality constraints.
Its feasible set is
\[
 K_{\mathrm{tor},1}
 =\{\boldsymbol\theta:\sin^2\theta_1+\sin^2\theta_2\le\tfrac14\}.
\]
Indeed, $x^4+u^4\ge0$, and its zero set is contained in $\{u=0\}$.
The objective $u$ has coordinate Fourier degree two and $g$ has
coordinate Fourier degree ten.

For $0<\varepsilon\le1/32$, put
\begin{equation}\label{eq:torfew-normalization}
 f_\varepsilon
 =\frac{u+\varepsilon}{7/4-\varepsilon},\qquad
 \delta_\varepsilon
 =\frac{\varepsilon}{7/4-\varepsilon}.
\end{equation}
Since $u\in[-7/4,1/4]$ on the full torus and
$\min_{K_{\mathrm{tor},1}}u=0$, these polynomials satisfy
\[
 \|f_\varepsilon\|_{\T^2,\infty}=1,\qquad
 \min_{K_{\mathrm{tor},1}}f_\varepsilon=\delta_\varepsilon.
\]

\begin{theorem}[Sharp one-inequality torus example]
\label{thm:torfew-one-sharp}
With certificate degree measured using coordinate Fourier degree,
\begin{equation}\label{eq:torfew-one-sharp}
 d_{\mathcal M_{\mathrm{tor},1}}(f_\varepsilon)
 =\Theta\!\left(\delta_\varepsilon^{-1/2}
                    \log(1/\delta_\varepsilon)\right)
 \qquad(\varepsilon\downarrow0).
\end{equation}
More explicitly, every trigonometric certificate for $u+\varepsilon$
with $0<\varepsilon\le e^{-12}$ has degree budget at least
\begin{equation}\label{eq:torfew-explicit-lower}
 \frac1{60}\varepsilon^{-1/2}\log(1/\varepsilon).
\end{equation}
\end{theorem}

\begin{proof}
Consider an arbitrary real trigonometric SOS certificate of budget $d$,
\begin{equation}\label{eq:torfew-trig-certificate}
 u+\varepsilon=\sq{\boldsymbol T_0}
                  +g\sq{\boldsymbol T_1}.
\end{equation}
Positive rescaling of a certificate does not change its degree, so it
is enough to treat $u+\varepsilon$ rather than $f_\varepsilon$.

\emph{Cosine-parity reduction.}
Both $u$ and $g$ are invariant under the independent reflections
$\theta_j\mapsto\pi-\theta_j$.
Average \eqref{eq:torfew-trig-certificate} under these four
reflections. By the cosine-parity decomposition described above,
the result is
\begin{equation}\label{eq:torfew-polynomial-certificate}
 u+\varepsilon=S_0+gS_1,\qquad
 S_i=\sum_{\beta\in\{0,1\}^2}
       W_\beta\sq{\boldsymbol P_{i,\beta}},
 \qquad
 W_\beta=(1-x^2)^{\beta_1}(1-y^2)^{\beta_2}.
\end{equation}
Each $\boldsymbol P_{i,\beta}$ is a finite vector of real
polynomials. Every component has degree at most $d$ in each variable
and hence total degree at most $2d$. These deliberately loose bounds
follow from the original Fourier budget: a factor of coordinate
Fourier degree $n$ produces parity polynomials of degree at most
$n-\beta_j$ in variable $j$.
Since the sine map covers $[-1,1]^2$, the averaged equality is a
polynomial identity. In particular, all evaluations below are valid.
Write
\[
 \boldsymbol P_1=\bigoplus_{\beta\in\{0,1\}^2}
                          \boldsymbol P_{1,\beta}.
\]

\emph{Bounds on interior circles.}
Fix $s\in[\varepsilon,1/8]$ and put
$r_s=\sqrt{1/4-s}$. On the circle $x^2+y^2=r_s^2$, we have
$u=s$, $r_s^2\in[1/8,1/4]$, and
\[
 W_\beta\ge(3/4)^2=9/16.
\]
All terms in \eqref{eq:torfew-polynomial-certificate} are
nonnegative there. Thus
\[
 s(x^4+s^4)S_1\le s+\varepsilon\le2s,
 \qquad
 x^4S_1\le2.
\]
Since $S_1\ge(9/16)\|\boldsymbol P_1\|_{\ell^2}^2$, it follows
that
\begin{equation}\label{eq:torfew-weighted-circle}
 \|x^2\boldsymbol P_1(x,y)\|_{\ell^2}
 \le\frac{4\sqrt2}{3}\le2
 \qquad(x^2+y^2=r_s^2).
\end{equation}

Parametrize this circle and define
\[
 \boldsymbol V_s(\varphi)
 =r_s^2\sin^2\varphi\,
   \boldsymbol P_1(r_s\sin\varphi,r_s\cos\varphi).
\]
Its components are trigonometric polynomials of degree at most
$2d+2$, because the parity factors have total degree at most $2d$.
Twice applying Bernstein's inequality gives
\[
 \|\boldsymbol V_s''(0)\|_{\ell^2}
 \le(2d+2)^2\sup_{\varphi\in\R}
          \|\boldsymbol V_s(\varphi)\|_{\ell^2}
 \le2(2d+2)^2.
\]
The vector form follows by taking inner products with arbitrary
Euclidean unit vectors before applying the scalar inequalities.
Since the factor $\sin^2\varphi$ has a double zero at zero,
\[
 \boldsymbol V_s''(0)=2r_s^2\boldsymbol P_1(0,r_s).
\]
Using $r_s^2\ge1/8$ therefore yields
\begin{equation}\label{eq:torfew-axis-bound}
 \|\boldsymbol P_1(0,r_s)\|_{\ell^2}
 \le32(d+1)^2.
\end{equation}

\emph{Extrapolation across the boundary.}
The polynomial vector $y\mapsto\boldsymbol P_1(0,y)$ has degree
at most $d$ and obeys \eqref{eq:torfew-axis-bound} on
\[
 I_\varepsilon
 =\left[\frac1{2\sqrt2},\sqrt{\tfrac14-\varepsilon}\right]
 =:[a,b_\varepsilon].
\]
Let $y_\varepsilon=\sqrt{1/4+2\varepsilon}$.
For $0<\varepsilon\le1/32$,
\[
 b_\varepsilon-a\ge\frac1{10},\qquad
 y_\varepsilon-b_\varepsilon
 =\frac{3\varepsilon}
  {\sqrt{1/4+2\varepsilon}+\sqrt{1/4-\varepsilon}}
 \le4\varepsilon.
\]
Hence the argument in the Chebyshev extrapolation inequality is at
most $1+80\varepsilon$. With
$\operatorname{arcosh}(1+v)\le\sqrt{2v}$, we conclude that
\begin{equation}\label{eq:torfew-chebyshev-bound}
 \begin{aligned}
 \|\boldsymbol P_1(0,y_\varepsilon)\|_{\ell^2}
 &\le32(d+1)^2 T_d(1+80\varepsilon)\\
 &\le32(d+1)^2\exp(13d\sqrt\varepsilon).
 \end{aligned}
\end{equation}

At $(x,y)=(0,y_\varepsilon)$, however,
\[
 u=-2\varepsilon,\qquad g=-32\varepsilon^5,
 \qquad 0\le W_\beta\le1.
\]
The last assertion holds because $y_\varepsilon^2\le5/16<1$.
Equation~\eqref{eq:torfew-polynomial-certificate} gives
\[
 32\varepsilon^5 S_1(0,y_\varepsilon)
 =\varepsilon+S_0(0,y_\varepsilon)\ge\varepsilon.
\]
Since $S_1\le\|\boldsymbol P_1\|_{\ell^2}^2$ at this point,
\begin{equation}\label{eq:torfew-factor-lower}
 \|\boldsymbol P_1(0,y_\varepsilon)\|_{\ell^2}
 \ge\frac1{\sqrt{32}\,\varepsilon^2}.
\end{equation}

Combining \eqref{eq:torfew-chebyshev-bound} and
\eqref{eq:torfew-factor-lower}, and putting
$L=\log(1/\varepsilon)$ and $v=(d+1)\sqrt\varepsilon$, gives
\[
 \varepsilon^{-1}
 \le128\sqrt2\,v^2e^{13v},\qquad
 L\le\log(128\sqrt2)+2\log v+13v
 \le\log(128\sqrt2)+15v.
\]
For $0<\varepsilon\le e^{-12}$, we have
$L\ge12>2\log(128\sqrt2)$ and
$L/(30\sqrt\varepsilon)\ge2$.
Consequently $v\ge L/30$ and
\[
 d\ge\frac{L}{60\sqrt\varepsilon},
\]
proving \eqref{eq:torfew-explicit-lower}.
Corollary~\ref{cor:torfew-upper} supplies the matching upper bound.
Finally, $\delta_\varepsilon\asymp\varepsilon$ and positive
rescaling preserves the degree, yielding
\eqref{eq:torfew-one-sharp}.
\end{proof}

\paragraph{Two nonredundant inequalities.}
On $\T^3$, retain $g(\theta_1,\theta_2)$ and adjoin
$g_2(\theta_3)=\sin\theta_3$. Let $\mathcal M_{\mathrm{tor},2}$
be the resulting module, again with no additional equalities.
Its feasible set is
\[
 K_{\mathrm{tor},1}\times
 \{\theta_3\in\T:\sin\theta_3\ge0\},
\]
and neither inequality is redundant. Every certificate over
$\mathcal M_{\mathrm{tor},1}$ remains a certificate over
$\mathcal M_{\mathrm{tor},2}$. Conversely, specializing
$\theta_3=0$ in a certificate for a function independent of
$\theta_3$ removes the second weighted SOS term and does not
increase any coordinate Fourier degree. Thus
\begin{equation}\label{eq:torfew-two-equality}
 d_{\mathcal M_{\mathrm{tor},2}}(f_\varepsilon)
 =d_{\mathcal M_{\mathrm{tor},1}}(f_\varepsilon)
 =\Theta\!\left(\delta_\varepsilon^{-1/2}
                   \log(1/\delta_\varepsilon)\right).
\end{equation}
Together with Corollary~\ref{cor:torfew-upper}, this proves
worst-case sharpness for both one and two torus inequalities.

\begin{remark}[A family of higher-order degeneracies]
\label{rem:torfew-family}
The construction extends to
\[
 g_{k,\ell}=u(x^{2k}+u^{2\ell}),\qquad
 u=\tfrac14-x^2-y^2,
 \qquad k,\ell\in\N_{\ge1},\quad \ell>k/2.
\]
The feasible set is unchanged. On the same interior circles,
positivity bounds $\|x^k\boldsymbol P_1\|$ by a constant.
Applying the $k$th trigonometric Bernstein inequality to
\[
 r_s^k\sin^k\varphi\,
 \boldsymbol P_1(r_s\sin\varphi,r_s\cos\varphi)
\]
gives $\|\boldsymbol P_1(0,r_s)\|\le C_k(d+1)^k$,
because its $k$th derivative at zero is
$k!r_s^k\boldsymbol P_1(0,r_s)$.
At the same exterior point, $g_{k,\ell}
=-2^{2\ell+1}\varepsilon^{2\ell+1}$, which forces a factor
norm of order at least $\varepsilon^{-\ell}$.
Chebyshev extrapolation therefore gives
\[
 \varepsilon^{-\ell}
 \le C_{k,\ell}(d+1)^k\exp(C_{k,\ell}d\sqrt\varepsilon).
\]
After setting $v=(d+1)\sqrt\varepsilon$, the positive gap
$\ell-k/2$ yields
$d=\Omega(\varepsilon^{-1/2}\log(1/\varepsilon))$.
The one-generator upper bound proves the matching order throughout
this family. The explicit example above is $k=\ell=2$.
\end{remark}

\subsection{No additional constraints: the fourth-root law on the full torus}
\label{subsec:full-torus-zero}

Here the feasible set is the entire torus: there are no inequality
generators and no additional trigonometric equality constraints.
We construct local cubic models, complete them to sums of squares,
and integrate them against a polynomial square kernel whose fourth
moment is $O(N^{-4})$. A fixed-degree Fourier correction makes the
resulting certificate exact. A local quartic obstruction then proves
the matching lower bound.

\subsubsection{Cubic jets admit a uniform quartic SOS completion}
\label{subsubsec:cubic-jet-completion}

The improvement on the full torus rests on a finite-dimensional
moment-completion lemma.  The fourth moments need only be bounded;
the first three moments can be matched exactly.  All constants in
the following two lemmas depend only on the dimension.

\begin{lemma}[Completion of moments through degree three]
\label{lem:cubic-moment-completion}
For every integer $n\ge1$ there is a constant $C_n\le6000n^6$ with
the following property.  Let $L:\R[\mathbf v]_{\le4}\to\R$ be linear,
with $L(1)=1$ and
\[
 L(q^2)\ge0\qquad(q\in\R[\mathbf v]_{\le2}).
\]
There is a finitely supported probability measure $\mu$ on $\R^n$
such that
\begin{equation}\label{eq:cubic-moment-matching}
 \int p\,d\mu=L(p)\qquad(\deg p\le3),
 \qquad
 \int\|\mathbf v\|^4\,d\mu
 \le C_nL(\|\mathbf v\|^4).
\end{equation}
\end{lemma}

\begin{proof}
We first treat a centered functional, $L(v_i)=0$, after an orthogonal
change of coordinates diagonalizes its covariance:
\[
 L(v_iv_j)=\lambda_i\mathbf1_{i=j},\qquad\lambda_i\ge0.
\]
Write $A=\operatorname{diag}(\lambda_1,\ldots,\lambda_n)$,
$B_{ijk}=L(v_iv_jv_k)$, and $M_4=L(\|\mathbf v\|^4)$.
Cauchy--Schwarz for the positive semidefinite form
$(a,b)\mapsto L(ab)$ on polynomials of degree at most two gives
\begin{equation}\label{eq:cubic-third-moment-cs}
 B_{ijk}^2\le\lambda_iL(v_j^2v_k^2).
\end{equation}
In particular, $B_{ijk}=0$ if any of
$\lambda_i,\lambda_j,\lambda_k$ is zero.  If
$m_{ijk}=\min(\lambda_i,\lambda_j,\lambda_k)>0$, symmetry and
\eqref{eq:cubic-third-moment-cs} show
\begin{equation}\label{eq:cubic-third-moment-ratio}
 \frac{B_{ijk}^2}{m_{ijk}}\le M_4.
\end{equation}
Here every individual $L(v_a^2v_b^2)$ is nonnegative and bounded
by $M_4=\sum_{a,b}L(v_a^2v_b^2)$.
We also have
\begin{equation}\label{eq:cubic-covariance-trace}
 (\operatorname{tr}A)^2=L(\|\mathbf v\|^2)^2\le M_4.
\end{equation}

Put $N=8n^3$.  For each ordered triple $(i,j,k)$ with
$m_{ijk}>0$ and each $\boldsymbol\epsilon\in\{-1,1\}^3$, define
\[
 \mathbf w=\epsilon_1\mathbf e_i+\epsilon_2\mathbf e_j
                         +\epsilon_3\mathbf e_k,
 \qquad
 b=\frac{B_{ijk}\epsilon_1\epsilon_2\epsilon_3}{48},
 \qquad
 a=\frac{m_{ijk}}{18N}.
\]
Keep the terms indexed separately even if their vectors coincide.
The cubic polarization identity is
\[
 \operatorname{sym}(\mathbf e_i\otimes\mathbf e_j\otimes\mathbf e_k)
 =\frac1{48}\sum_{\boldsymbol\epsilon\in\{-1,1\}^3}
 \epsilon_1\epsilon_2\epsilon_3
 (\epsilon_1\mathbf e_i+\epsilon_2\mathbf e_j
                         +\epsilon_3\mathbf e_k)^{\otimes3}.
\]
Consequently, the corresponding finite collection, denoted by
$\{(\mathbf w_\alpha,a_\alpha,b_\alpha)\}_\alpha$, satisfies
\begin{equation}\label{eq:cubic-tensor-decomposition}
 B=\sum_\alpha b_\alpha\mathbf w_\alpha^{\otimes3}.
\end{equation}

For every $\mathbf z\in\R^n$, weighted Cauchy--Schwarz gives
\[
 a_\alpha(\mathbf w_\alpha\cdot\mathbf z)^2
 \le a_\alpha\left(\sum_{\ell:\lambda_\ell>0}
                 \frac{(w_\alpha)_\ell^2}{\lambda_\ell}\right)
                 \mathbf z^{\mathsf T}A\mathbf z
 \le\frac1{2N}\mathbf z^{\mathsf T}A\mathbf z.
\]
The last inequality uses $\|\mathbf w_\alpha\|\le3$ and
$\lambda_\ell\ge m_{ijk}$ on the support of $\mathbf w_\alpha$.
There are at most $N$ terms, so
\begin{equation}\label{eq:cubic-covariance-budget}
 \sum_\alpha a_\alpha\mathbf w_\alpha\mathbf w_\alpha^{\mathsf T}
 \preceq\tfrac12 A.
\end{equation}

For each $a>0$ and $b\in\R$, a two-point random variable $Y$ can
be chosen to satisfy
\begin{equation}\label{eq:cubic-two-point-moments}
 \mathbb E Y=0,\quad \mathbb E Y^2=a,\quad
 \mathbb E Y^3=b,\quad
 \mathbb E Y^4=a^2+\frac{b^2}{a}.
\end{equation}
Indeed, let $r_-<0<r_+$ be the roots of
$r^2-(b/a)r-a=0$, and put masses
$r_+/(r_+-r_-)$ and $-r_-/(r_+-r_-)$ at $r_-$ and $r_+$,
respectively.  Their mean is zero, and the remaining identities follow
by multiplying $Y^2=(b/a)Y+a$ by successive powers of $Y$.

Choose independent variables $Y_\alpha$ with the moments
\eqref{eq:cubic-two-point-moments}.  Let $Z$ be an independent,
centered, finitely supported random vector with covariance
\[
 A-\sum_\alpha a_\alpha\mathbf w_\alpha\mathbf w_\alpha^{\mathsf T}
\]
and zero third moments.  Such $Z$ is obtained by diagonalizing the
displayed positive semidefinite matrix and taking independent symmetric
two-point variables along its eigenvectors.  The random vector
\[
 X=Z+\sum_\alpha Y_\alpha\mathbf w_\alpha
\]
has mean zero, covariance $A$, and third moment tensor $B$, by
independence and \eqref{eq:cubic-tensor-decomposition}.
Thus its law matches $L$ on polynomials of degree at most three.

For independent centered random vectors, expansion of the fourth
power of the Euclidean norm bounds the mixed terms by six times
the products of their second radial moments.  Applying this expansion
to the independent scalar components of $Z$ and the vectors
$Y_\alpha\mathbf w_\alpha$, and using
\eqref{eq:cubic-two-point-moments}, gives
\[
 \mathbb E\|X\|^4
 \le3(\operatorname{tr}A)^2
       +\sum_\alpha\|\mathbf w_\alpha\|^4
                           \frac{b_\alpha^2}{a_\alpha}.
\]
For a term indexed by $(i,j,k,\boldsymbol\epsilon)$,
\[
 \|\mathbf w_\alpha\|^4\frac{b_\alpha^2}{a_\alpha}
 \le\frac{81\cdot18N}{48^2}
                         \frac{B_{ijk}^2}{m_{ijk}}.
\]
Summing over the eight signs and the $n^3$ ordered triples, then using
\eqref{eq:cubic-third-moment-ratio} and
\eqref{eq:cubic-covariance-trace}, proves the convenient bound
\begin{equation}\label{eq:cubic-centered-fourth-bound}
 \mathbb E\|X\|^4\le\left(4+\frac{81}{2}n^6\right)M_4.
\end{equation}

For a general normalized functional, put
$\mathbf m=(L(v_1),\ldots,L(v_n))$ and center by
$\mathbf v\mapsto\mathbf v-\mathbf m$ before the preceding construction.
We have
\[
 \|\mathbf m\|^4\le L(\|\mathbf v\|^4),\qquad
 L(\|\mathbf v-\mathbf m\|^4)
     \le16L(\|\mathbf v\|^4).
\]
For completeness, the second inequality follows because
\[
 8(\|\mathbf v\|^4+\|\mathbf m\|^4)
       -\|\mathbf v-\mathbf m\|^4
\]
is SOS.  To see this, put $A_0=\|\mathbf v\|^2$,
$B_0=\|\mathbf m\|^2$, and $C_0=\|\mathbf v-\mathbf m\|^2$;
then the polynomial equals
\[
 4(A_0-B_0)^2+
 [2(A_0+B_0)-C_0][2(A_0+B_0)+C_0],
\]
and both factors in the product are SOS, the first being
$\|\mathbf v+\mathbf m\|^2$.
The first inequality follows from Cauchy--Schwarz and
$L(\|\mathbf v-\mathbf m\|^2)\ge0$.
Translate the constructed centered measure back by $\mathbf m$ and
use the pointwise inequality
$\|X+\mathbf m\|^4\le8(\|X\|^4+\|\mathbf m\|^4)$.
Together with \eqref{eq:cubic-centered-fourth-bound}, this gives
\[
 \mathbb E\|X+\mathbf m\|^4
 \le\bigl(520+5184n^6\bigr)L(\|\mathbf v\|^4)
 \le6000n^6L(\|\mathbf v\|^4).
\]
Orthogonal changes of coordinates preserve radial moments.  The
result follows, with a finitely supported measure throughout.
\end{proof}

\begin{lemma}[Uniform SOS completion of a cubic jet]
\label{lem:cubic-sos-completion}
Let $n\ge1$, $M\ge0$, and $p\in\R[\mathbf v]_{\le3}$ satisfy
\[
 p(\mathbf v)\ge-M\|\mathbf v\|^4\qquad(\mathbf v\in\R^n).
\]
Then
\begin{equation}\label{eq:cubic-sos-completion}
 p(\mathbf v)+C_nM\|\mathbf v\|^4\in\Sigma[\mathbf v]_{\le4},
 \qquad C_n=6000n^6,
\end{equation}
where the right-hand side denotes sums of squares of polynomials of
degree at most two.
\end{lemma}

\begin{proof}
The cone of SOS polynomials of degree at most four is closed in the
finite-dimensional coefficient space.  One direct verification uses
a basis of polynomials of degree at most two orthonormal for a fixed
Gaussian probability measure: a positive semidefinite Gram matrix
for an SOS polynomial has trace equal to the integral of that
polynomial.  A coefficient-convergent sequence therefore has bounded
Gram matrices, and a convergent subsequence gives a Gram matrix for
the limit.

By finite-dimensional separation, it is enough to show that every
linear functional $L$ nonnegative on squares of degree at most two
is nonnegative on the polynomial in
\eqref{eq:cubic-sos-completion}.
If $L(1)>0$, normalize it to $L(1)=1$ and use
Lemma~\ref{lem:cubic-moment-completion} to obtain $\mu$.  Then
\[
 L(p)=\int p\,d\mu
 \ge-M\int\|\mathbf v\|^4\,d\mu
 \ge-C_nM L(\|\mathbf v\|^4).
\]
If $L(1)=0$, Cauchy--Schwarz gives $L(q)=0$ for every polynomial
$q$ of degree at most two.  In particular, $L(v_i^2)=0$, so another
application of Cauchy--Schwarz gives
$L(v_iv_jv_k)=0$ for all $i,j,k$.  Hence $L(p)=0$, while
$L(\|\mathbf v\|^4)\ge0$.  This treats all separating functionals
and proves the claim.
\end{proof}

\subsubsection{Localized models and an exact fourth-root upper bound}

\begin{lemma}[A localized polynomial square kernel]
\label{lem:torus-fourth-moment-kernel}
For every integer $N\geq2$, there is a nonnegative trigonometric
polynomial $K_N$ on $\mathbb T$ with normalized integral one such that
\[
 K_N=|W_N|^2,\qquad \degF  W_N\leq2N-2,\qquad
 \int_{-\pi}^{\pi}|t|^4K_N(t)\,\frac{dt}{2\pi}
 \leq C N^{-4},
\]
where $C$ is absolute.  For the product kernel
$\mathcal K_N(\mathbf t)=\prod_{j=1}^qK_N(t_j)$, one has
\[
 \int_{[-\pi,\pi]^q}\|\mathbf t\|_2^4
       \mathcal K_N(\mathbf t)\,\frac{d\mathbf t}{(2\pi)^q}
 \leq C_qN^{-4}.
\]
\end{lemma}

\begin{proof}
Put
\[
 A_N(t)=\sum_{j=0}^{N-1}e^{ijt},\qquad
 a_N=\int_{\mathbb T}|A_N(t)|^8\,\frac{dt}{2\pi},
 \qquad
 K_N(t)=\frac{|A_N(t)|^8}{a_N}.
\]
The elementary sine formula for $A_N$ gives, for $|t|\leq\pi$,
\[
 |A_N(t)|\leq\min\{N,\pi/|t|\},
 \qquad
 |A_N(t)|\geq 2N/\pi\quad (|t|\leq1/N).
\]
The value at $t=0$ is interpreted by continuity.  Consequently
$cN^7\leq a_N\leq CN^7$.  Splitting the fourth-moment integral
at $|t|=1/N$ gives
\[
 \int_{-\pi}^{\pi}|t|^4K_N(t)\,\frac{dt}{2\pi}
 \leq\frac{C}{N^7}
 \left(N^8\int_0^{1/N}t^4\,dt
             +\int_{1/N}^{\pi}t^{-4}\,dt\right)
 \leq C'N^{-4}.
\]
The centered factor
\[
 W_N(t)=a_N^{-1/2}e^{-2i(N-1)t}A_N(t)^4
\]
has frequencies between $-2N+2$ and $2N-2$, and $K_N=|W_N|^2$.
Finally, $\|\mathbf t\|_2^4\leq q\sum_j|t_j|^4$, so the product
claim follows from the normalization of the other factors.
\end{proof}

\begin{lemma}[Uniform cubic models]
\label{lem:torus-uniform-cubic-models}
Fix $q,D$.  There are constants $M,A>0$, depending only on $q,D$,
with the following property.  If $p$ is a real trigonometric
polynomial satisfying $\degF p\leq D$ and $0\leq p\leq1$ on
$\mathbb T^q$, then for each $\boldsymbol\phi\in\mathbb T^q$ the
cubic polynomial
\begin{equation}\label{eq:torus-cubic-model}
 \begin{aligned}
 J_{\boldsymbol\phi}(\mathbf v)
 ={}&p(\boldsymbol\phi)
       +\sum_i p_i(\boldsymbol\phi)v_i
       +\tfrac12\sum_{i,j}p_{ij}(\boldsymbol\phi)v_iv_j\\
    &+\tfrac16\sum_{i,j,k}p_{ijk}(\boldsymbol\phi)v_iv_jv_k
       +\tfrac16\sum_i p_i(\boldsymbol\phi)v_i^3
 \end{aligned}
\end{equation}
satisfies
\[
 J_{\boldsymbol\phi}(\mathbf v)\geq-M\|\mathbf v\|_2^4
 \qquad(\mathbf v\in\mathbb R^q).
\]
Let $C_q$ be the constant in Lemma~\ref{lem:cubic-sos-completion} and
set
\[
 Q_{\boldsymbol\phi}(\mathbf v)
 =J_{\boldsymbol\phi}(\mathbf v)+C_qM\|\mathbf v\|_2^4.
\]
Then $Q_{\boldsymbol\phi}$ is an SOS of polynomials of total degree
at most two, and for every $\mathbf t\in[-\pi,\pi]^q$,
\begin{equation}\label{eq:torus-cubic-model-error}
 \left|Q_{\boldsymbol\phi}(\sin\mathbf t)
              -p(\boldsymbol\phi+\mathbf t)\right|
 \leq A\|\mathbf t\|_2^4.
\end{equation}
Here $\sin\mathbf t=(\sin t_1,\ldots,\sin t_q)$, and subscripts on
$p$ denote angular derivatives.
\end{lemma}

\begin{proof}
The Fourier-degree bound and $\|p\|_\infty\leq1$ give uniform
bounds, depending only on $q,D$, on all derivatives of $p$ through
order four.  Indeed,
\[
 \|\partial^\alpha p\|_\infty
 \leq D^{|\alpha|}\|p\|_{\mathrm W}
 \leq D^{|\alpha|}(2D+1)^{q/2}
 \qquad(1\leq|\alpha|\leq4).
\]
For $\|\mathbf v\|_2\leq1/2$, consider the nonnegative smooth
function $p(\boldsymbol\phi+\arcsin\mathbf v)$, with componentwise
principal inverse sines.  Its cubic Taylor polynomial at zero is
exactly \eqref{eq:torus-cubic-model}, since
$\arcsin v=v+v^3/6+O(v^5)$.  Its fourth derivatives are uniformly
bounded on this ball.  Taylor's theorem therefore gives
\[
 \left|p(\boldsymbol\phi+\arcsin\mathbf v)
                    -J_{\boldsymbol\phi}(\mathbf v)\right|
 \leq M_0\|\mathbf v\|_2^4.
\]
Nonnegativity of $p$ proves the required lower bound there.
Outside that ball the uniformly bounded coefficients of $J$ give
\[
 |J_{\boldsymbol\phi}(\mathbf v)|
 \leq B(1+\|\mathbf v\|_2+\|\mathbf v\|_2^2+\|\mathbf v\|_2^3)
 \leq M_1\|\mathbf v\|_2^4.
\]
Thus $M=\max\{M_0,M_1\}$ works globally, independently of $p$ and
$\boldsymbol\phi$.  Lemma~\ref{lem:cubic-sos-completion} gives the SOS
assertion.

If $\|\mathbf t\|_2\leq1/2$, then
$\arcsin(\sin\mathbf t)=\mathbf t$ and
$\|\sin\mathbf t\|_2\leq\|\mathbf t\|_2$.  The same Taylor
estimate, together with the added quartic term, proves
\eqref{eq:torus-cubic-model-error} in this ball.
On the rest of $[-\pi,\pi]^q$, both functions on the left are
uniformly bounded and $\|\mathbf t\|_2^4\geq1/16$.
Increasing $A$ proves the global estimate.
\end{proof}

\begin{theorem}[Uniform full-torus bound]
\label{thm:independent-torus-quarter}\label{thm:full-torus-fourth-root}
For every fixed $q\geq1$ and $D\geq0$, there is a constant $C_{q,D}$
such that every real trigonometric polynomial $f$ with
\[
 \degF  f\leq D,\qquad \|f\|_{\mathbb T^q,\infty}\leq1,
 \qquad f\geq\delta>0\text{ on }\mathbb T^q
\]
has a real trigonometric SOS certificate of budget
\[
 d\leq C_{q,D}\bigl(\delta^{-1/4}+1\bigr).
\]
It can be chosen with
\[
 \size_{\mathrm W}\leq1+(d+1)^q,
 \qquad
 \log\size_{\mathrm W}=O_{q,D}(\log(2/\delta)).
\]
\end{theorem}

\begin{proof}
Constant objectives are immediate.  Put $p=f-\delta/2$, so that
$0\leq p\leq1$, and use the uniformly chosen $M$ from
Lemma~\ref{lem:torus-uniform-cubic-models}.  For $N\geq2$, define
\begin{equation}\label{eq:torus-local-model-integral}
 F_N(\boldsymbol\theta)
 =\int_{\mathbb T^q}
 \mathcal K_N(\boldsymbol\theta-\boldsymbol\phi)
 Q_{\boldsymbol\phi}\bigl(\sin(\boldsymbol\theta-
                                      \boldsymbol\phi)\bigr)
 \,\frac{d\boldsymbol\phi}{(2\pi)^q}.
\end{equation}

For each fixed $\boldsymbol\phi$, the integrand is a real
trigonometric SOS.  To check its budget, write
$Q_{\boldsymbol\phi}(\mathbf v)=\sum_jB_{\boldsymbol\phi,j}
(\mathbf v)^2$ with $\deg B_{\boldsymbol\phi,j}\leq2$.
The product of the complex kernel factor
$\prod_iW_N(\theta_i-\phi_i)$ with
$B_{\boldsymbol\phi,j}(\sin(\boldsymbol\theta-
\boldsymbol\phi))$ has coordinate Fourier degree at most
$(2N-2)+2=2N$.  Its modulus square splits into the squares of
its real and imaginary parts.  Thus each integrand has SOS budget
at most $4N$.

Integration preserves this finite degree budget without any need to
choose the factors $B_{\boldsymbol\phi,j}$ continuously.
Indeed, the coefficients of the integrand depend continuously on
$\boldsymbol\phi$, so positive Riemann sums converge to $F_N$ in
the finite-dimensional coefficient space.  The cone of real
trigonometric SOS with factor degree at most $2N$ is closed:
in a real orthonormal Fourier basis, a positive semidefinite Gram
matrix has trace equal to the integral of the represented
polynomial.  A convergent sequence of represented polynomials
therefore admits Gram matrices with bounded trace, and a subsequence
converges to a positive semidefinite Gram matrix for the limit.
It follows that $F_N$ is an exact finite SOS of budget at most $4N$.

The polynomial $F_N$ also has the much smaller Fourier degree
$\degF F_N\leq D$.  To see this, expand
$J_{\boldsymbol\phi}(\mathbf v)=\sum_{|\alpha|\leq3}
a_\alpha(\boldsymbol\phi)\mathbf v^\alpha$.
Every coefficient $a_\alpha$ is a fixed linear combination of
derivatives of $p$, and hence has Fourier support in
$\{\mathbf k:\|\mathbf k\|_\infty\leq D\}$.
Changing variables $\mathbf t=\boldsymbol\theta-
\boldsymbol\phi$ in \eqref{eq:torus-local-model-integral} shows
that the corresponding contribution is the convolution of
$a_\alpha$ with
$\mathcal K_N(\mathbf t)(\sin\mathbf t)^\alpha$; this cannot
introduce new Fourier frequencies.  The quartic correction contributes
only a constant because $M$ is independent of $\boldsymbol\phi$.

By \eqref{eq:torus-cubic-model-error} and the fourth-moment estimate,
\[
 \|F_N-p\|_\infty\leq A C_qN^{-4}.
\]
The residual $r_N=p-F_N$ has degree at most $D$, so
\begin{equation}\label{eq:torus-fixed-degree-residual}
 \|r_N\|_{\mathrm W}
 \leq(2D+1)^{q/2}\|r_N\|_\infty
 \leq C'_{q,D}N^{-4}.
\end{equation}

We use an elementary Fourier SOS correction.  If $r$ is any real
trigonometric polynomial and $\tau\geq\|r\|_{\mathrm W}$, then
$\tau+r$ is a real trigonometric SOS of budget at most
$2\degF r$.  In fact, choose one frequency from each pair
$\{\mathbf k,-\mathbf k\}$ and write
\[
 r=r_0+\sum_{\mathbf k\in E}
             2\operatorname{Re}(r_\mathbf k e^{i\mathbf k\cdot\theta}).
\]
Then
\[
 \tau+r
 =\left(\tau+r_0-2\sum_{\mathbf k\in E}|r_\mathbf k|\right)
 +\sum_{\mathbf k\in E}
 \left|\sqrt{|r_\mathbf k|}
       +\frac{r_\mathbf k}{\sqrt{|r_\mathbf k|}}
                    e^{i\mathbf k\cdot\theta}\right|^2,
\]
with zero coefficients omitted.  The constant is nonnegative by
the assumption on $\tau$, and real and imaginary parts give real
square factors.

Choose
\[
 N\geq\max\left\{2,D,
                  \left(2C'_{q,D}/\delta\right)^{1/4}\right\}.
\]
Then \eqref{eq:torus-fixed-degree-residual} gives
$\|r_N\|_{\mathrm W}\leq\delta/2$.  Applying the correction
with $\tau=\delta/2$ yields the exact identity
\[
 f=F_N+(\delta/2+r_N)
\]
as a real trigonometric SOS.  Its budget is at most
$\max\{4N,2D\}=4N$, proving the asserted degree bound.

Finally, for any resulting representation $f=\sum_jT_j^2$ with
$2\degF T_j\leq d$, Parseval and Cauchy--Schwarz give
\[
 \sum_j\|T_j\|_{\mathrm W}^2
 \leq(d+1)^q\sum_j\|T_j\|_{L^2}^2
 =(d+1)^q\int_{\mathbb T^q}f\,\frac{d\boldsymbol\theta}{(2\pi)^q}
 \leq(d+1)^q.
\]
This proves both size estimates.
\end{proof}

\subsubsection{A matching fourth-root lower bound}
\label{subsec:independent-torus-lower}

The lower bound is governed by the quartic Taylor term of the objective
at a zero.  We formulate the obstruction for an arbitrary nonnegative
quaternary quartic that is not a sum of squares, and give a concrete
example from Choi~\cite{Choi1975}.

\begin{theorem}[A quartic obstruction to trigonometric SOS]
\label{thm:independent-torus-lower}
Let $P\in\R[x_1,x_2,x_3,x_4]$ be a nonnegative homogeneous quartic
that is not a sum of squares of homogeneous quadratic polynomials.  Set
\[
 M_P=\max_{x\in[-1,1]^4}P(x)>0,\qquad H=\frac{8}{M_P}P.
\]
There exists $c_P>0$, independent of $\varepsilon$ and of the number
of squares in a certificate, such that
\[
 d_{\Sigma_{\rm trig}}
 \bigl(H(\sin\theta_1,\sin\theta_2,\sin\theta_3,\sin\theta_4)
          +\varepsilon\bigr)
 \geq c_P\varepsilon^{-1/4}
 \qquad(0<\varepsilon\leq1).
\]
In particular, for
\[
 F_\varepsilon(\boldsymbol\theta)
 =\frac{H(\sin\theta_1,\sin\theta_2,
              \sin\theta_3,\sin\theta_4)+\varepsilon}{8+\varepsilon},
 \qquad
 \delta_\varepsilon=\frac{\varepsilon}{8+\varepsilon},
\]
we have $\degF F_\varepsilon\leq4$,
$\|F_\varepsilon\|_\infty=1$,
$\min F_\varepsilon=\delta_\varepsilon$, and
\[
 d_{\Sigma_{\rm trig}}(F_\varepsilon)
 =\Theta_P(\delta_\varepsilon^{-1/4}).
\]
One explicit choice, for which $M_P=11$ and
$\degF F_\varepsilon=4$, is
\begin{equation}\label{eq:independent-choi-form}
\begin{aligned}
 P(x_1,x_2,x_3,x_4)
 ={}&x_1^4+x_2^4+x_3^2x_4^2
       +2x_1^2x_3^2+2x_2^2x_4^2\\
    &-2(x_1^2+x_2^2)x_3x_4.
\end{aligned}
\end{equation}
\end{theorem}

We first record the elementary analytic estimate used in the proof.
All vector norms below are Euclidean; their constants are independent
of the dimension of the vector.

\begin{lemma}[A vector trigonometric derivative estimate]
\label{lem:independent-trig-derivatives}
If $R$ is a vector of trigonometric polynomials of degree at most an
integer $m\geq1$, and $M=\sup_{t\in\mathbb R}\|R(t)\|$, then
\[
 \|R'\|_\infty\leq emM,
 \qquad
 \|R''\|_\infty\leq (em)^2M.
\]
\end{lemma}
\begin{proof}
For a scalar Laurent polynomial $r(t)=\sum_{k=-m}^ma_ke^{ikt}$,
let $p(w)=w^m\sum_{k=-m}^ma_kw^k$.  The maximum-modulus principle,
applied to $p(w)$ inside the unit disk and to its reversed polynomial
outside it, gives
\[
 |r(t+\mathrm i s)|\leq e^{m|s|}\|r\|_\infty.
\]
Applying this to scalar products with arbitrary complex unit vectors
gives the same estimate for the norm of $R(t+\mathrm i s)$.
Cauchy's formula on the circle of radius $1/m$ centered at a real
point therefore gives
\[
 \|R^{(k)}(t)\|\leq k!e m^k M.
\]
For $k=1,2$, the stated bounds follow from $2e\leq e^2$.
\end{proof}

\begin{lemma}[A quadratic local bound controls the second-order quotient]
\label{lem:independent-quotient}
Let $T$ be a real vector trigonometric polynomial of degree at most
$m\geq1$.  Suppose that for some $\varepsilon>0$,
\begin{equation}\label{eq:independent-local-growth}
 \|T(t)\|\leq\sqrt\varepsilon+64\sqrt2(1-\cos t),
 \qquad
 \|T(0)\|=\sqrt\varepsilon,
 \qquad
 \left.\frac{d^2}{dt^2}\|T(t)\|^2\right|_{t=0}=0.
\end{equation}
Put $A=T'(0)$ and $L=em$.  There is a real vector trigonometric
polynomial $R$, of degree at most $m$, such that
\begin{equation}\label{eq:independent-exact-division}
 T(t)=T(0)+\sin t\,A+(1-\cos t)R(t).
\end{equation}
Writing $M=\|R\|_\infty$, we have
\begin{equation}\label{eq:independent-quotient-bounds}
 M\leq400(1+L^2\sqrt\varepsilon),
 \qquad
 \|A\|\leq\varepsilon^{1/4}\sqrt M.
\end{equation}
\end{lemma}
\begin{proof}
In the Laurent variable $z=e^{it}$, the numerator
$T(t)-T(0)-\sin t\,A$ has a double zero at $z=1$.
Since $1-\cos t=-(z-1)^2/(2z)$, division is exact and gives a
Laurent polynomial of degree at most $m-1$.  It is real on the
real axis, proving \eqref{eq:independent-exact-division}.

The derivative estimate in Lemma~\ref{lem:independent-trig-derivatives}
gives $\|R'\|_\infty\leq LM$.  It follows that
\begin{equation}\label{eq:independent-norming}
 M\leq2\sup_{1/(2L)\leq|t|\leq\pi}\|R(t)\|.
\end{equation}
Indeed, if a point of maximum norm lies in $|t|<1/(2L)$, moving
it to the nearer endpoint of that interval changes the norm by at
most $M/2$; if it lies outside, the assertion is immediate.
On the range in \eqref{eq:independent-norming},
\[
 1-\cos t\geq\frac{2t^2}{\pi^2}\geq\frac1{2\pi^2L^2},
 \qquad
 \frac{|\sin t|}{1-\cos t}
   =|\cot(t/2)|\leq\frac2{|t|}\leq4L.
\]
Thus \eqref{eq:independent-local-growth} and the exact division give
\begin{equation}\label{eq:independent-M-prebound}
 M\leq128\sqrt2+8\pi^2L^2\sqrt\varepsilon+8L\|A\|.
\end{equation}
The vanishing quadratic coefficient in
\eqref{eq:independent-local-growth}, together with $T''(0)=R(0)$,
gives
\[
 \|A\|^2=-\langle T(0),R(0)\rangle\leq\sqrt\varepsilon M.
\]
Substitute this into \eqref{eq:independent-M-prebound} and use
\[
 8L\varepsilon^{1/4}\sqrt M
 \leq\tfrac12M+32L^2\sqrt\varepsilon.
\]
We obtain
\[
 M\leq256\sqrt2+(16\pi^2+64)L^2\sqrt\varepsilon
 \leq400(1+L^2\sqrt\varepsilon),
\]
as required.
\end{proof}

\begin{proof}[Proof of Theorem~\ref{thm:independent-torus-lower}]
We first verify the properties of the explicit form
\eqref{eq:independent-choi-form}.  Its nonnegativity follows directly
when $x_3x_4\leq0$, since every summand is then nonnegative.
For $x_3x_4>0$, put
\[
 a=\frac{x_1^2}{x_4^2},\qquad
 b=\frac{x_2^2}{x_4^2},\qquad
 t=\frac{x_3}{x_4}>0.
\]
Then
\[
 \frac{P(x)}{x_4^4}
 =a^2+b^2+t^2+2a(t^2-t)+2b(1-t).
\]
For $0<t\leq1$, this becomes
\[
 (a+t^2-t)^2+b^2+2b(1-t)+t^3(2-t)\geq0,
\]
while for $t\geq1$ it equals
\[
 a^2+2a(t^2-t)+(b+1-t)^2+2t-1\geq0.
\]
Thus $P\geq0$ on $\R^4$.

Suppose $P=\sum_jq_j^2$ for homogeneous quadratics $q_j$.
The vanishing coefficients of $x_3^4,x_4^4$ force every $q_j$
to contain no $x_3^2,x_4^2$ terms.  Then the vanishing
coefficients of $x_1^2x_4^2,x_2^2x_3^2$ also exclude
$x_1x_4,x_2x_3$.  Consequently, each quadratic has the form
\[
 q_j=a_jx_1^2+b_jx_2^2+c_jx_1x_2
          +d_jx_1x_3+e_jx_2x_4+h_jx_3x_4.
\]
Comparing the coefficients of $x_1^4,x_2^4,x_3^2x_4^2$ and
$x_1^2x_3x_4,x_2^2x_3x_4$ yields
\[
 \sum_j a_j^2=\sum_jb_j^2=\sum_jh_j^2=1,
 \qquad \sum_ja_jh_j=\sum_jb_jh_j=-1.
\]
Equality in Cauchy--Schwarz implies $a_j=b_j=-h_j$ for
all $j$.  But the coefficient of $x_1^2x_2^2$ would then be
\[
 0=2\sum_ja_jb_j+\sum_jc_j^2
   =2+\sum_jc_j^2,
\]
a contradiction.  Hence $P$ is not SOS.
Finally, termwise estimation gives
$P(x)\leq11\|x\|_\infty^4$, with equality at
$(1,1,-1,1)$.  Thus $M_P=11$, and the normalized form
for this example is $H=\frac8{11}P$.

We now prove the claimed bound for an arbitrary $P$ as in the
theorem.  By homogeneity and the normalization,
\[
 H\geq0,\qquad H\notin\Sigma_{4,4},\qquad
 H(x)\leq8\|x\|_\infty^4,
 \qquad \max_{[-1,1]^4}H=8,
\]
where $\Sigma_{4,4}$ denotes the cone of sums of squares of
homogeneous quadratic polynomials.

For completeness, let $\mathcal S$ denote the cone of sums of squares
of homogeneous quadratics, viewed in the finite-dimensional space of
homogeneous quartics.  This cone is closed.  To see this, let $z(x)$
be the vector of degree-two monomials and write an SOS quartic as
$z(x)^TGz(x)$ with $G\succeq0$.  The matrix
\[
 B=\int_{S^3}z(x)z(x)^T\,d\omega(x)
\]
is positive definite: a nonzero homogeneous quadratic cannot vanish
on the entire sphere.  Therefore the sphere integral of a convergent
sequence of SOS quartics bounds the trace of their positive
semidefinite Gram matrices.  A convergent subsequence of those
matrices supplies a positive semidefinite Gram matrix for the limit.

Let $V=\{0,1,2,3,4\}^4$ and put
\[
 \|U\|_V=\max_{v\in V}|U(v)|,
 \qquad
 \Delta=\inf_{Q\in\mathcal S}\|H-Q\|_V.
\]
The first expression is a norm on homogeneous quartics: repeated
univariate interpolation shows that a polynomial of degree at most
four in each variable which vanishes on $V$ vanishes identically.
Closedness and $H\notin\mathcal S$ consequently imply $\Delta>0$.

Consider an arbitrary certificate
\begin{equation}\label{eq:independent-certificate}
 H(\sin\boldsymbol\theta)+\varepsilon
 =\sum_{j=1}^rT_j(\boldsymbol\theta)^2
 =\|\boldsymbol T(\boldsymbol\theta)\|^2,
 \qquad \max_j\degF T_j=n.
\end{equation}
Necessarily $n\geq1$.  For $v\in V$, restrict to the integer
direction $\boldsymbol\theta=tv$.  The resulting vector
$T_v(t)=\boldsymbol T(tv)$ has trigonometric degree at most
$m=16n$.  For every integer $0\leq k\leq4$,
$|\sin(kt)|\leq k|\sin t|$, so
\[
 \sqrt{H(\sin(tv))}
 \leq\sqrt8\,4^2\sin^2t
 \leq64\sqrt2(1-\cos t).
\]
Together with \eqref{eq:independent-certificate}, this proves the
growth assumption in Lemma~\ref{lem:independent-quotient}.
Moreover, $H(\sin(tv))=H(v)t^4+O(t^6)$, so its quadratic
coefficient is zero.  The lemma applies to every $v$ with the same
$L=16en$.  Let $A_v=T_v'(0)$ and let $R_v$ be its quotient.
Writing
\[
 a=L\varepsilon^{1/4},\qquad M_v=\|R_v\|_\infty,
\]
we have
\begin{equation}\label{eq:independent-uniform-M}
 M_v\leq400(1+a^2),\qquad
 \|A_v\|\leq\varepsilon^{1/4}\sqrt{M_v}.
\end{equation}

Define the homogeneous SOS quartic
\[
 Q(x)=\frac14\left\|
   \sum_{i,j=1}^4x_ix_j\,\partial_i\partial_j
                      \boldsymbol T(0)\right\|^2.
\]
The coefficient of $t^4$ in \eqref{eq:independent-certificate}
along the line $tv$ gives the exact identity
\begin{equation}\label{eq:independent-quartic-jet}
 H(v)-Q(v)
 =\frac13\langle A_v,T_v'''(0)\rangle
  +\frac1{12}\langle\boldsymbol T(0),T_v''''(0)\rangle.
\end{equation}
The exact division implies
\[
 T_v'''(0)=-A_v+3R_v'(0),\qquad
 T_v''''(0)=-R_v(0)+6R_v''(0).
\]
The derivative lemma and \eqref{eq:independent-uniform-M} therefore
bound the absolute value of \eqref{eq:independent-quartic-jet} by
\begin{align*}
 |H(v)-Q(v)|
 &\leq \frac13\|A_v\|^2+L\|A_v\|M_v
       +\frac{\sqrt\varepsilon}{12}(1+6L^2)M_v\\
 &\leq aM_v^{3/2}+a^2M_v.
\end{align*}
Here the second inequality uses $L\geq1$ and
$\|A_v\|^2\leq\sqrt\varepsilon M_v$.
If $a\leq1$, then $M_v\leq800$, and hence
\begin{equation}\label{eq:independent-distance-bound}
 \|H-Q\|_V\leq(800^{3/2}+800)a<24000a.
\end{equation}
Since $Q\in\mathcal S$, the left side is at least $\Delta$.
It follows that
\[
 16en\varepsilon^{1/4}
 >a_0:=\min\{1,\Delta/48000\}>0.
\]
Thus every certificate has budget
\[
 2n>\frac{a_0}{8e}\,\varepsilon^{-1/4},
\]
which proves the lower bound with a positive constant depending only
on the fixed form $H$.

Finally, $H(0)=0$ by homogeneity, and the maximum of $H$ on
$[-1,1]^4$ is $8$ by normalization.  Since the sine map is onto
this cube, $\|F_\varepsilon\|_\infty=1$ and
$\min F_\varepsilon=\delta_\varepsilon$.  Also
$\degF F_\varepsilon\leq4$; for the explicit form
\eqref{eq:independent-choi-form}, the $x_1^4$ term shows that
$\degF F_\varepsilon=4$.  Positive rescaling does not change
certificate degree, and $\delta_\varepsilon\asymp\varepsilon$.
The matching upper bound follows from
Theorem~\ref{thm:full-torus-fourth-root}, proving the stated
$\Theta_P$ order.
\end{proof}

The same family proves sharpness on every $\mathbb T^q$ with $q>4$:
ignore the additional angles, then fix them in an arbitrary certificate
to obtain a certificate on $\mathbb T^4$ without increasing its degree.

\begin{remark}[Dimension and quantifiers]\label{rem:torus-zero-dimensions}
The sharpness assertion is a worst-case statement over fixed torus
dimensions, attained already at $q=4$ and $D=4$.  It does not assert
that every dimension has the same sharp rate.
For $q=1$, Fej\'er--Riesz factorization gives a degree bound depending
only on $D$, independent of $\delta$.
For $q=2$, every fixed nonnegative trigonometric polynomial $p$ is
a finite trigonometric SOS; hence $p+\varepsilon$ has bounded
certificate degree as $\varepsilon\downarrow0$.
Nevertheless, no bound depending only on the input degree exists
uniformly over all nonnegative bivariate trigonometric polynomials.
Both assertions are recorded in the corrected factorization theorem
and its discussion in \cite{DritschelKlepMcCulloughVolcic2026},
Theorem~1.1 and the paragraph following it.  The earlier claimed
partial degree bounds are explicitly corrected there.
Thus boundedness for a fixed limiting polynomial must not be confused
with a uniform bound over all objectives of degree at most $D$.
No sharper dimension-specific optimum is asserted here for $q=2,3$.
\end{remark}

\subsection{Zero inequalities with an additional equality}
\label{subsec:toreq-zero}

The absence of inequality generators does not remove the logarithmic
obstruction if additional equality constraints are allowed.  We show
this by converting the preceding one-inequality example into a
presentation on a larger torus with one equality and no inequalities.

Retain the functions from \eqref{eq:torfew-example}:
\[
 x=\sin\theta_1,\qquad y=\sin\theta_2,\qquad
 t=x^2+y^2,\qquad u=\tfrac14-t,\qquad g=u(x^4+u^4).
\]
Introduce one further angle $\psi$ and define
\begin{equation}\label{eq:toreq-module}
 h(\theta_1,\theta_2,\psi)=\sin^2\psi-g(\theta_1,\theta_2),
 \qquad
 \mathcal M_{\mathrm{tor},0}^{=}
 =\Sigma_{\mathrm{trig}}(\T^3)+(h).
\end{equation}
Here $\Sigma_{\mathrm{trig}}(\T^3)$ denotes finite sums of squares of
real trigonometric polynomials, and $(h)$ is the ideal generated by $h$
in the real trigonometric polynomial ring.
This presentation has zero inequality generators and one fixed
additional equality.  The equality has coordinate Fourier degree ten.

The projection of its feasible set onto the first two angles is
exactly $K_{\mathrm{tor},1}$.
Indeed, $h=0$ is solvable in $\psi$ precisely when $0\le g\le1$.
On $K_{\mathrm{tor},1}$ we have
$0\le u\le1/4$, $x^4\le1/16$, and $u^4\le1/256$, so
\[
 0\le g\le\frac{17}{1024}<1.
\]
Outside $K_{\mathrm{tor},1}$, $g<0$.

\paragraph{A fixed SOS for the complementary factor.}
The reduction below produces factors $1-g$.  They can be removed
using a fixed SOS identity rather than a further constraint.
In fact, the following polynomial identity holds on all of $\R^2$:
\begin{equation}\label{eq:toreq-complement-sos}
 \begin{aligned}
 1-g={}&t(t-1)^4+\frac{11}{4}t^2(t-1)^2
          +\frac18t^3+\frac{27}{32}(t-\tfrac23)^2\\
       &+\frac{37}{256}t+\frac{639}{1024}
          +tx^4+\frac12x^2y^2+\frac14y^4.
 \end{aligned}
\end{equation}
Every term is SOS after substituting $t=x^2+y^2$.
For example,
\[
 \begin{aligned}
 t(t-1)^4&=[x(t-1)^2]^2+[y(t-1)^2]^2,\\
 t^3&=(xt)^2+(yt)^2,\qquad
 tx^4=(x^3)^2+(x^2y)^2.
 \end{aligned}
\]
The other terms are positive multiples of explicit squares.
Consequently, after substituting the sine coordinates we have a fixed
trigonometric SOS representation
\begin{equation}\label{eq:toreq-complement-factors}
 1-g=\sum_{\ell}H_\ell^2,
 \qquad \degF H_\ell\le5.
\end{equation}

\begin{lemma}[Degree comparison for the angular square-root lift]
\label{lem:toreq-degree-comparison}
For every real trigonometric polynomial $F(\theta_1,\theta_2)$ in
$\mathcal M_{\mathrm{tor},1}$, viewed as independent of $\psi$,
\begin{equation}\label{eq:toreq-degree-comparison}
 \frac16\,d_{\mathcal M_{\mathrm{tor},1}}(F)
 \le d_{\mathcal M_{\mathrm{tor},0}^{=}}(F)
 \le d_{\mathcal M_{\mathrm{tor},1}}(F).
\end{equation}
Degrees are measured using the coordinate Fourier budgets of the
preceding subsections.
\end{lemma}

\begin{proof}
A certificate over the one-inequality module lifts directly:
\[
 \begin{aligned}
 F&=\sum_jA_j^2+g\sum_jB_j^2\\
  &=\sum_jA_j^2+\sum_j(\sin\psi\,B_j)^2-h\sum_jB_j^2.
 \end{aligned}
\]
Since $g$ and $h$ both have coordinate Fourier degree ten, the
equality-ideal terms have the same budgets as the original weighted
terms.  The new square factors have budgets at most
$2+2\degF B_j\le10+2\degF B_j$.
Thus the total budget does not increase, proving the second inequality.

Conversely, take a certificate of budget $d$ in the equality-only module:
\begin{equation}\label{eq:toreq-certificate}
 F=\sum_jP_j(\boldsymbol\theta,\psi)^2
                +h\lambda(\boldsymbol\theta,\psi),
 \qquad \boldsymbol\theta=(\theta_1,\theta_2).
\end{equation}
Put $s=\sin\psi$, $c=\cos\psi$, and $z=s^2$.
Every real trigonometric polynomial $P_j$ has a decomposition
\begin{equation}\label{eq:toreq-four-sectors}
 P_j=A_{j,00}(\boldsymbol\theta,z)
       +sA_{j,10}(\boldsymbol\theta,z)
       +cA_{j,01}(\boldsymbol\theta,z)
       +scA_{j,11}(\boldsymbol\theta,z),
\end{equation}
where the $A_{j,ab}$ are polynomials in $z$ with real trigonometric
coefficients in $\boldsymbol\theta$.
To obtain this decomposition, express the $\psi$ dependence polynomially
in $s,c$, replace $c^2$ by $1-s^2$, and separate the even and odd powers
of $s$.

Both $F$ and $h$ are invariant under the independent sign changes
$s\mapsto-s$ and $c\mapsto-c$, corresponding to
$\psi\mapsto-\psi$ and $\psi\mapsto\pi-\psi$.
Average \eqref{eq:toreq-certificate} under these four transformations.
Cross terms between different sectors disappear, giving
\[
 \begin{aligned}
 F={}&\sum_j\bigl(A_{j,00}^2+zA_{j,10}^2
           +(1-z)A_{j,01}^2+z(1-z)A_{j,11}^2\bigr)\\
    &+(z-g)\Lambda(\boldsymbol\theta,z).
 \end{aligned}
\]
This is a polynomial identity in $z$ with trigonometric coefficients:
it holds for every $z=\sin^2\psi\in[0,1]$.
Substitution $z=g(\boldsymbol\theta)$ therefore gives
\begin{equation}\label{eq:toreq-reduced-certificate}
 \begin{aligned}
 F={}&\sum_jA_{j,00}(\boldsymbol\theta,g)^2
       +g\sum_jA_{j,10}(\boldsymbol\theta,g)^2\\
    &+(1-g)\sum_jA_{j,01}(\boldsymbol\theta,g)^2
       +g(1-g)\sum_jA_{j,11}(\boldsymbol\theta,g)^2.
 \end{aligned}
\end{equation}
Replace $1-g$ by \eqref{eq:toreq-complement-factors}.
The right-hand side becomes an SOS plus $g$ times an SOS, hence a
certificate in the original one-inequality module.

We now bound its degree.  Fix a factor $P_j$ and let
$N_j=\degF P_j\le d/2$.
The construction of \eqref{eq:toreq-four-sectors} preserves total
polynomial degree in $s,c$, so
\[
 A_{j,ab}(\boldsymbol\theta,z)
   =\sum_{k=0}^{\lfloor(N_j-a-b)/2\rfloor}
                  a_{j,ab,k}(\boldsymbol\theta)z^k,
 \qquad \degF a_{j,ab,k}\le N_j,
\]
with a sector interpreted as zero if $N_j<a+b$.
Since $\degF g=10$,
\[
 \degF A_{j,ab}(\boldsymbol\theta,g)
 \le N_j+10\left\lfloor\frac{N_j-a-b}{2}\right\rfloor.
\]
For the sector $(a,b)$, replacement of $(1-g)^b$ introduces a square
factor of degree at most $5b$, and the remaining generator weight is
$g^a$.  Its resulting certificate budget is therefore at most
\[
 \begin{aligned}
 &10a+2N_j
       +20\left\lfloor\frac{N_j-a-b}{2}\right\rfloor+10b\\
 &\qquad\le12N_j\le6d.
 \end{aligned}
\]
This proves
$d_{\mathcal M_{\mathrm{tor},1}}(F)\le6d$, and hence the first
inequality in \eqref{eq:toreq-degree-comparison}.
\end{proof}

\begin{corollary}[A sharp logarithmic example with zero torus inequalities]
\label{cor:toreq-sharp}
For the fixed equality-only presentation \eqref{eq:toreq-module},
let $f_\varepsilon$ and $\delta_\varepsilon$ be as in
\eqref{eq:torfew-normalization}, with $f_\varepsilon$ independent of
$\psi$.  Then
\begin{equation}\label{eq:toreq-sharp}
 d_{\mathcal M_{\mathrm{tor},0}^{=}}(f_\varepsilon)
 =\Theta\!\left(\delta_\varepsilon^{-1/2}
                      \log(1/\delta_\varepsilon)\right).
\end{equation}
The objective has unit sup norm on $\T^3$ and minimum
$\delta_\varepsilon$ on the equality-constrained feasible set.
\end{corollary}

\begin{proof}
Apply Lemma~\ref{lem:toreq-degree-comparison} and
Theorem~\ref{thm:torfew-one-sharp}.
The normalization is unchanged because the objective is independent
of $\psi$, and its minimum is unchanged because the feasible-set
projection is $K_{\mathrm{tor},1}$.
\end{proof}

Thus an improved bound for positivity on the entire torus must be
stated with the absence of additional equality constraints as an
explicit hypothesis.  The equality-only presentation
\eqref{eq:toreq-module} has no inequality generators but retains the
sharp single-logarithm obstruction.

\subsection{Sharp worst-case orders and the role of equalities}
\label{subsec:torus-classification}

The preceding results distinguish the number of inequality generators
from additional equality constraints. For each fixed presentation and
fixed input Fourier degree, they give the following uniform upper
bounds, each attained by a fixed normalized family in the stated class:
\[
\begin{array}{c|c|c}
\text{Torus presentation} & \text{Sharp worst-case order}
                         & \text{Attaining dimension}\\ \hline
m=3 & \delta^{-1/2}\log^2(1/\delta) & q=3\\[2pt]
m=1 & \delta^{-1/2}\log(1/\delta) & q=2\\[2pt]
m=2 & \delta^{-1/2}\log(1/\delta) & q=3\\[2pt]
m=0,\ \text{no additional equalities} & \delta^{-1/4} & q=4\\[2pt]
m=0,\ \text{additional equalities allowed}
    & \delta^{-1/2}\log(1/\delta) & q=3
\end{array}
\]
The three-inequality assertion is proved in
Section~\ref{subsec:torus-three}; the one- and two-inequality assertions
are Corollary~\ref{cor:torfew-upper},
Theorem~\ref{thm:torfew-one-sharp}, and
\eqref{eq:torfew-two-equality}.
The equality-free case follows from
Theorems~\ref{thm:full-torus-fourth-root} and~\ref{thm:independent-torus-lower};
the equality-constrained case is Corollary~\ref{cor:toreq-sharp}.
In the fourth row, sharpness is attained at dimension four and in all
higher dimensions, while the dimension-specific qualifications in
Remark~\ref{rem:torus-zero-dimensions} remain in force.

Thus, removing all additional constraints changes the power of the
positivity margin as well as removing the logarithm. The equality-only
example shows why the ambient circle relations and further equality
equations must be distinguished in formulating this conclusion.

\subsection{Relaxation error and dependence on the generators}
The matching degree bounds in Theorem~\ref{thm:main} also determine
how quickly SOS relaxations approach the optimum. We first invert the degree--margin relation for
the fixed polynomial example. Since membership bounds alone do not
identify the dual moment optimum, we next verify strict feasibility
and use semidefinite duality. Finally, we vary the presentation while
keeping its feasible set fixed to show that the constants cannot be
presentation-independent.
For the fixed polynomial example, define
\[
 e_d=\inf\{\varepsilon>0:u+\varepsilon\in(\M_*)_d\}.
\]
We first invert the matching degree estimates to obtain the
relaxation-error order:
\begin{equation}\label{eq:hierarchy}
 e_d=\Theta\!\left(\frac{(\log d)^4}{d^2}\right).
\end{equation}
For the upper bound substitute $\varepsilon=A(\log d)^4/d^2$ and choose $A$ large, using the upper bound for the normalized family and the fact that multiplying the objective by a positive constant does not change its certificate degree.
For the lower bound, if $\varepsilon\ge d^{-1}$ the assertion is immediate for large $d$; otherwise $\log(1/\varepsilon)\ge\log d$ and the degree lower bound gives \eqref{eq:hierarchy}.
The inequalities hold for all admissible $\varepsilon$
and therefore pass to their infimum.
The same argument gives a general upper bound on
relaxation error. For every fixed objective $f$ on
a fixed Archimedean presentation, applying
Theorem~\ref{thm:main} to
$(f-\min_K f+\varepsilon)/A_f$, where
$A_f=1+\|f-\min_K f\|_{Q,\infty}$ and
$0<\varepsilon\le1$, gives an SOS membership error
of $O((\log d)^4/d^2)$.
By weak duality, the moment relaxation value lies
between the SOS relaxation value and $\min_K f$,
so its error has the same upper bound.

To establish the matching moment lower bound for
the fixed example, we identify its SOS and moment
relaxation values using strict feasibility.
For the objective $u$, define
\[
u_d^{\mathrm{SOS}}
=\sup\{\gamma:u-\gamma\in(\M_*)_d\},\qquad
u_d^{\mathrm{mom}}
=\inf\{\ell(u):\ell(1)=1,\ \ell((\M_*)_d)\ge0\},
\]
where $\ell$ is linear on polynomials of degree at most $d$.
For sufficiently large $d$, the SOS problem is feasible; since
$\min_{K_*}u=0$, the definitions give
$e_d=-u_d^{\mathrm{SOS}}$ and weak duality gives
$u_d^{\mathrm{SOS}}\le u_d^{\mathrm{mom}}\le0$.
To obtain the matching \emph{lower} bound for the moment error, observe
that all four generators of $\M_*$ are strictly positive on a small
closed box around $(u,x,y)=(1/4,-1/4,-1/4)$. Normalized Lebesgue
measure $\mu$ on this box satisfies $\int p^2\,d\mu>0$ and
$\int g_i p^2\,d\mu>0$ for every nonzero polynomial $p$ and every
generator $g_i$. Hence every truncated moment and localizing matrix is positive
definite on its monomial basis (of degree at most $\lfloor d/2\rfloor$
for the moment matrix, and $\lfloor(d-\deg g_i)/2\rfloor$
for the $g_i$-localizing matrix): the moment SDP is strictly feasible.
For sufficiently large $d$, $u+1\in(\M_*)_d$ by Putinar's theorem,
so its dual SOS problem is feasible and its optimum is finite.
The Slater strong-duality theorem for semidefinite programs
\cite[Theorem~3.1]{VandenbergheBoyd1996} therefore yields
$u_d^{\mathrm{mom}}=u_d^{\mathrm{SOS}}$; consequently,
\[
 \min_{K_*}u-u_d^{\mathrm{mom}}=e_d
 =\Theta\!\left(\frac{(\log d)^4}{d^2}\right).
\]

Finally, we show that the constants cannot depend only on the
feasible set: changing its generators can make certificates
arbitrarily difficult even when the positivity margin stays fixed.
We use a reparametrized and positively rescaled version of
the univariate quadratic-module family studied by
Henrion, Le Franc, and Magron~\cite{HenrionLeFrancMagron2025}.
Unlike their analysis of finite convergence for the
objective $x$, we keep the objective and its positive
margin fixed. Consider
\[
 \M_t=\QM\left(1-x^2,\frac{x(1+x+t)}{2+t}\right),\quad
 f(x)=\frac{x+1/2}{3/2},\quad 0<t\le1.
\]
All these presentations define the same feasible set
$[0,1]$ and admit a common Archimedean witness.
The objective has degree one, both inequality
generators have degree two, and the positivity
margin is fixed at $1/3$. In a rescaled certificate
\[
 x+\tfrac12=\sigma_0+(1-x^2)\sigma_b+x(1+x+t)\sigma_g,
\]
nonnegativity gives $0\le\sigma_g\le2$ on $[1/2,1]$, while evaluation at $-1$ gives $\sigma_g(-1)\ge1/(2t)$. By \eqref{eq:chebyshev-extrapolation},
\[
 \frac1{2t}\le2T_d(7)\le2(7+4\sqrt3)^d,
 \qquad
 d\ge\frac{\log(1/(4t))}{\log(7+4\sqrt3)}.
\]
On $[-4,4]$ the given objective has supremum norm $3$; after
normalization it is $(x+1/2)/(9/2)$, with fixed margin $1/9$.
Thus the constant in Theorem~\ref{thm:main} must retain information about the presentation.

\Needspace{10\baselineskip}
\section{Related work}\label{sec:related}

\paragraph{Positivity certificates and SOS hierarchies.}
The Krivine--Stengle Positivstellensatz gives general algebraic
certificates for semialgebraic positivity~\cite{Stengle1974}.
On compact basic closed semialgebraic sets, Schm\"udgen's
Positivstellensatz uses the \emph{preordering} generated by the defining
inequalities~\cite{Schmudgen1991}. Putinar's Positivstellensatz gives a
representation in the generally smaller \emph{quadratic module} under
an Archimedean hypothesis~\cite{Putinar}. Lasserre and Parrilo used
bounded-degree SOS representations to construct semidefinite
relaxations for polynomial optimization~\cite{Lasserre2001,Parrilo2003}.
Here the generators and equality ideal are fixed throughout, and our
degree budget applies to the weighted SOS terms and equality-ideal terms.
A degree estimate for a preordering cannot, by itself, be interpreted as
an estimate for the corresponding quadratic module.

\paragraph{Effective bounds for general presentations.}
Schweighofer proved a polynomial dependence on the inverse positivity
margin for effective Schm\"udgen representations on a fixed compact
semialgebraic set~\cite{Schweighofer2004}. Nie and Schweighofer obtained
effective bounds for Putinar representations, with substantially larger
margin dependence~\cite{NieSchweighofer2007}. Kriel gave another proof
of the existence of uniform degree bounds at a fixed positive margin,
with the description, objective-degree bound, and norm bound held
fixed~\cite{Kriel}. Scheiderer's nonstability results concern bounds depending on the
objective degree alone, in contrast with positivity-margin-dependent
bounds~\cite{Scheiderer2005}.
Our upper proof uses a coefficient-independent effective
Krivine--Stengle degree bound~\cite{LPR}, but requires an additional
quantitative conversion from Stengle certificates to certificates in
the original quadratic module.

\paragraph{Bounds governed by the defining constraints.}
Baldi and Mourrain obtained polynomial effective Putinar bounds
involving a \L ojasiewicz exponent associated with the constraint
description and quantitative moment approximation~\cite{BaldiMourrain}.
Baldi, Mourrain, and Parusi\'nski refined this analysis, including the
role of constraint qualification and the associated error-bound
constants~\cite{BaldiMourrainParusinski}. More recently,
Heijmans-Kuryatnikova, Vera, and Zuluaga used a lift-and-project
construction to obtain a Putinar degree bound with dependence
$O(\kappa(f)^{2\mathbf L_g})$ on their condition measure, for fixed
remaining parameters and under their Archimedean inequality
assumptions~\cite[Theorem~4]{HKVZ}. Their exponent $\mathbf L_g$
depends on the chosen inequality presentation. These estimates make
dependence on constraint geometry explicit. Our
$O(\delta^{-1/2}\log^2(2/\delta))$ bound instead determines the
margin dependence for fixed constraint polynomials, including
equality generators, with their dependence absorbed in the constant.

\paragraph{The hypercube and polynomial-kernel methods.}
For the standard hypercube, Laurent and Slot obtained
$O(\delta^{-1/2})$ degree bounds for Schm\"udgen certificates using
polynomial approximation kernels~\cite{LaurentSlot2023}. For the
quadratic module generated by $1-x_i^2$, Baldi and Slot proved an
$O(\delta^{-1})$ upper bound and an $\Omega(\delta^{-1/8})$ lower-bound
example, after fixing the dimension, objective degree, and
normalization~\cite{BaldiSlot}. Their Theorem~27 also gives a
degree-controlled conversion from the preordering generated by
$\eta^2-x_i^2$ on an enlarged cube to the quadratic module generated
by $1-x_i^2$, with $\eta=n^{1/(2q)}$, $q\in\N_{\ge1}$, and degree shift
$n(2q-2)$~\cite[Theorem~27]{BaldiSlot}. Earlier studies of hypercube
optimization and product-based certificates include work of
de Klerk and Laurent~\cite{deKlerkLaurent2010} and
Magron~\cite{Magron2015}. Some quantitative estimates in this line
of work are conditional on the de Klerk--Laurent product-certificate
conjecture. Gribling, de Klerk, and Vera subsequently
revisited the polynomial-kernel approach and its
limitations~\cite{GriblingRevisiting2026}; in a separate, later 2026
preprint they improved the error estimate for the standard hypercube
Lasserre hierarchy to $O(\log^3 r/r^2)$ for fixed problem
parameters~\cite[Theorem~7]{GriblingKlerkVera2026}. These estimates
concern the standard hypercube generators. Bounds for a prescribed kernel
construction and bounds for all certificates in a quadratic module
address different questions.

For Archimedean quadratic modules with fixed generators, Putinar's theorem
already gives the qualitative membership $g^\alpha+\rho\in\M$
for every squarefree generator product and every $\rho>0$, since
$g^\alpha\ge0$ on $K$~\cite{Putinar}. This existence statement alone
does not give quantitative bounds on the analytic certificate size.

\paragraph{The unit ball, simplex, and sphere.}
For the unit ball $B^n=\{x:\sum_i x_i^2\le1\}$ and standard simplex
$\Delta^n=\{x:x_i\ge0,\ \sum_i x_i\le1\}$, Slot proves
Schm\"udgen relaxation error bounds of order $r^{-2}$
\cite[Theorems~3--4]{Slot2024}. More precisely, for a polynomial $f$
of degree $D\ge1$, with extrema $f_{\min},f_{\max}$ on
$X\in\{B^n,\Delta^n\}$, the preordering lower bound
$f_r^{\mathrm{Sch}}$ with weighted-term degree budget $2r$ satisfies
\[
 0\le f_{\min}-f_r^{\mathrm{Sch}}
 \le \frac{C_X(n,D)}{r^2}(f_{\max}-f_{\min}),
 \qquad r\ge2nD\sqrt D.
\]
The ball has the single generator $1-\sum_i x_i^2$, so its
preordering equals its quadratic module. The simplex result uses
products of its defining affine inequalities. For fixed dimension,
objective degree, and normalization, these bounds yield certificate
degrees $O(\delta^{-1/2})$ for the corresponding preorderings
and quadratic modules.

On $S^{N-1}=\{x:\sum_i x_i^2=1\}$, Fang and Fawzi consider a
homogeneous polynomial $p$ of degree $2a$, where $1\le a\le N$
\cite[Theorem~1]{FangFawzi2021}. Their SOS upper bound $p_\ell$ for
$p_{\max}$ satisfies, for $\ell\ge C_aN$,
\[
 0\le p_\ell-p_{\max}
 \le (C_aN/\ell)^2(p_{\max}-p_{\min}),
\]
where $C_a$ depends only on $a$. Here $p_\ell$ uses squares of
polynomials of degree at most $\ell$ on the sphere, equivalently an
SOS certificate modulo $1-\sum_i x_i^2$ with degree budget $2\ell$.
Thus the sphere result has an equality constraint and an even-degree
homogeneity hypothesis; the ball and simplex results above apply to
general polynomials in their stated preordering cones.

\paragraph{Trigonometric positivity and torus optimization.}
Dritschel studied factorizations of strictly positive multivariate
trigonometric polynomials~\cite{Dritschel2004}, and Niculescu and
Putinar established a toric Positivstellensatz with applications to
delay systems~\cite{NiculescuPutinar2011}. For \emph{unconstrained}
optimization of trigonometric polynomials, Bach and Rudi proved an
$O(r^{-2})$ SOS relaxation error bound and exponential convergence
under additional assumptions of finitely many global minimizers with
invertible Hessians~\cite{BachRudi2023}. Our torus results allow prescribed inequality and equality
constraints, and the matching torus lower bound applies to arbitrary
real trigonometric polynomials in the SOS decompositions,
including polynomials depending on cosine variables.
For the unconstrained full torus, Theorem~\ref{thm:full-torus-fourth-root}
improves the general upper bound to $O(\delta^{-1/4})$, and
Theorem~\ref{thm:independent-torus-lower} proves worst-case sharpness
in dimension four. For two torus variables,
the corrected factorization theorem of
Dritschel, Klep, McCullough, and Vol\v{c}i\v{c}~\cite{DritschelKlepMcCulloughVolcic2026}
also explains the distinction between bounded certificate degree for
each fixed family $p+\varepsilon$ and uniform bounds over all objectives
of a prescribed degree; see Remark~\ref{rem:torus-zero-dimensions}.

\paragraph{Lower bounds and sharpness.}
Stengle's one-dimensional example, using the nonstandard generator
$(1-x^2)^3$, gives a certificate-degree lower bound of order
$\delta^{-1/2}$~\cite{Stengle1996}.
More recently, Henrion~\cite{Henrion2026} determined
the exact order-$r$ moment-SOS relaxation value
$-1/[r(r-2)]$ for this example, establishing a sharp
$\Theta(r^{-2})$ convergence rate.
Baldi and Slot's hypercube
example concerns the standard generators and yields a different,
weaker lower-bound exponent~\cite{BaldiSlot}.
Both lines of work derive lower bounds from extremal inequalities
for polynomials.
Positive linear functionals and their moment and localizing matrices
provide a complementary dual viewpoint on SOS representations~\cite{Lasserre2001,Laurent2009}.
A related algebraic obstruction occurs for
$\QM(X,Y,-(1+XY))$: although its real feasible set is empty,
the quadratic module is proper~\cite[Example~1.5]{Marshall2009}.
Our explicit functional in Section~\ref{sec:corner} instead
establishes a quantitative lower bound on analytic coefficient norms,
which Section~\ref{sec:lower} converts into a certificate-degree
lower bound using Chebyshev extrapolation.
Henrion, Le Franc, and Magron~\cite{HenrionLeFrancMagron2025}
studied a parametric univariate quadratic module
exhibiting arbitrarily slow finite convergence as
the parameter tends to zero. Section~\ref{sec:consequences}
uses the same module family with a fixed positive
objective margin to demonstrate the necessity of
presentation-dependent constants. These examples and
Scheiderer's nonstability results~\cite{Scheiderer2005} underscore
the importance of the prescribed presentation. Our fixed
three-variable quadratic module, and its three-inequality-generator
torus counterpart, establish a matching
$\Omega(\delta^{-1/2}\log^2(1/\delta))$ certificate-degree lower bound.
The resulting \emph{sharpness} is worst-case optimality over fixed
Archimedean quadratic modules with specified generators: one
fixed constraint system attains the asymptotic order of the general upper bound.

\section{Discussion}\label{sec:discussion}
Sections~\ref{sec:upper} and~\ref{sec:lower} established matching
certificate-degree bounds, and Section~\ref{sec:consequences} gave their
torus and relaxation consequences. We first compare the two appearances
of the squared logarithm, then discuss the dependence on the fixed
presentation. The common quantitative theme is the relation between
certificate \emph{degree} and analytic \emph{coefficient size}:
the upper proof bounds degree through logarithmic certificate size,
while the lower proof forces large analytic coefficient norms.
The Archimedean witness absorbs truncation errors exactly.

The squared logarithm appears independently on both sides. In the upper proof, a product over dyadic scales has coefficient norm bounded by $\exp(O(k^2))$, with $k\asymp\log(1/\delta)$. In the lower proof, the linear functional constructed above has degree-$j$ norm $\exp(O(j^2))$; balancing it against geometric Taylor decay yields the same scale. The two arguments identify distinct mechanisms that produce squared-logarithmic growth.

The constants in the upper bound depend on the fixed constraint polynomials,
Archimedean witness, and objective-degree bound, reflecting the use of
effective Positivstellensatz estimates and real quantifier elimination.
For real-algebraic input data with a specified Archimedean witness,
the construction is effective.

The generator count gives a further distinction. With at most two
inequalities, the direct internal Stengle conversion has polynomial
certificate size, so the degree bound contains one logarithm.
The disk generator $g=(1-x^2-y^2)(x^4+(1-x^2-y^2)^4)$ and its torus
counterpart attain this rate through a fifth-order boundary zero.
Their lower bounds use two Bernstein derivatives followed by
Chebyshev extrapolation, and their Archimedean and degree-comparison
identities are explicit.

For the full torus without additional constraints, positivity gives
uniform cubic Taylor models with quartic remainders. A finite moment
construction completes these models to sums of squares. Integration
against a polynomial square kernel, followed by a fixed-degree Fourier
correction, gives the fourth-root upper bound. The matching lower bound
follows from a quartic obstruction and a direct derivative argument.
Additional equality equations can encode the one-inequality obstruction:
the construction in Section~\ref{subsec:toreq-zero} retains the
single-logarithm lower bound even with no inequality generators.
Section~\ref{subsec:torus-classification} states these cases together
with their quantifiers and attaining dimensions.

\section*{Acknowledgments}

Drawing on ideas from harmonic analysis, the author used
analytic power-series coefficient norms to study quantitative
certificate complexity. This perspective motivated a
reconsideration of the role of the \L ojasiewicz exponent
in positivity-margin bounds for fixed constraint systems.

The author gratefully acknowledges OpenAI's ChatGPT for
its assistance in constructing counterexamples,
exploring proof strategies. The author takes full responsibility for the mathematical
claims and any remaining errors. Comments and corrections
are welcome at \texttt{yuxiangzhao@whu.edu.cn}.

\end{document}